\documentclass[preprint]{imsart}

\usepackage{geometry}
\usepackage{fancyhdr}
\usepackage{amsmath, mathtools}
\usepackage{bbm, mathrsfs}
\usepackage{amssymb,amsthm,upgreek}
\usepackage[inline]{enumitem}
\usepackage{hyperref,xcolor, textcmds}
\usepackage{xargs}
\usepackage{booktabs}
\usepackage{tikz}
\usepackage{caption, subcaption, placeins}
\usepackage{todonotes}
\usepackage{algorithm}
\usepackage{algpseudocode}

\usepackage{tikz}
\usetikzlibrary{intersections} 
\usetikzlibrary{arrows.meta,calc,decorations.pathreplacing}

\usetikzlibrary{calc,arrows.meta,positioning, fit}

\usepackage{pgfplots}
\usepgfplotslibrary{fillbetween}
\pgfplotsset{compat=newest}
\usepgfplotslibrary{groupplots}
\pgfplotsset{compat=1.18}
\usepgfplotslibrary{statistics}
\usepackage[noabbrev,capitalize]{cleveref}

\newtheorem{theorem}{Theorem}
\crefname{theorem}{theorem}{Theorems}
\Crefname{Theorem}{Theorem}{Theorems}

\newtheorem{lemma}{Lemma}
\crefname{lemma}{lemma}{lemmas}
\Crefname{Lemma}{Lemma}{Lemmas}

\crefname{corollary}{corollary}{corollaries}
\Crefname{Corollary}{Corollary}{Corollaries}

\newtheorem{proposition}{Proposition}
\crefname{proposition}{proposition}{propositions}
\Crefname{Proposition}{Proposition}{Propositions}

\crefname{remark}{remark}{remarks}
\Crefname{Remark}{Remark}{Remarks}

\newtheorem{example}[theorem]{Example}
\crefname{example}{example}{examples}
\Crefname{Example}{Example}{Examples}

\crefname{figure}{figure}{figures}
\Crefname{Figure}{Figure}{Figures}

\crefname{algorithm}{algorithm}{algorithms}
\Crefname{Algorithm}{Algorithm}{Algorithms}

\newtheorem{assumption}{\textbf{H}\hspace{-3pt}}
\Crefname{assumption}{\textbf{H}\hspace{-3pt}}{\textbf{H}\hspace{-3pt}}
\crefname{assumption}{\textbf{H}}{\textbf{H}}

\Crefname{assumptionAO}{\textbf{AO}\hspace{-3pt}}{\textbf{AO}\hspace{-3pt}}
\crefname{assumptionAO}{\textbf{AO}}{\textbf{AO}}

\Crefname{assumptionL}{\textbf{L}\hspace{-3pt}}{\textbf{L}\hspace{-3pt}}
\crefname{assumptionL}{\textbf{L}}{\textbf{L}}

\crefname{definition}{definition}{definitions}
\Crefname{Definition}{Definition}{Definitions}

\setlist[enumerate]{leftmargin=*}

\def\msa{\mathsf{A}}

\def\msd{\mathsf{D}}

\def\mss{\mathsf{S}}

\def\mso{\mathsf{O}}

\def\msm{\mathsf{M}}
\def\msu{\mathsf{U}}

\def\msx{\mathsf{X}}
\def\msy{\mathsf{Y}}

\def\mcb{\mathcal{B}}  

\def\mcy{\mathcal{Y}}
\def\mcx{\mathcal{X}}

\def\mcu{\mathcal{U}}

\def\mcr{\mathcal{R}}

\def\rset{\mathbb{R}}

\def\rmd{\mathrm{d}}

\newcommand{\R}{\mathbb R}

\def\couplingPi{\pi_{0,1}}

\newcommandx{\functionspace}[2][1=+]{\mathbb{F}_{#1}(#2)}
\newcommand{\argmax}{\operatorname*{arg\,max}}

\newcommandx{\VarDeux}[3][3=]{\operatorname{Var}^{#3}_{#1}\left\{#2 \right\}}

\newcommand{\1}{\mathbbm{1}}

\newcommand{\LeftEqNo}{\let\veqno\@@leqno}

\newcommand{\N}{\ensuremath{\mathbb{N}}}

\newcommand{\PE}{\mathbb{E}}
\newcommand{\PP}{\mathbb{P}}

\newcommandx{\Vnorm}[2][1=V]{\| #2 \|_{#1}}
\newcommandx{\VnormEq}[2][1=V]{\left\| #2 \right\|_{#1}}
\newcommandx{\norm}[2][1=]{\ifthenelse{\equal{#1}{}}{\left\Vert #2 \right\Vert}{\left\Vert #2 \right\Vert^{#1}}}
\newcommandx{\normLigne}[2][1=]{\ifthenelse{\equal{#1}{}}{\Vert #2 \Vert}{\Vert #2\Vert^{#1}}}

\newcommand{\parentheseDeux}[1]{\left[ #1 \right]}

\newcommand{\ps}[2]{\left\langle#1,#2 \right\rangle}

\newcommandx\probaMarkovTilde[2][2=]
{\ifthenelse{\equal{#2}{}}{{\widetilde{\mathbb{P}}_{#1}}}{\widetilde{\mathbb{P}}_{#1}\left[ #2\right]}}

\def\ie{\textit{i.e.}}

\def\eqsp{\;}
\newcommand{\coint}[1]{\left[#1\right)}

\newcommand{\ooint}[1]{\left(#1\right)}
\newcommand{\ccint}[1]{\left[#1\right]}

\def\metric{\mathbf{g}}
\def\cut{\mathrm{Cut}}
\def\ID{\mathrm{ID}}

\newcommandx{\weight}[2][2=n]{\omega_{#1,#2}^N}

\def\dist{\operatorname{dist}}

\newcommandx\sequence[3][2=,3=]
{\ifthenelse{\equal{#3}{}}{\ensuremath{\{ #1_{#2}\}}}{\ensuremath{\{ #1_{#2}, \eqsp #2 \in #3 \}}}}
\newcommandx\process[3][2=,3=]
{\ifthenelse{\equal{#3}{}}{\ensuremath{( #1_{#2})}}{\ensuremath{( #1_{#2})_{ #2 \in #3}}}}
\newcommandx\processt[2][2=]
{\ifthenelse{\equal{#2}{}}{\ensuremath{( #1_{t})}}{\ensuremath{( #1_{t})_{ t \in #2}}}}
\newcommandx\sequenceD[3][2=,3=]
{\ifthenelse{\equal{#3}{}}{\ensuremath{\{ #1_{#2}\}}}{\ensuremath{( #1)_{ #2 \in #3} }}}

\newcommandx{\sequencen}[2][2=n\in\N]{\ensuremath{\{ #1_n, \eqsp #2 \}}}
\newcommandx\sequenceDouble[4][3=,4=]
{\ifthenelse{\equal{#3}{}}{\ensuremath{\{ (#1_{#3},#2_{#3}) \}}}{\ensuremath{\{  (#1_{#3},#2_{#3}), \eqsp #3 \in #4 \}}}}
\newcommandx{\sequencenDouble}[3][3=n\in\N]{\ensuremath{\{ (#1_{n},#2_{n}), \eqsp #3 \}}}

\def\eg{e.g.}

\newcommand{\opnorm}[1]{{\left\vert\kern-0.25ex\left\vert\kern-0.25ex\left\vert #1
    \right\vert\kern-0.25ex\right\vert\kern-0.25ex\right\vert}}

\def\bfe{\mathbf{e}}

\def\bfv{\mathbf{v}}
\def\bfu{\mathbf{u}}

\def\Id{\operatorname{Id}}

\newcommandx{\CPE}[3][1=]{{\mathbb E}_{#1}\left[\left. #2 \, \middle \vert \, #3 \right. \right]} 

\newcommandx{\CPELigne}[3][1=]{{\mathbb E}_{#1}[\left. #2 \,  \vert \, #3 \right. ]} 

\newcommandx{\CPVar}[3][1=]{\mathrm{Var}^{#3}_{#1}\left\{ #2 \right\}}
\newcommand{\CPP}[3][]
{\ifthenelse{\equal{#1}{}}{{\mathbb P}\left(\left. #2 \, \right| #3 \right)}{{\mathbb P}_{#1}\left(\left. #2 \, \right | #3 \right)}}

\def\scrA{\mathscr{A}}

\newcommandx{\osc}[2][1=]{\mathrm{osc}_{#1}(#2)}

\def\Id{\operatorname{Id}}

\def\sphere{\mss}

\def\Jac{\operatorname{Jac}}

\newcommand\coupling[2]{\Gamma(\mu,\nu)}

\renewcommand{\geq}{\geqslant}
\renewcommand{\leq}{\leqslant}

\newcommandx{\wasserstein}[3][1=\distance,3=]{\mathscr{W}_{#1}^{#3}\left(#2\right)}
\newcommandx{\wassersteinLigne}[3][1=\distance,3=]{\mathscr{W}_{#1}^{#3}(#2)}
\newcommandx{\wassersteinD}[1][1=\distance]{\mathscr{W}_{#1}}
\newcommandx{\wassersteinDLigne}[1][1=\distance]{\mathscr{W}_{#1}}

\newcommand{\txts}{\textstyle}

\def\rmM{\mathrm{M}}

\def\pre_Stif{\mathcal{V}}

\def\bfo{\mathbf{o}}
\def\Vlyap{\mathfrak{V}}

\newcommandx{\Voi}[1][1=i]{\Vlyap_{\bfo,#1}}
\newcommandx{\Vlyapc}[2][1=\bfo,2=i]{\Vlyap_{#1,#2}}
\newcommandx{\tVlyapc}[2][1=\bfo,2=i]{\tilde{\Vlyap}_{#1,#2}}

\def\Wlyap{\mathfrak{W}}

\newcommandx{\Woi}[1][1=i]{\Wlyap_{\bfo,#1}}
\newcommandx{\Wlyapc}[2][1=\bfo,2=i]{\Wlyap_{#1,#2}}

\newcommand{\npi}{q_1}

\newcommand{\mcut}{\msm\setminus\cut(x_1)}
\newcommand{\grass}{\mathrm{Gr}(n,p)}
\newcommand{\stiefel}{\mathrm{St}(n,p)}
\def\Rvol{\rho_{\metric}}
\newcommand{\sinc}{\mathrm{sinc}}

\newcommand{\Exp}{\mathrm{Exp}}
\newcommand{\Log}{\mathrm{Log}}

\newcommand{\Input}{\item[\textbf{Input:}]}
\newcommand{\Output}{\item[\textbf{Output:}]}

\newcommand{\frips}{FRIPS}

\newcommand{\Smsm}[1]{S_{#1}\msm}

\newcommand{\Dmsm}[1]{D_{#1}\msm}

\def\cuttime{\mathbf{c}}

\def\grad{\mathrm{grad}}
\def\curv{\mathbf{R}}
\def\detA{\scrA_{x_1,\xi}}

\usepackage{autonum}

\begin{document}

\begin{frontmatter}

\title{Sampling from multi-modal distributions on Riemannian manifolds with training-free stochastic interpolants}


\begin{aug}

\author[A]{\fnms{Alain}~\snm{Durmus}},
\author[A]{\fnms{Maxence}~\snm{Noble}}
\and
\author[A,B]{\fnms{Thibaut}~\snm{Pellerin\thanks{Alphabetical order. Corresponding author: \url{thibaut.pellerin@cea.fr}}}}


\address[A]{CMAP, CNRS, Ecole polytechnique, 91120 Palaiseau, France }

\address[B]{Université Paris-Saclay, CEA, List, F-91120 Palaiseau, France }

\end{aug}

\begin{abstract}
In this paper, we propose a general methodology for sampling from un-normalized densities defined on Riemannian manifolds, with a particular focus on multi-modal targets that remain challenging for existing sampling methods. Inspired by the framework of diffusion models developed for generative modeling, we introduce a sampling algorithm based on the simulation of a non-equilibrium deterministic dynamics that transports an easy-to-sample noise distribution toward the target. At the marginal level, the induced density path follows a prescribed stochastic interpolant between the noise and target distributions, specifically constructed to respect the underlying Riemannian geometry. In contrast to related generative modeling approaches that rely on machine learning, our method is entirely training-free. It instead builds on iterative posterior sampling procedures using only standard Monte Carlo techniques, thereby extending recent diffusion-based sampling methodologies beyond the Euclidean setting. We complement our approach with a rigorous theoretical analysis and demonstrate its effectiveness on a range of multi-modal sampling problems, including high-dimensional and heavy-tailed examples.
\end{abstract}

\end{frontmatter}

\section{Introduction}
In this paper, we study the problem of sampling from a \emph{target} probability distribution supported on a state space $\msm$, assuming that its density is available up to a normalization constant. This problem can be viewed as the counterpart of generative modeling: while the objective in both cases is to generate samples from the target, generative modeling typically assumes access to a finite dataset of available samples.

While many classical examples focus on discrete state spaces or Euclidean domains, sampling problems defined on Riemannian manifolds arise frequently. A canonical example is when $\msm$ is the hypersphere of $\rset^{d+1}$ for some dimension $d$, denoted by $\sphere^d$, which appears naturally in applications such as imaging \cite{marignier_posterior_2023} or constrained probabilistic models \cite{lan_spherical_2013}. Somewhat counterintuitively, sampling on $\sphere^d$ may even be preferable to sampling directly on $\mathbb{R}^d$: for instance, when the target has heavy tails or a large effective support, mapping the distribution to the sphere via stereographic projection can provide substantial practical benefits \cite{yang_stereographic_2024}. Further examples include Bayesian inference problems in which parameters inherently lie on non-Euclidean manifolds \cite{graham_manifold_2022, pourzanjani_bayesian_2021}, as well as settings where $\msm$ is defined through constraints that endow it with a Riemannian structure, such as polytopes \cite{noble_unbiased_2023}.

In Euclidean spaces, Markov chain Monte Carlo (MCMC) methods constitute the dominant class of sampling algorithms, and many of them admit principled extensions to Riemannian manifolds; see, for example, \cite{cheng_theory_2023, durmus_geodesic_2024}. Under suitable assumptions on the target distribution, most notably log-concavity, strong theoretical guarantees can be established, including non-asymptotic convergence rates \cite{durmus_non-asymptotic_2016,dwivedi_log-concave_2019}. Nonetheless, the practical performance of these methods often deteriorates when the target is complex. A well-known difficulty arises for multi-modal targets, whose probability mass concentrates around several isolated regions, commonly referred to as ``modes''. This issue is further exacerbated in high-dimensional settings and affects both Euclidean and Riemannian sampling schemes.

In contrast, the challenge of multi-modality has been largely mitigated in generative modeling through the widespread success of score-based diffusion models \cite{song2020score}. These models simulate a non-equilibrium \emph{stochastic} dynamics that gradually transforms a simple noise distribution, \eg, Gaussian, into the data distribution, thereby naturally accommodating highly multi-modal structures. The evolution of this dynamics is designed to follow, at the marginal level, a path of intermediate distributions obtained by applying progressively decreasing Gaussian convolutions to the target. A central ingredient of these models is the score function, namely the gradient of the logarithm of the intermediate densities, which is generally intractable and therefore learned from data using machine learning techniques. Closely related to diffusion models, flow matching methods \cite{lipman_flow_2023} propose an alternative formulation based on \emph{deterministic} dynamics that similarly transport the noise distribution to the target. Empirically, such deterministic flows are often less sensitive to time discretization errors than their stochastic counterparts. Both diffusion and flow matching models rely on score-based quantities and can in fact be unified within the framework of stochastic interpolants \cite{albergo_stochastic_2023}, which provides a principled way to navigate between stochastic and deterministic dynamics that share the same prescribed marginal path. In the paper, we adopt this unifying perspective and terminology.

The remarkable empirical performance of diffusion models in Euclidean generative modeling has recently motivated their adaptation to sampling problems; see, for instance, \cite{he_zeroth-order_2024, huang_faster_2024,grenioux_stochastic_2024}. In this context, the absence of data samples makes score estimation the central challenge. Existing approaches typically adapt classical MCMC techniques to estimate these scores on the fly, \ie, during the simulation of the dynamics, yielding encouraging practical results. In parallel, diffusion-based models on Riemannian manifolds have been the subject of early and extensive investigation \cite{bortoli_riemannian_2022, huang_riemannian_2022}. More recently, the most efficient approaches have relied on flow matching, notably the work of \cite{chen_flow_2024}, which leverages tractable geodesics to construct geometry-aware dynamics while avoiding the direct simulation of manifold-valued diffusion processes. A natural question is whether such method can be adapted to the sampling setting. Despite its appeal, this question has not been addressed so far, and closing this gap is the main objective of the present work. Our main contributions are summarized as follows:
\begin{itemize}
    \item Inspired by \cite{chen_flow_2024}, we develop a rigorous stochastic interpolant framework on a broad class of Riemannian manifolds to construct non-equilibrium dynamics transporting a noise distribution to the target. Tailored to the sampling setting, we establish theoretical results that justify the approach and provide explicit expressions for its core components.
    \item Building on ideas from diffusion-based sampling \cite{grenioux_stochastic_2024}, we introduce a practical methodology to simulate the resulting deterministic Riemannian dynamics. The approach relies on estimating the vector field of the dynamics on the fly, expressed as conditional expectations, using standard MCMC tools. These tools are employed in a setting where they are significantly more effective than when directly used to sample from the target. We refer to the resulting method as \textbf{Flow-based Riemannian Iterative Posterior Sampling} (\frips{}). We further detail the FRIPS algorithm for two representative manifolds, namely $\R^d$ and the $d$-sphere $\sphere^d$, for which we derive efficient simplifications.
    \item Finally, we demonstrate the effectiveness of FRIPS on a range of controlled Riemannian sampling problems, where pronounced multi-modality and high dimensionality present substantial challenges. Our results show that FRIPS consistently outperforms existing Riemannian sampling methods under comparable computational budgets, highlighting its potential for real-world applications.
\end{itemize}

\paragraph*{Notation} We introduce notation used throughout the paper. Let $(\msx,\mcx)$ and $(\msy,\mcy)$ be measurable spaces, let $\mu$ be a measure on $(\msx,\mcx)$, and let $f\colon \msx\to\msy$ be a measurable map. The symbol $\sharp$ denotes the standard \textit{pushforward} operation, defined by
\begin{equation}
A\in\mcy \mapsto f_{\sharp}\mu(A) = \mu\left(f^{-1}(A)\right)\eqsp.
\end{equation}
If a random variable $X$ taking values in $(\msx,\mcx)$ has distribution $\nu$, where $\nu$ is a probability measure on $(\msx,\mcx)$, we write $X\sim\nu$. For $d\geqslant 1$, $\mathcal{N}(0,1)$ denotes the standard normal distribution on $\R^d$, while $\mcu(0,1)$ denotes the uniform distribution on $\ccint{0,1}$.

\section{Riemmanian stochastic interpolant and its Markov projection}
\label{sec:theory}

In this section,  we provide assumptions on the state space
where the measure we wish to sample from is defined. We also introduce the theoretical foundations underlying our sampler and in particular the Riemmanian Markovian projection. More precisely, we begin with a brief overview of Riemannian geometry to formalize our setting. Building on this background, we then present the main concepts associated with Riemannian stochastic interpolants in the specific context of sampling from a given target distribution and its Markovian projection. Although related constructions were previously introduced by \cite{chen_flow_2024} to perform generative modeling, we enrich this framework by providing precise assumptions and theoretical results that establish the validity of the approach in our sampling context. 

\subsection{A brief detour on Riemannian geometry : background and notation}\label{sub:geometry}

In this section, we introduce the necessary concepts from
Riemannian geometry that are needed to describe the sampling problem at hand and our method to address it. Additional details are provided in
\Cref{app:geometry} and we refer to \cite{lee_introduction_2018, chavel_riemannian_2006} for
basics on this topic. We first make the following assumption on the state
space $\msm$.

\begin{assumption}\label{ass:1}
    The state space $\msm$ is a complete, connected, $d$-dimensional, smooth manifold endowed with a smooth Riemannian metric $\metric$.
\end{assumption}

We denote by $\mcb(\msm)$ the Borel $\sigma$-algebra associated to $\msm$.
Assumption~\Cref{ass:1} ensures the existence of a measure on $(\msm, \mcb(\msm))$, referred to as the \emph{Riemannian measure} of $\msm$, which we denote by $\Rvol$ (see \Cref{app:geometry} or \cite[Section II.5]{sakai_riemannian_1996}). Note that this measure reduces to the Lebesgue measure if $\msm$ is a Euclidean space endowed with the trivial metric. In this paper, we consider a target distribution denoted by $\pi_1$, with density with respect to the Riemannian measure $\Rvol$, of the form:
\begin{equation}
    \label{eq:def_target}
    \txts \pi_1(\rmd x) = p_1 (x) \Rvol(\rmd x) \eqsp, \text{ where } p_1(x) = \npi(x)/  \int_{\msm}\npi(y)\Rvol(\rmd y)\eqsp,
\end{equation}
and $\npi$ a non-negative and $\Rvol$-integrable function on $\msm$. Typically, the normalizing constant  $\int_{\msm}\npi(y)\Rvol(\rmd y)$ is unknown, and therefore we suppose here that $p_1$ is only known up to a multiplicative constant. Unless specified otherwise, all other probability densities defined in the rest of this paper are supposed to be defined with respect to $\Rvol$. \\

We denote the \emph{tangent space} at $x\in\msm$ by $T_x\msm$. It is equipped with the inner product provided by the metric $\metric_x$, denoted by $\langle v,w \rangle_{x}$ for $v,w \in T_x\msm$, with the associated norm $\Vert v\Vert_{x}^2 = \ps{v}{v}_{x}$. The \emph{tangent bundle} $T\msm$ is defined to be the disjoint union of those tangent spaces.
Under \Cref{ass:1}, following e.g.,  \cite[Theorems 4.27, 6.19]{lee_introduction_2018}, for any $(x,v) \in T\msm$, there exists a path $\upgamma_{x,v} : \R \to \msm$, with initial value $\upgamma_{x,v}(0) = x$ and velocity $\dot{\upgamma}_{x,v}(0) = v$, solution of the geodesic equation associated to the metric $\metric$: $\nabla_{\dot{\upgamma}} \dot{\upgamma} = 0$, where $\nabla$ is the \emph{Levi-Civita connection} associated to $\metric$. In charts, this corresponds to a second order Ordinary Differential Equation (ODE). Any such path is called a \emph{geodesic}.
For any $(x,y)\in\msm \times \msm$, there exists a \emph{minimizing geodesic} between $x$ and $y$, \ie, a geodesic $\upgamma_x^y : \R \rightarrow \msm$ such that minimizes the functional $\upgamma \mapsto \int_0^1 \norm{\dot{\upgamma}(t)}_{\upgamma(t)} \rmd t$ under the endpoint constraints $\upgamma_x^y(0)=x,\upgamma_x^y(1)=y$, see \cite[Corollary 6.21]{lee_introduction_2018}. The minimum of this functional defines the \emph{geodesic distance} between $x$ and $y$, denoted by $\dist(x,y)$.

For any $x\in\msm$, the \emph{cut locus} at $x$, denoted by $\cut(x)$, corresponds to the set of points $y$ such that there exists more than one minimizing geodesic mapping $x$ to $y$; see \cite[Chapter 10]{lee_introduction_2018} or \cite[Chapter II]{chavel_riemannian_2006}. Intuitively, it corresponds to the set of points for which the shortest path to $x$ is ambiguous. We denote $\msd =  \{(x,y)\in \msm\times\msm \,: \, y \not \in \cut(x) \}$. Then, for every $(x,y)\in\msd$, the minimizing geodesic $\upgamma_x^y$ is unique \cite[Proposition 10.32]{lee_introduction_2018}.

Given a differentiable function $f : \msm \rightarrow \R$, we denote its differential at $x\in\msm$ as $\rmd f_x$. The \emph{Riemannian gradient} $\grad f(x) \in T_x\msm$ is the only tangent vector at $x$ satisfying $\rmd f_x = \langle \grad f(x),\,\cdot\,\rangle_x $.

Any function $u:\msm \to T\msm$ is called a \emph{vector field}. If not specified otherwise, any vector field will be assumed smooth. 
We write $\mathrm{div}(\bfu)$ for the \emph{Riemannian divergence} of a continuously differentiable vector field $\bfu$, which is usually defined using local coordinates: see \Cref{app:geometry}.

Geodesic curves induce the \emph{exponential map} on $\msm$, defined for  $x\in \msm$ and $v \in T_x \msm$ as $\Exp_x(v) = \upgamma_{x,v}(1)$, where $\upgamma_{x,v}$ is the geodesic with initial value $\upgamma_{x,v}(0) = x$ and velocity $\dot{\upgamma}_{x,v}(0) = v$. Given $x \in\msm$, the map $x\mapsto \Exp_x( v)$  is a diffeomorphism once restricted to the \emph{injectivity domain} $\ID(x)\subset T_x\msm$ onto $\msm\setminus \cut(x)$, where
\begin{equation}
  \label{eq:def_id_x}
  \ID(x)=\{v\in T_x\msm\,:\,\Vert v\Vert_{x}< \cuttime_x(v/\Vert v\Vert_{x})\}\eqsp,
\end{equation}
with $\cuttime_x$ denoting the \emph{cut time} function at $x$; see \cite[Proposition 10.32, Theorem 10.34]{lee_introduction_2018}. For any $x\in\msm, v\in T_x\msm$, it is defined as
\begin{equation}
  \label{eq:def_cuttime}
  \cuttime_x(v) = \sup\{t>0\,:\,{\upgamma_{x,v}}_{|\ccint{0,t}}~\text{is minimizing}\}\eqsp.
\end{equation}
 This restriction allows us to define its inverse, referred to as the \emph{logarithm map} $y\in\msm\setminus\cut(x)\mapsto \Log_x(y)$.
As by definition of minimizing geodesics, the following equalities hold: $\dist(x,y) = \Vert \Log_xy\Vert_{x}$ for any $y \in \msm \setminus\cut(x)$ and $\dist(x,\Exp_x v) = \Vert v\Vert_{x}$ for any $v\in\ID(x)$. Furthermore, we have $\dot{\upgamma}_x^y(0) = \Log_xy$. 

\begin{example}[The $d$-sphere example.]\label{ex:sphere} We briefly illustrate the Riemmanian notions introduced above in the case where $\msm$ is the $d$-sphere $\sphere^d=\{x \in\rset^{d+1}\,:\, \norm{x} = 1\}$, with its metric $\metric$ induced by the flat metric of $\R^{d+1}$. We refer to \cite{lee_introduction_2018} for additional details. Here, for every point $x\in\sphere^d$, the tangent space $T_x\sphere^d$ corresponds to the collection of vectors $v\in\R^{d+1}$ such that $\ps{x}{v} = 0$, with $\ps{x}{v}$ being the usual dot product between $x$ and $y$ in $\R^{d+1}$, and the geodesics are great circles. In particular, we can derive the exponential and logarithmic maps as follows: for $x \in\msm$, $v \in T_x\msm$ and $y \in \msm \setminus \cut(x)$,
  \begin{align}
    \label{eq:exp_log_sphere}
 \Exp_xv = x\cos(\Vert v\Vert) + (v/\norm{v})\sin(\Vert v\Vert) \eqsp,
    \quad  \Log_x y=(\widehat{xy}/\sin(\widehat{xy})){y-\cos(\widehat{xy})x}\eqsp,
\end{align}
where $\widehat{xy} = \arccos\ps{x}{y}$ denotes the angle between two points $x$ and $y$. Note that the cut locus of any point $x$ only contains the antipodal point $-x$; hence,
we have $\msd = \{(x,y)\in\sphere^d\times\sphere^d\,:\,\widehat{xy}<\pi\}$. As great circles with length greater than $\pi$ necessarily connect two antipodal points, the cut time at any point $x$, in any direction $v$, is equal to $\pi$, and therefore, we have $\ID(x) = \{v\in T_x\sphere^d\,:\,\Vert v\Vert < \pi\}$ (\ie{}, the sphere of radius $\pi$ in $T_x\sphere^d$).
\end{example}

\paragraph*{Outline of the next subsections} In the rest of the section, we revisit in a rigorous manner the conceptual framework of so-called ``Riemannian stochastic interpolants'' introduced by \cite{chen_flow_2024}, within the specific context of sampling. 

In \Cref{subsec:interpolant}, we first construct an explicit path of un-normalized densities bridging an easy-to-sample \emph{noise} distribution, which we will denote throughout the paper by $\pi_0$, to the target distribution $\pi_1$, and tailored to the Riemannian geometry introduced above. 

In \Cref{subsec:markovian}, we then derive marginal-preserving deterministic dynamics, which critically rely on {denoising posterior distributions} analyzed in \Cref{subsec:posterior}.

\subsection{Riemannian stochastic interpolation}
\label{subsec:interpolant} Following the stochastic interpolant formulation originally formulated by \cite{albergo_stochastic_2023}, we aim to build an {interpolation process}, denoted by
$\processt{X}[\ccint{0,1}]$, such that $X_0$ and $X_1$ are respectively distributed according to $\pi_0$ and $\pi_1$. 
To do so, we align closely with \cite{chen_flow_2024} and specifically define $\processt{X}[\ccint{0,1}]$ by leveraging a coupling between $\pi_0$ and
$\pi_1$ (\eg, independent coupling), and geodesic interpolations on $\msm$.

\paragraph*{Initial boundary coupling} In our framework, $(X_0,X_1)$ is built as a couple of random variables with distribution $\couplingPi$, defined on $\msm^2$ as a coupling between $\pi_1$ and $\pi_1$, \ie, $X_0$ (resp. $X_1$) is marginally distributed according to $\pi_0$ (resp. $\pi_1$). Further, we make the following assumption on $\couplingPi$.
\begin{assumption}\label{ass:abs}
    The coupling $\couplingPi$ between  $\pi_0$ and  $\pi_1$ admits a density $p_{0,1}$ with respect to $\Rvol\otimes\Rvol$.
  \end{assumption}
Note that \Cref{ass:abs} holds, for instance, when both $\pi_0$ and $\pi_1$ are absolutely continuous with respect to $\Rvol$ and the coupling is simply $\couplingPi = \pi_0\otimes \pi_1$ (independent coupling), as done in most of flow matching applications. 

Under \Cref{ass:abs}, the density $p_1$ \eqref{eq:def_target} is a marginal density of $p_{0,1}$: for almost all $x_1$, we have $p_1(x_1) = \int_{\msm} p_{0,1}(x_0,x_1) \Rvol(\rmd x_0)$. Similarly, we denote by $p_0$ the first marginal density defined for $\Rvol$ almost all $x_0$ by  $p_0(x_0) = \int_{\msm} p_{0,1}(x_0,x_1) \Rvol(\rmd x_1)$ for any $x_0 \in\msm$ and . We also denote by $p_{0|1}$ the conditional density of $X_0$ given $X_1$, and defined for $x_0,x_1\in\msm^2$ by $p_{0|1}(x_0|x_1) = p_{0,1}(x_0,x_1)/p_{1}(x_1)$ if $p_1(x_1) \neq 0$, $p_{0|1}(x_0|x_1) = p_0(x_0)$ otherwise.

\paragraph*{Geodesic interpolation} Given $(X_0,X_1)\sim \couplingPi$, a natural candidate for the interpolation process $(X_t)_{t\in [0,1]}$ is given by a {deterministic interpolant} that exploits the minimizing geodesic between $X_0$ and $X_1$. To formalize this idea, we first introduce for any $x_1\in \msm$ and any $t \in \ccint{0,1}$, an interpolation map $\psi_t(\cdot;x_1)$ defined by
\begin{equation}\label{eq:psi_def}
 \psi_t(\cdot;x_1) :    x_0 \in \msm\setminus\mathrm{Cut}(x_1) \mapsto \upgamma_{x_0}^{x_1}(t)=\Exp_{x_1}((1-t)\Log_{x_1}x_0)\eqsp,
\end{equation}
and detail its properties in \Cref{prop:diffeo}.
We denote the image of $y\mapsto \psi_t(y;x_1)$, for $x_1 \in\msm$ by
\begin{equation}
    \label{eq:def_o_t}
    \mso_t(x_1)  = \{ \psi_t(y;x_1) \,: \, y \in \msm \setminus \cut(x_1)\} \eqsp.
\end{equation}
\vfill
\begin{proposition}\label{prop:diffeo}
Assume \Cref{ass:1} and \Cref{ass:abs}. Then, the following results hold.
\begin{enumerate}[label = (\roman*),wide, labelindent=0pt]
    \item
\label{prop_smooth_psi_t_i} For $\Rvol\otimes \Rvol$-almost any $(x_0,x_1)\in \msm^2$, the map $t\in \ccint{0,1}\mapsto  \psi_t(x_0\,;x_1)$ is well-defined and smooth.
\item
\label{prop_smooth_psi_t_ii}
For any $x_1 \in \msm$ and $t\in \coint{0,1}$,
the map $\psi_t(\cdot;x_1) $ is a smooth diffeomorphism from $\mcut$ into its  image $\mso_t(x_1)$, whose inverse writes for any $x_t \in \mso_t(x_1)$,
    \begin{equation}\label{eq:inv_psi}
        \psi_t^{-1}(x_t;x_1) = \Exp_{x_1}\left(\frac{\Log_{x_1}x_t}{1-t}\right)\eqsp.
      \end{equation}
\item
\label{prop_smooth_psi_t_iii}
The set $\mso_t(x_1)$ is an open set of $\msm$, and $x_t\in\mso_t(x_1)$ if and only if
\begin{equation}
\dist(x_1,x_t) < (1-t)\times \cuttime_{x_1}\left(\frac{\Log_{x_1}x_t}{\dist(x_1,x_t)}\right)\,,
\end{equation}
where $\cuttime_{x_1}$ is the cut time at $x_1$, defined in \eqref{eq:def_cuttime}.
\end{enumerate}
\end{proposition}
\begin{figure}
    \centering
    \begin{tikzpicture}

\def\R{1.5}
\def\Rinner{1.0} 
\def\dotSize{2pt}

\begin{scope}[shift={(0,0)}]

\coordinate (x1) at ({cos(60)*\R},{sin(60)*\R});
\coordinate (xt) at ({cos(-30)*\R},{sin(-30)*\R});
\coordinate (x0) at ({cos(-75)*\R},{sin(-75)*\R});

\draw[thick] (0,0) circle (\R);

\draw[red, line width=2pt] ({cos(60-120)*\R},{sin(60-120)*\R}) arc[start angle=-60, end angle=180, radius=\R];

\fill (x1) circle (\dotSize) node[above] {$x_1$};
\fill (xt) circle (\dotSize) node[right] {$x_{1/3}$};
\fill (x0) circle (\dotSize) node[above] {$x_0$};

\draw[blue, ->, thick] ({cos(-75)*\Rinner},{sin(-75)*\Rinner}) arc[start angle=-75, end angle=60, radius=\Rinner];

\end{scope}

\begin{scope}[shift={(6,0)}]

\coordinate (x1) at ({cos(60)*\R},{sin(60)*\R});
\coordinate (xt) at ({cos(-30)*\R},{sin(-30)*\R});
\coordinate (x0) at ({cos(-210)*\R},{sin(-210)*\R});
\coordinate (xt') at ({cos(90)*\R},{sin(90)*\R});

\draw[thick] (0,0) circle (\R);

\draw[red, line width=2pt] ({cos(60-60)*\R},{sin(60-60)*\R}) arc[start angle=0, end angle=120, radius=\R];

\fill (x1) circle (\dotSize) node[above] {$x_1$};
\fill (xt) circle (\dotSize) node[right] {$x_{2/3}$};
\fill (xt') circle (\dotSize) node[above] {$\psi_{2/3}(x_0;x_1)$};
\fill (x0) circle (\dotSize) node[left] {$x_0$};

\draw[blue, ->, thick] ({cos(-210)*\Rinner},{sin(-210)*\Rinner}) arc[start angle=-210, end angle=60, radius=\Rinner];

\end{scope}

\end{tikzpicture}
    \caption{Illustration of the invertibility property of the map $x_t\mapsto\psi_t^{-1}(x_t|x_1)$ in the case $\msm=\sphere^1$. Given $x_1\in\msm$ and $t\in[0,1)$, the image set $\mso_t(x_1)=\{x_t\in\sphere^1\,:\, \widehat{x_tx_1}<(1-t)\pi\}$ is shown in \textcolor{red}{red}, while the geodesic mapping $x_0$ to $x_1$ and passing through $x_t$ at time $t$ is depicted in \textcolor{blue}{blue}. \textbf{(Left)} If $t=1/3$ and $x_{1/3}$ lies in the image, then $x_0=\psi_{1/3}^{-1}(x_{1/3};x_1)$ satisfies $x_{1/3}=\psi_{1/3}(x_0;x_1)$. The geodesic passing through $x_{1/3}$ at time $1/3$ is therefore minimizing. \textbf{(Right)} If $t=2/3$ and $x_{2/3}$ lies outside the image, the geodesic passing through $x_{2/3}$ at time $2/3$ is not minimizing; consequently, $x_0=\psi_{2/3}^{-1}(x_{2/3};x_1)$ satisfies $x_{2/3}\neq\psi_{2/3}(x_0;x_1)$.}
    \label{fig:inverse}
\end{figure}
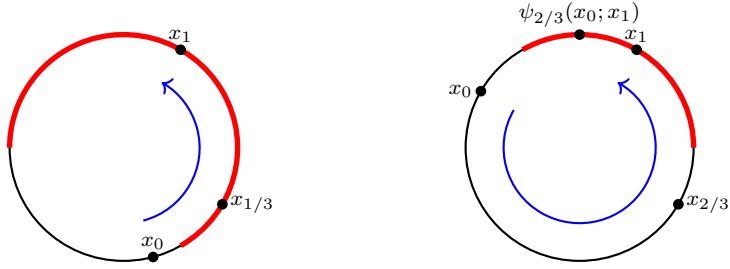

\begin{proof}
The proof of this proposition is postponed to \Cref{proof:diffeo}.
\end{proof}
\Cref{prop:diffeo} shows that for any time $t\in\coint{0,1}$, any $x_1\in\msm$, and any $x_t\in\mcut$, $\psi_t^{-1}(x_t;x_1)$ is the unique point $x_0\in\mcut$ such that the minimizing geodesic $\upgamma_{x_0}^{x_1}$ passes through $x_t$ at time $t$. Note that the expression \eqref{eq:inv_psi} is still well-defined if $x_t\notin\mso_t(x_1)$, however it is no longer the inverse of $x_0\mapsto \psi_t(x_0;x_1)$, which we illustrate in \Cref{fig:inverse} in the case of the sphere. Based on these results, we set for every $t \in \ccint{0,1}$,
\begin{equation}\label{eq:interpolation_def}
    X_t = \psi_t(X_0;X_1) \eqsp, \quad (X_0,X_1) \sim \couplingPi \eqsp.
\end{equation}
Note that \Cref{prop:diffeo} implies that $\PP(X_0\in\cut(X_1)) = \couplingPi(\msd^{\complement}) =  0$; hence the interpolation process  $\processt{X}[\ccint{0,1}]$ is almost surely well-defined. We define the time marginal distributions of $\processt{X}[\ccint{0,1}]$ by $(\pi_t)_{t\in\ccint{0,1}}$, \ie, $\pi_t(\msa) = {\psi_t}_{\sharp}\couplingPi(\msa) = \PP(\psi_t(X_0;X_1)\in\msa)$ for every $\msa\in\mcb(\msm)$, thereby obtaining a smooth distribution path bridging $\pi_0$ to $\pi_1$ on time interval $[0,1]$. For any $t\in [0,1]$, we further denote by $p_t$ the density of $\pi_t$ with respect to $\Rvol$.

\subsection{Riemannian Markovian projection}
\label{subsec:markovian}
As such, the interpolation process $(X_t)_{t\in [0,1]}$ defined in \eqref{eq:interpolation_def} cannot be used to sample from $\pi_1$, as it requires by definition to be able to sample from $\pi_1$. Following \cite{ chen_flow_2024, lipman_flow_2024}, we aim to build a Markov process $\processt{X^{\mathrm{M}}}[\ccint{0,1}]$ that shares the same time-marginal distributions as $\processt{X}[\ccint{0,1}]$, but is defined through the flow of a
time-inhomogeneous manifold-valued ODE of the form
\begin{equation}\label{eq:ode}
    \dot{x}_t  = \bfv_t(x_t) \eqsp,
\end{equation}
where $\bfv : (t,x)\in \coint{0,1}\times \msm \mapsto \bfv_t(x) \in T_x\msm$
is a continuous function, referred to as a
\emph{time-dependent vector field}. If ODE \eqref{eq:ode} admits
solutions on $\ccint{0,1}$, for any starting point $x_0\in\msm$, then the differential flow associated with this ODE is defined as the map
$\phi^{\bfv} : (t,x)\in \Omega \mapsto \phi^{\bfv}_t(x) \in \msm$,
such that for any $x \in \msm$,
$t \in \ccint{0,1} \mapsto \phi^{\bfv}_t(x)$ solves \eqref{eq:ode}
with initial condition $\phi^{\bfv}_0 \equiv \Id$, where $\Id$ denotes
the identity map. Note that in general, similarly as in the Euclidean case,
assuming that $\bfv$ is locally Lipschitz only ensures existence of
{local} solutions; see \cite[Theorem
9.48]{lee_introduction_2012}.

Importantly, existence of stochastic solutions to \eqref{eq:ode}  with arbitrarily prescribed marginals $(\mu_t)_{t\in\coint{0,1}}$ is guaranteed under appropriate conditions summarized now, restated from \cite[Chapter 1]{villani_optimal_2009}. A similar result can be found in \cite[Proposition 8.1.8]{ambrosio_gradient_2005} for the Euclidean case.
   Consider $(\mu_t)_{t\in\coint{0,1}}$ a family of probability measures defined on $\msm$,
   and   a time-dependent vector field  $(t,x)\in \coint{0,1}\times \msm\mapsto \bfv_t(x)$ that satisfy the following conditions.
   \begin{enumerate}[label=$\mathbf{C}$\arabic*),wide,leftmargin=*,labelwidth=!,labelindent=0pt]
   \item  \label{ass:continu_equation_i} The vector field $\bfv$ is continuous with respect to $t$ and locally Lipschitz continuous with respect to $x$.
   \item \label{ass:continu_equation_i_0} It holds $       \int_0^1\int_{\msm} \Vert\bfv_t(x)\Vert_{x}\mu_t(\rmd x)\rmd t < \infty$.
   \item  \label{ass:continu_equation_ii} The family  $(\mu_t)_{t\in\coint{0,1}}$  satisfy the continuity equation on $[0,1]$ associated to $\bfv$ in the weak sense
     \begin{equation}\label{eq:mass}
       \partial_t \mu_t(x) +\mathrm{div}(\mu_t(x) \bfv_t(x)) = 0\eqsp,
     \end{equation}
 \ie, for any smooth function $f:\ooint{0,1} \times \msm\to \rset$ with compact support, we have
    \begin{equation}\int_0^1\int_{\msm} \partial_t f(t,x) \mu_t(\rmd x) \rmd t  + \int_0^1\int_{\msm} \ps{\grad_x f(t,x)}{\bfv_t(x)}_{x} \mu_t(\rmd x) \rmd t= 0 \eqsp,
     \end{equation}
    where  $\grad_xf(t,x)$ denotes the Riemannian gradient of $y \mapsto f(t,y)$ at $x$.
   \end{enumerate}
Then, \eqref{eq:ode} admits a solution on $\coint{0,1}$ for $\mu_0$-almost all starting point $x_0$, and $\mu_t = {\phi_t}^{\bfv}_{\sharp}\mu_0$ holds for every $t\in\coint{0,1}$.

Based on this result, we propose designing a time-dependent vector
field satisfying conditions \ref{ass:continu_equation_i}--\ref{ass:continu_equation_ii} for the bridging distributions $(\pi_t)_{t\in [0,1]}$ introduced in \Cref{subsec:interpolant}. Indeed, equipped with such a vector field denoted by $\bfu$, there exists almost surely $\processt{X^{\rmM}}[\coint{0,1}]$ solution to the so-called \emph{probability flow} ODE
\begin{equation}
  \label{eq:defu_markov_proj}
  \dot{X}_t^{\rmM} = \bfu_t(X_t^{\rmM}) \eqsp, \quad X_0^{\rmM} \sim \pi_0\eqsp,
\end{equation}
 that has the same marginal distributions as the interpolation process $\processt{X}[\ccint{0,1}]$ defined in \eqref{eq:interpolation_def}. We refer to this solution as its Markovian projection. Remarkably, it can be approximately simulated using numerical methods, in order to obtain approximate samples from the target
distribution $\pi_1$ at time $t=1$, starting from samples drawn from $\pi_{t_0}$ for
any $t_0 \in \coint{0,1}$. While $t_0$ is most often chosen close to $0$ in generative modeling applications, this is not the case in the sampling methodology that is further developed here, as explained in \Cref{sec:riem-flow-iter}. In the rest of this section, we make explicit a valid candidate for $\bfu$.

\paragraph*{Construction of $\bfu$ and justification}

We start with a first lemma that ensures the existence and provides a practical expression of the time-derivative of the stochastic interpolant $(X_t)_{t\in [0,1]}$.

\begin{lemma}
\label{lem:condi_vector_field}
 Assume \Cref{ass:1} and \Cref{ass:abs} and let $(X_0,X_1) \sim \pi_{0,1}$.  The conditional vector field $t\in\ccint{0,1}\mapsto \dot{X}_t = \dot{\psi_t}(X_0;X_1) \in T_{X_t}\msm$ is almost surely well-defined and in addition,  it holds almost surely,
    \begin{equation} \label{eq:lem:condi_vector_field}
    \dot{X}_t = \frac{\Log_{X_t}X_1}{1-t}\eqsp.
\end{equation}
\end{lemma}
\begin{proof}
The proof of this proposition is postponed to \Cref{proof:log}.
\end{proof}

For each time pair $(s,t) \in \ccint{0,1}^2$, we define the Markov kernel corresponding to the conditional distribution of $X_t$ given $X_s$, by $(x_s,\msa) \in\msm\times \mcb(\msm) \mapsto \pi_{t|s}(\msa|x_s)$, such that $\pi_{t|s}(\msa|X_s) = \PP(X_t\in\msa|X_s)$, almost surely, and denote if it exists by $p_{t|s}$ their density. Among these conditional distributions, the so-called \emph{denoising} posterior distributions corresponding to the family  $\{\pi_{1|t}\}_{t\in [0,1]}$, are key to derive marginal-preserving dynamics, as we  see next. In particular, we provide explicit expressions of their density in \Cref{subsec:posterior}.

Based on \Cref{lem:condi_vector_field}, for any $t\in\ccint{0,1}$,  we define the vector field $\bfu_t$ for $\pi_t$-almost every $x_t$ by
\begin{equation}\label{eq:u_def}
\bfu_t(x_t) = \mathbb{E}[\dot{X}_t|X_t=x_t]=\int_{\msm} \frac{\Log_{x_t}x_1}{1-t} \pi_{1|t}(\rmd x_1|x_t) \eqsp,
\end{equation}
and $\bfu_t = 0$ elsewhere. In order to show that $\bfu$ and $(\pi_t)_{t\in\ccint{0,1}}$ jointly satisfy \ref{ass:continu_equation_i}-\ref{ass:continu_equation_ii}, 
we make the following additional assumption.
\begin{assumption}\label{ass:finite}
\begin{enumerate*}[label=(\roman*)]
\item\label{ass:finite_iii} The time-dependent vector field $(t,x) \mapsto \bfu_t(x)$ defined in \eqref{eq:u_def} is locally Lipschitz-continuous with respect to $x$ and continuous with respect to $t$.
\item\label{ass:finite_ii} The coupling $\pi_{0,1}$ satisfies $\PE\left[\dist(X_0,X_1)\right] < +\infty$.
\end{enumerate*}
\end{assumption}

\Cref{ass:finite}-\ref{ass:finite_iii} is mainly a technical condition that we expect to be verified case by case. We show that it holds under mild regularity assumptions on $p_1$ in the case $\msm = \sphere^d$ and $\msm=\R^d$ in \Cref{app:sphere,app:euclidean}. Note that \Cref{ass:finite}-\ref{ass:finite_ii} holds, for instance, when $\msm$ is compact, or, in the Euclidean case, when $\pi_0$ and $\pi_1$ have finite first moments. In addition,
\Cref{ass:finite}-\ref{ass:finite_ii}  ensures that
we have for every $t\in\ccint{0,1}$,
\begin{equation}
\int_{\msm} \Vert\bfu_t(x)\Vert_{x}\pi_t(\rmd x) = \PE\left[\Vert\bfu_t(X_t)\Vert_{X_t}\right] \leq \PE\left[\PE\left[\Vert\dot{X}_t\Vert_{X_t}|X_t\right]\right]
    = \PE\left[\dist(X_0,X_1)\right]\eqsp\eqsp,
\end{equation}
where we first used the conditional Jensen inequality and then relied on the constant speed property of geodesics \cite[Corollary 5.6]{lee_introduction_2018} implying that $ \Vert\dot{X}_t\Vert_{X_t} = \Vert\dot{\upgamma}_{X_0}^{X_1}(t)\Vert_{X_t} = \Vert\dot{\upgamma}_{X_0}^{X_1}(0)\Vert_{X_0} = \dist(X_0,X_1)$. Therefore
\Cref{ass:finite}-\ref{ass:finite_ii} implies that  $\bfu$ and $(\pi_{t})_{t\in\ccint{0,1}}$ satisfy \ref{ass:continu_equation_i_0}.

With those additional assumptions, we can establish the following result.
\begin{proposition}
  \label{prop:verify_prop_cont}
  Assume \Cref{ass:1}, \Cref{ass:abs} and \Cref{ass:finite}. Then, the vector field $\bfu$, defined by \eqref{eq:u_def}, and the bridging distributions $(\pi_{t})_{t\in\ccint{0,1}}$ satisfy the continuity equation derived in \ref{ass:continu_equation_ii}.
\end{proposition}

\begin{proof}
The proof of this result is postponed to \Cref{proof:mass}.
\end{proof}

 As a matter of fact, we directly obtain the following result on the Markovian projection of $(X_t)_{t\in [0,1]}$.

\begin{proposition}\label{prop:markov_proj}
  Assume \Cref{ass:1}, \Cref{ass:abs} and \Cref{ass:finite}. For any $t\in\ccint{0,1}$ we have $\pi_t = {\phi^{\bfu}_t}_{\sharp}\pi_0$, with $\phi^{\bfu}$ being the flow associated to the time-dependent vector field $\bfu$ defined in \eqref{eq:u_def}. As a consequence, the process $\processt{X^{\rmM}}[\ccint{0,1}]$ defined by the deterministic dynamics \eqref{eq:defu_markov_proj} has the same time marginal distributions as the interpolation process $\processt{X}[\ccint{0,1}]$, \ie, $X_t^M\sim\pi_t$ for any $t\in [0,1]$.
\end{proposition}

\begin{proof}
    This is an immediate corollary of \Cref{prop:verify_prop_cont}.
\end{proof}

In generative modeling applications, the time-dependent vector field $\bfu$ given in \eqref{eq:u_def} is learnt by a neural network by leveraging a mean-square error formulation, using available samples $x_1$ from $\pi_1$. In the context of sampling, such samples do not exist. Following \cite{grenioux_stochastic_2024}, we propose to approximate $\bfu$ using Monte Carlo methods, by intermediately sampling from the {denoising posterior distributions} $\{\pi_{1|t}\}_{t\in [0,1]}$ that appear in \eqref{eq:u_def}; see \Cref{sec:riem-flow-iter}. To apply this strategy, we characterize their density $\{p_{1|t}\}_{t\in[0,1]}$ in the next section.

\subsection{Expression of denoising posterior densities}\label{subsec:posterior} Let $t\in [0,1)$.  Assuming that the conditional distribution $\pi_{t|1}$ admits a density $p_{t|1}$ with respect to $\Rvol$, then for any $x_t\in \msm$, it follows from the Bayes' rule that the density of $\pi_{1|t}(\cdot|x_t)$ is given by
\begin{equation}\label{eq:bayes}
    p_{1|t}(x_1|x_t) =\frac{p_{t|1}(x_t|x_1)\npi(x_1)}{\int_\msm p_{t|1}(x_t|y)\npi(y)\rmd y} \eqsp.
\end{equation}
Based on this observation, characterizing $p_{1|t}$ actually amounts to show that $\pi_{t|1}$ is absolutely continuous with respect to $\Rvol$  and identifying $p_{t|1}(\cdot|x_1)$ for any $x_1\in \msm$, which is the purpose of this section.

To this end, we remark that
\begin{equation}\label{eq:p_t|1}
    \pi_{t|1}(\cdot|x_1)={\psi_t(\cdot;x_1)}_{\sharp}\pi_{0|1}(\cdot|x_1) \eqsp,
\end{equation}
where $\psi_t$ is the geodesic interpolant introduced in \eqref{eq:psi_def}. Hence, we aim to use a change-of-variable formula via the diffeomorphism $\psi_t(\cdot;x_1)$ defined from $\mcut$ onto $\mso_t(x_1)$ in \eqref{eq:p_t|1}, which requires the introduction of further concepts from integration on Riemannian manifolds; see \eg, \cite[Chapter III]{chavel_riemannian_2006} for details.

Fix a point $z\in\msm$. Recall that, under \Cref{ass:1}, the map $\Exp_z$ is a diffeomorphism from $\ID(z)\subset T_z\msm$, defined in \eqref{eq:def_id_x} onto $\msm\setminus\cut(z)$. By restriction to $\ID^*(z) = \ID(z)\setminus\{0\}$, we obtain a diffeomorphism onto $\msm^z = \msm\setminus\{\cut(z) \cup \{z\}\}$. Denote $S_z\msm\subset T_z\msm$ the unit sphere of $T_z\msm$.
Then, it holds that $  \ID^*(z) = \{  t\xi\,:\,\xi\in S_z\msm, 0< t<\cuttime_z(\xi)\}$.

Let $\Dmsm{z}\subset   S_z\msm \times \ooint{0,+\infty}$ denotes the set of pairs $(\xi,t)$ such that $t\xi\in\ID^*(z)$.
The map $(\xi,t) \mapsto \varphi_z(\xi,t) = \Exp_z t\xi$ is a diffeomorphism from $\Dmsm{z}$ onto $\msm^z$,   
whose inverse writes
\begin{equation}
    x\in\msm^z\mapsto \varphi^{-1}_z(x) = \left(\frac{\Log_z x}{\dist(z,x)},\dist(z,x)\right)\eqsp,
\end{equation}
\ie{} the distance from $x$ to $z$ and  the unit vector in $S_z\msm$ proportional to the initial velocity vector of the minimizing geodesic $\upgamma_z^x$. Since $\Rvol(\cut(z)\cup\{z\})=0$, almost every point in $\msm$ can be parameterized in this way, referred to as its \emph{spherical coordinates with base point $z$}: see \Cref{fig:spherical} for a simple illustration. In particular, this transformation allows to make explicit any integral against $\Rvol$ using this diffeomorphism, as explained next.
\begin{figure}
    \centering
    \begin{tikzpicture}[scale=2.2,>=Stealth]

\def\R{1}       
\def\angz{130}  
\def\angx{50}   

\coordinate (O) at (0,0);
\coordinate (Z) at ({\R*cos(\angz)},{\R*sin(\angz)});
\coordinate (X) at ({\R*cos(\angx)},{\R*sin(\angx)});

\coordinate (Tdir) at ({-sin(\angz)},{cos(\angz)});

\coordinate (XiPos)  at ($(Z) + 1*(Tdir)$);  
\coordinate (XiNeg)  at ($(Z) - 1*(Tdir)$);  

\draw[thick,black] (\R,0) arc[start angle=0,end angle=180,radius=\R];

\draw[blue,thick] ($(Z)-1.2*(Tdir)$) -- ($(Z)+1.2*(Tdir)$);
\node[blue,above left] at ($(Z)-0.8*(Tdir)$) {$T_z\mathsf{M}$};

\filldraw[red] (XiPos) circle (0.7pt);
\filldraw[red] (XiNeg) circle (0.7pt);

\node[red,below right] at (XiNeg) {$\xi\in S_z\mathsf{M}$};

\draw[green!70!black,thick]
    ({\R*cos(\angz)},{\R*sin(\angz)}) arc[start angle=\angz,end angle=\angx,radius=\R];

\filldraw[black] (Z) circle (0.8pt) node[above left,xshift=-1pt] {$z$};

\filldraw[black] (X) circle (0.8pt) node[above right,xshift=2pt] {$x=\mathrm{Exp}_z t\xi$};

\path let \p1=($(Z)!0.5!(X)$) in
    node[green!70!black] at (\p1) {$t = \widehat{zx}$};

\end{tikzpicture}
    \caption{Illustration of the spherical coordinates with a base point $z$ on the manifold $\msm=\sphere^1$. Apart from $z$ and $-z$, every point $x\in\msm$ is associated with a unique pair $(t,\xi)$ such that $x=\Exp_z t\xi$, with $\xi$ (colored in \textcolor{red}{red}) being a unit vector of $T_z\msm$ (here, there are only two such vectors), and $t=\dist(z,x)=\widehat{zx}$ (colored in \textcolor{green}{green}).}
    \label{fig:spherical}
\end{figure}
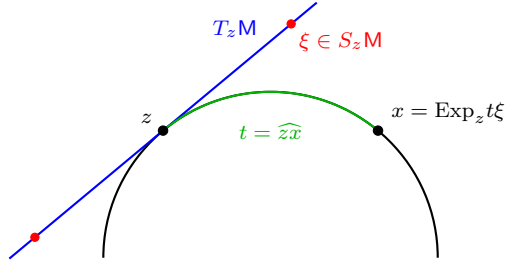

For any $\xi\in S_z\msm$ and $t\geq 0$, let $\tau_{z,\xi}(t)\colon T_z\msm \to T_{\upgamma_{z,\xi}(t)}\msm$ denote parallel transport along the geodesic $\upgamma_{z,\xi}$; see \cite[Theorem 4.32]{lee_introduction_2018}. We then define, for each $\xi\in S_z\msm$ and $t\geq 0$, the linear map $\mcr_{z,\xi}(t)\colon T_z\msm \to T_z\msm$ by
\begin{equation}
    \mcr_{z,\xi}(t)v = \tau_{z,\xi}(t)^{-1} \parentheseDeux{ \curv\left(\dot{\upgamma}_{z,\xi}(t),\tau_{z,\xi}(t) v\right)\, \dot{\upgamma}_{z,\xi}(t)}\eqsp,
  \end{equation}
for any $v\in T_z\msm$, where $\curv$ denotes the \emph{Riemannian curvature tensor} \cite[Section I.4]{chavel_riemannian_2006}.
For any $\xi\in S_z\msm$, we consider the associated linear second-order ODE
\begin{equation}\label{eq:ode_A}
    \ddot{\scrA}_{z,\xi} + \mcr_{z,\xi}(t)\scrA_{z,\xi} = 0\eqsp,
\end{equation}
with initial conditions $\scrA_{z,\xi}(0)=0$ and $\dot{\scrA}_{z,\xi}(0)=\Id$.

Using the symmetries of the curvature tensor (see \cite[Equation I.4.4]{chavel_riemannian_2006}), the operator $\mcr_{z,\xi}(t)$ vanishes on $\R\xi$ for all $t\geq 0$. Consequently, any solution $\scrA_{z,\xi}$ of \eqref{eq:ode_A} can be viewed as an endomorphism of $\xi^{\perp}\subset T_z\msm$, the orthogonal complement of $\R\xi$. Denoting by $(t,\xi)\in (0,+\infty)\times S_z\msm \mapsto \scrA_{z,\xi}(t)$ the solution of \eqref{eq:ode_A} associated with $\xi$, we obtain the following change-of-variable formula \cite[Theorem III.3.1]{chavel_riemannian_2006}: for any $z\in\msm$ and any $\Rvol$-integrable function $f$,
\begin{equation}\label{eq:change-variable}
        \int_{\msm} f(x)\Rvol(\rmd x) = \int_{S_z\msm} \mu_{z}(\rmd \xi) \int_0^{\cuttime_z(\xi)} \rmd t \,\, f(\Exp_z t\xi) \det\scrA_{z,\xi}(t)  \eqsp,
    \end{equation}
where $\mu_z$ denotes the Riemannian measure on $S_z\msm$ induced by the Lebesgue measure on $T_z\msm$, identified with $\R^d$ via a choice of orthonormal basis.

In \Cref{app:geometry}, we provide a detailed derivation of \eqref{eq:change-variable} and emphasize the connection between $\scrA$, the differential of the exponential map, and Jacobi fields. In particular, the function $(t,\xi)\mapsto \det \scrA_{z,\xi}(t)$, or equivalently, the Jacobian determinant of the exponential map, admits a closed-form expression in locally symmetric spaces \cite{chevallier_exponential-wrapped_2022}. 
When $\msm$ has constant sectional curvature, the computation of $\det \scrA_{z,\xi}(t)$ for all $(t,\xi)\in\ooint{0,+\infty}\times S_z\msm$ is immediate: it equals $\sin(t)^{d-1}$ on $\sphere^d$ and $t^{d-1}$ on $\R^d$; see \cite[Section III.1]{chavel_riemannian_2006}.

Following \eqref{eq:change-variable}, we are able to make explicit the conditional density $p_{1|t}$, as shown in the following proposition.
\begin{proposition}\label{prop:density}
  Assume \Cref{ass:1} and \Cref{ass:abs}. Then, for any $t\in\coint{0,1}$ and $x_1\in\msm$, $\pi_{t|1}(\cdot|x_1)$ has a density with respect to $\Rvol$, defined by
    \begin{equation}\label{eq:pt1}
        x_t\in \mcut \mapsto p_{t|1}(x_t|x_1) = p_{0|1}(\psi_t^{-1}(x_t;x_1)|x_1)\Jac_t(x_t;x_1)\1_{\mso_t(x_1)}(x_t)\eqsp,
    \end{equation}
    $\Rvol$-almost everywhere, with
     \begin{equation}\label{eq:jac_a}
    \Jac_t(x_t;x_1) = \frac{1}{1-t}\frac{\det \scrA_{x_1,\xi_t}\left(\frac{\dist(x_t,x_1)}{1-t}\right)}{\det \scrA_{x_1,\xi_t}\left(\dist(x_t,x_1)\right)}\eqsp,
\end{equation}
$\xi_t$ denoting $\Log_{x_1}x_t/\dist(x_t,x_1)$.
    If we further assume that the conditional density $x_0\mapsto p_{0|1}(x_0|x_1)$
  is smooth for any $x_1 \in\msm$, then $x_t\mapsto p_{t|1}(x_t|x_1)$ is smooth on $\mso_t(x_1)$ for any $t \in \coint{0,1}$.
 
\end{proposition}

\begin{proof}
    The proof of this result is postponed to \Cref{proof:spherical}.
\end{proof}

In particular, the term $\Jac_t(\cdot;x_1)$ introduced in \eqref{eq:jac_a} coincides with the Jacobian determinant of the inverse map $\psi_t^{-1}(\cdot;x_1)$, as prescribed by the standard change-of-variable formula. We stress that the conditional density $p_{t|1}(\cdot|x_1)$ may be supported on a \emph{strict} subset of $\msm$ whenever $\mso_t(x_1)\neq \msm$. According to \Cref{prop:diffeo}, this situation arises precisely when $\ID(x_1)\neq T_{x_1}\msm$, \ie{} when the exponential map $\Exp_{x_1}$ is not a global diffeomorphism. This phenomenon is unavoidable on any non-trivial manifold: under \Cref{ass:1}, if $\Exp_{x_1}$ were a global diffeomorphism, then $\msm$ would be diffeomorphic to its tangent space $T_{x_1}\msm$, and hence to $\R^d$. The support restriction reflects a genuinely geometric aspect of the interpolation procedure, intrinsic to the non-Euclidean setting.
As an illustration, when $\msm=\sphere^d$, $p_{t|1}(x_t|x_1)=0$ holds whenever $\widehat{x_tx_1}\geq (1-t)\pi$.

Although the expression in \eqref{eq:pt1} is generally intractable for arbitrary couplings $\pi_{0,1}$ and manifolds $\msm$, we focus in this work on settings where it can be handled explicitly. One such case is obtained by choosing $\pi_{0,1}$ as the independent coupling and taking $\pi_0$ to be the uniform distribution on $\msm$, assumed to be compact. In this situation, the conditional density reduces to $p_{0|1}(\cdot|x_1)\propto \1_{\msm}$. An illustration of this construction for $\msm=\sphere^2$ is provided in \Cref{fig:pt1}.
\begin{figure}
    \centering
\includegraphics[width=0.6\linewidth]{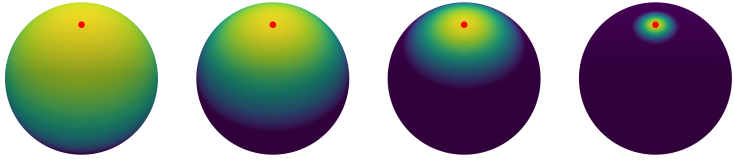}
    \caption{Illustration of  $x_t\mapsto p_{t|1}(x_t|x_1)\propto\Jac_t(x_t|x_1)\1_{\mso_t(x_1)}(x_t)$ for different values of $t$, in the case where $\msm=\sphere^2$, with $x_1\in \msm$ being displayed in \textcolor{red}{red}. \textbf{(From left to right)}: $t=0.3,0.5,0.7,0.9$. While $p_{t|1}(\cdot|x_1)$ is close to the uniform distribution for small times, we observe that it concentrates more and more around $x_1$ as $t$ increases, since its support gradually shrinks.}
    \label{fig:pt1}
\end{figure}

\section{FRIPS methodology}
\label{sec:riem-flow-iter}
Relying on the theoretical concepts elaborated in \Cref{sec:theory}, we present a generic sampling algorithm that outputs approximate samples from a target distribution $\pi_1$ of the form \eqref{eq:def_target}, by simulating the Markovian process $\processt{X^{\rmM}}_{t\in\ccint{0,1}}$ described in \eqref{eq:defu_markov_proj}. Since the velocity field $\bfu$ driving the underlying ODE is not tractable, we propose to adopt a training-free strategy previously used in Euclidean space \cite{huang_reverse_2024,he_zeroth-order_2024, vacher_polynomial_2025, grenioux_stochastic_2024}: we estimate $\bfu_t(X_t^{\rmM})$ \emph{on-the-fly} by (i) sampling from the intermediate denoising posterior distributions evoked in \Cref{subsec:posterior} and (ii) computing an empirical approximation of the expectation in \eqref{eq:u_def}. We coin this methodology as \emph{Flow-based Riemannian Iterative Posterior Sampling} (FRIPS). We specify its main generic features in \Cref{subsec:generic_method}, before describing versions that can be implemented in practice in \Cref{sec:sampling}. For the rest of the paper, we assume that $\msm$ satisfies \Cref{ass:1} (illustrated by the running example of the sphere), that $\pi_{0,1}$ is the independent coupling and $\pi_0, \pi_1$ are absolutely continuous with respect to $\Rvol$, thus satisfying \Cref{ass:abs}. We assume that the $\Exp$ and $\Log$ maps are available in closed-form, as well as the Jacobian factor \eqref{eq:jac_a}.

\subsection{Generic version of FRIPS}
\label{subsec:generic_method}
In principle, one can get a sample from $\pi_1$ by following the probability flow ODE \eqref{eq:defu_markov_proj}, starting from a sample $x_0\sim\pi_0$. Unfortunately, even if the vector field $\bfu$ was perfectly known, integrating exactly this ODE is not feasible; henceforth, we need to rely on some discretization schemes. For ease of presentation, we restrict our methodology to the use of the Euler scheme, but generalization to the Exponential Integration scheme \cite{durmus_quantitative_2015} or higher-order methods could also be applied; see \eg{} \cite{munthe-kaas_high_1999, crouch_numerical_1993} regarding numerical integration on manifolds. Given a sequence of $K\geq 1$ increasing timesteps $0 \leq t_0 \leq \cdots \leq t_K \leq 1$, with $t_K$ being close to $1$, applying the Euler scheme to the ODE \eqref{eq:defu_markov_proj} amounts to define the sequence of iterates $\{X_k\}_{k=0}^{K}$ with $X_{0} \sim \pi_{t_0}$ and for any $k\in \{0, ..., K-1\}$
\begin{equation}
  \label{eq:def_euler_ideal}
  X_{k+1} = \Exp_{X_{k}} (h_{k+1} \bfu_{t_k}(X_{k})) \eqsp,
\end{equation}
where $h_{k+1} = t_{k+1} - t_k$. In practice, we aim to simulate this specific discretized sequence, which requires the simultaneous ability to (i) estimate precisely the velocity field $\bfu$ at times $\{t_k\}_{k=0}^{K-1}$ and (ii) sample the initial state $X_{0}\sim \pi_{t_0}$ accurately for any chosen starting time $t_0$. We respectively treat these two main aspects in the next paragraphs, and summarize the generic version of FRIPS in \Cref{alg:1}.

\begin{algorithm}
\caption{FRIPS : generic version}
\label{alg:1}
\begin{algorithmic}[1]
  \Input Initial time $t_0\in\coint{0,1}$, timesteps $\{t_k\}_{k=1}^K$
\Output an approximate sample $X_{K}$ from  $\pi_1$
\State Sample $X_{0} = \texttt{RLA-Init}(t_0)$, see \Cref{alg:ula}, or $X_{0} = \texttt{Pseudo-IMH-Init}(t_0)$, see \Cref{alg:pseudo_mcmc}
\For{$k = 0$ to $K-1$}
    \State Estimate $\hat{\bfu}_{t_k}({X}_k) = \texttt{Velocity}(t_{k},X_{k})$, see \Cref{alg:mean}
    \State Set $X_{k+1} = \Exp_{X_{k}} (h_{k+1}\hat{\bfu}_{t_k}(X_k))$, see \eqref{eq:def_euler_ideal}
\EndFor
\State \Return $X_K$
\end{algorithmic}
\end{algorithm}
\begin{algorithm}
\caption{\texttt{Velocity}$(t,x_t)$ : generic estimation of $\bfu_t(x_t)$}
\label{alg:mean}
\begin{algorithmic}[1]
\Input Time $t\in\coint{0,1}$, point $x_t\in\msm$, number of Monte Carlo samples $N\geq 1$
\State Get weights and samples $\{w^i, X_1^i\}_{i=1}^N$ via $\texttt{PosteriorSampling}(t,x_t)$, see \Cref{alg:sample_ula,alg:rejection,alg:importance}
\State \Return $\sum_{i=1}^N w^i \Log_{x_t}X_1^i/(1-t)$
\end{algorithmic}
\end{algorithm}
\begin{algorithm}
\caption{$\texttt{Score}(t,x_t)$ : generic estimation of $\grad \log p_t (x_t)$}
\label{alg:score}
\begin{algorithmic}[1]
\Input Time $t\in\coint{0,1}$, point $x_t\in\msm$, number of samples Monte Carlo $N\geq 1$
\State Get weights and samples $\{w^i, X_1^i\}_{i=1}^N$ via $\texttt{PosteriorSampling}(t,x_t)$, see \Cref{alg:sample_ula,alg:rejection,alg:importance}
\State \Return $\sum_{i=1}^N w^i \grad\log p_{t|1}(x_t|X_1^i)$, based on \eqref{eq:pt1}
\end{algorithmic}
\end{algorithm}

\begin{algorithm}
\caption{$\texttt{RLA-Init}(t_0)$ : generic RLA initialization of FRIPS at $t_0$}
\label{alg:ula}
\begin{algorithmic}[1]
\Input Time $t_0\in\coint{0,1}$, step-size $\delta >0$, number of RLA steps $M\geq 1$
\Output approximate sample ${X}^M_{0}\in\msm$ from $\pi_{t_0}$
\State Set an initial point ${X}^0_0\in\msm$
\For{$m = 0$ to $M-1$}
    \State Compute $v = \texttt{Score}(t_0,{X}^{m}_0)$, see \Cref{alg:score}
    \State Draw $z\in T_{{X}_0^{m}}\msm$ from $\mathcal{N}(0,I_d)$ 
    (identifying $T_{{X}_0^{m}}\msm \approx \R^d$ by picking an orthonormal basis)
    \State Set ${X}_0^{m+1}=\Exp_{{X}_0^{m}}(\delta v +\sqrt{\delta}z)$ (Riemannian Langevin update)
\EndFor
\State \Return ${X}_0^M$
\end{algorithmic}
\end{algorithm}

\begin{algorithm}
\caption{$\texttt{Pseudo-IMH-Init}(t_0)$ : pseudo-IMH initialization of FRIPS at $t_0$}
\label{alg:pseudo_mcmc}
\begin{algorithmic}[1]
\Input Time $t_0\in\coint{0,1}$, number of IMH steps $M\geq 1$
\Output approximate sample ${X}^M_{0}\in\msm$ from $\pi_{t_0}$
\State Set an initial point ${X}^0_0\in\msm$
\For{$m = 0$ to $M-1$}
    \State Draw $\tilde{X}_0^{m+1}\sim p_0$
    \State Get un-normalized weights $\{\tilde{w}_m^i\}_{i=1}^N$ via $\texttt{IS-PosteriorSampling}(t_0,X_0^m)$
    \State Get un-normalized weights $\{\tilde{w}_{m+1}^i\}_{i=1}^N$ via $\texttt{IS-PosteriorSampling}(t_0,\tilde{X}_0^{m+1})$
    \State Compute $\hat{p}_{t_0}(X_0^m) = \sum_{i=1}^N \tilde{w}_m^i$ and $\hat{p}_{t_0}(\tilde{X}_0^{m+1}) = \sum_{i=1}^N \tilde{w}_{m+1}^i$
    \State Define $\alpha = \min\left(1,\hat{p}_{t_0}(\tilde{X}_0^{m+1})/\hat{p}_{t_0}(X_0^m)\right)$
    \State Draw $U\sim\mathcal{U}(0,1)$
     \If{$U\leq \alpha$} $X_0^{m+1} = \tilde{X}_0^{m+1}$ \Else $\,\,X_0^{m+1}=X_0^{m}$
    \EndIf
\EndFor
\State \Return ${X}_0^M$
\end{algorithmic}
\end{algorithm}

\paragraph*{Estimation of the velocity field }

Consider $t\in\coint{0,1}$ and $x_t\in\msm$.
In generative modeling applications, the vector field $\bfu_t$ is learned by a neural network as the solution of a regression problem, by leveraging available samples from $\pi_1$, thus enabling a straightforward estimation of $\bfu_t(x_t)$. In contrast, here we do not suppose to have access to samples from $\pi_1$ but only to its un-normalized density $\npi$, see \eqref{eq:def_target}. Nonetheless, by remarking that $\bfu_t(x_t)$ is the integral of the tractable function $x_1\mapsto \Log_{x_t}x_1/(1-t)$ against the measure $\pi_{1|t}(\rmd x_1|x_t)$, whose density $p_{1|t}(\cdot|x_t)$ is known up to a normalization constant by \eqref{eq:bayes}, we can estimate it by using Monte-Carlo methods similarly to previous Euclidean diffusion-based sampling methods \cite{huang_reverse_2024,he_zeroth-order_2024, vacher_polynomial_2025, grenioux_stochastic_2024}. We denote by $\hat{\bfu}_t(x_t)$ the obtained estimation of $\bfu_t(x_t)$, of the form
\begin{equation}\label{eq:u_est}
    \hat{\bfu}_t(x_t) = \sum_{i=1}^N w^i\frac{\Log_{x_t}X_1^i}{1-t}
\end{equation}
where $\{X^i_1\}_{i=1}^N$ are $N\geq 1$ particles with corresponding normalized weights $\{w^i\}_{i=1}^N$ such that $\sum_{i=1}^N w^i \updelta_{X^i_1}\approx \pi_{1|t}(\cdot|x_t)$. We refer to \Cref{alg:mean} for the generic version of this algorithmic step, and to \Cref{sec:sampling} for its concrete implementation via several Monte Carlo methods.

\paragraph*{Why carefully choosing $t_0$ matters}
Overall, FRIPS methodology amounts to replace the original sampling problem targeting $\pi_1$ (potentially multi-modal) with a sequence of $K$ sampling sub-problems, associated to the distributions $\{\pi_{1|t_k}(\cdot|X_{k})\}_{k=0}^{K-1}$, where $X_k$ corresponds to the $k$-th iterate of the sequence defined in \eqref{eq:def_euler_ideal}, which are expected to be easier than the original one. Regarding their complexity, we can make the following key observations:
\begin{itemize}[wide]
    \item In the case where $t_0=0$ (as done in generative modeling), getting a sample $X_{0}$ from $\pi_{t_0}=\pi_0$ does not pose any difficulty. However, the very first step of the main loop in \Cref{alg:1} requires to be able to sample from the posterior $\pi_{1|0}(\cdot|X_{0})$, which is strictly equal to the target $\pi_1$ by choice of the independent coupling $\pi_{0,1}$. In particular, this first sub-problem is as difficult as the original one.
   
    \item On the other hand, we note that $\pi_{1|t}(\cdot|x_t)$ converges to the Dirac measure at $x_t$ as $t$ is closer to $1$. Thus, we expect  $\pi_{1|t}(\cdot|x_t)$ to resemble to a unimodal distribution centered around $x_t$ for $t$ close to $1$, a scenario for which existing Monte Carlo methods can be efficient. Hence, setting $t_0>0$ close enough to $1$ is a guarantee that the simplicity of the subsequent sampling sub-problems is thus appealing. Nevertheless, if $t_0$ is too close from 1, sampling accurately from $\pi_{t_0}$ to get the initial state $X_{0}$ is challenging, as $\pi_{t_0}$ resembles more and more to $\pi_1$. 
\end{itemize}
At the light of this discussion, the starting time $t_0$ in $[0,1)$ should be carefully chosen such that estimating the velocity field $\bfu$ and computing the initial state $X_0$ are both of relative difficulty. In practical applications, the choice of $t_0$ can significantly impact the overall performance of FRIPS. We discuss heuristic approaches to tune $t_0$ in \Cref{sec:t0} \emph{without access to target samples}, by leveraging supplementary information on the target distribution $\pi_1$. For now, we assume that $t_0$ is arbitrarily fixed and explain below how to sample the initial state. 

\paragraph*{Computation of the initial state} To sample the initial state $X_0$ distributed according to $\pi_{t_0}$, we may adopt the Langevin-based initialization strategy from \cite{grenioux_stochastic_2024, huang_reverse_2024}. In Euclidean spaces, this approach aims at running the \emph{Unadjusted Langevin Algorithm} (ULA) \cite{roberts_exponential_1996} on $\pi_{t_0}$ by leveraging the expression of its score $\grad \log p_{t_0}$ (as such, intractable) given by the Tweedie's formula
\begin{equation}\label{eq:tweedie}
    \grad \log p_{t_0} (x_{t_0}) = \int_{\msm}\grad \log p_{t_0|1}(x_{t_0}|x_1) \pi_{1|t_0}(\rmd x_1|x_{t_0})\eqsp,
\end{equation}
where the conditional score $\grad \log p_{t_0|1}(x_{t_0}|x_1)$ (with respect to the first variable $x_{t_0}$) is tractable, computed from \eqref{eq:p1t}. In our Riemannian setting, for any $x_{t_0}\in \msm$, we can estimate $\grad \log p_{t_0}(x_{t_0})$ in the exact same fashion as $\bfu_{t_0}(x_{t_0})$ (see \Cref{alg:mean}), using Monte Carlo approximation of the posterior distribution $\pi_{1|t_0}(\cdot|x_{t_0})$. The generic version of this estimation is summarized in \Cref{alg:score}, while the full Langevin-based initialization step, relying on the Riemannian version of ULA \cite{cheng_theory_2023,gatmiry_convergence_2022} denoted by RLA, is presented in \Cref{alg:ula}.

To be theoretically well grounded, this Langevin-based strategy requires sufficient regularity of the terms in \eqref{eq:tweedie} to justify exchanging the Riemannian gradient and the integral. However, this condition may fail to hold in some situations; see \Cref{sub:grass_exp} for an illustrative example. In such cases, we instead propose a zero-th order analog of this initialization strategy. Observing that $p_{t_0}(x_{t_0})$ can be interpreted as the normalizing constant of the posterior distribution $x_1 \mapsto q_1(x_1)p_{1|t_0}(x_1| x_{t_0})$ for a fixed state $x_{t_0}$, we resort to an instance of \emph{pseudo-marginal MCMC} \cite{andrieu2009pseudo} to sample from $\pi_{t_0}$. Specifically, we instantiate a Metropolis-Hastings (MH) scheme based on a prescribed proposal Markov kernel $\mathcal{Q}$, in which each evaluation of $p_{t_0}$ is replaced by an unbiased Importance Sampling estimator, using the technique described in paragraph (c) in \Cref{sec:sampling}. For compact manifolds, with $\pi_0$ the uniform distribution, a natural choice is $\mathcal{Q}=\pi_0$, which yields an Independent MH (IMH) algorithm summarized in \Cref{alg:pseudo_mcmc}.

\subsection{Posterior sampling with Monte Carlo methods}\label{sec:sampling}
A key aspect of FRIPS algorithm is the ability to efficiently and accurately sample from the posterior distributions $\{\pi_{1|t_k}(\cdot|X_{k})\}_{k=0}^{K-1}$ to perform velocity field estimation (\Cref{alg:mean}), and if needed score estimation (\Cref{alg:score}). By taking inspiration from previous diffusion-based sampling methods in Euclidean space \cite{grenioux_stochastic_2024, he_zeroth-order_2024, huang_reverse_2024, vacher_polynomial_2025}, we describe below three Riemannian-based Monte Carlo methods that solve this task: (a) Markov Chain simulation, (b) rejection sampling and (c) importance sampling - each one being a version of the generic step called \texttt{PosteriorSampling} used in \Cref{alg:mean,alg:score}. In the rest of this section, we consider an arbitrary posterior $\pi_{1|t}(\cdot|x_t)$, where $t\in [0,1)$ and $x_t\in \msm$, which we aim to sample from.  Based on the result of \Cref{prop:density}, we recall that this distribution admits a density $x_1\mapsto p_{1|t}(x_1|x_t)$ satisfying
\begin{equation}\label{eq:p1t}
    p_{1|t}(x_1|x_t) \propto p_{0|1}( \psi_t^{-1}(x_t;x_1)|x_1) \Jac_t(x_t;x_1)\1_{\mso_t(x_1)}(x_t) q_1(x_1)\eqsp.
\end{equation}

\paragraph*{(a) Markov Chain Monte Carlo (MCMC)}\label{sec:mcmc_method} As widely done in general sampling tasks, we first propose to get samples from $\pi_{1|t}(\cdot|x_t)$ by building a Markov chain whose invariant distribution is precisely $\pi_{1|t}(\cdot|x_t)$, as presented by \cite{grenioux_stochastic_2024}. Depending on the regularity of the target distribution $\pi_1$, one may either use (i) \emph{zeroth-order} schemes, that only involve the density $q_1$ (for instance geodesic slice sampling $\cite{durmus_geodesic_2024}$), or (ii) \emph{first-order} approaches, such as RLA, assuming the existence and access to the score $\grad \log q_1$ on $\msm$. In the experiments of \Cref{sec:expe}, we employ a Riemannian analogue of the Metropolis-Adjusted Langevin Algorithm (MALA) \cite{roberts_exponential_1996}, which can be viewed as an extension to general Riemannian manifolds of \cite{girolami_riemann_2011}. The RLA proposal is followed by a Metropolis-Hasting correction step, relying on the posterior density \eqref{eq:p1t} and the (typically approximate) proposal density $f$ given in \eqref{eq:trans_density}, see \Cref{app:details}.

While MCMC methods are proven to have a large mixing time when targeting highly multi-modal distributions (see \eg{} \cite{mangoubi_simple_2021}), we expect to use them in a more favorable setting, by targeting posterior distributions that are less complex than the target distribution $\pi_1$ and increasingly simpler (\ie, more unimodal) as $t$ increases. This can be ensured by carefully setting $t_0$ as discussed in the previous section. In particular, if $t_0$ is not well tuned, the intermediate MCMC runs could have a prohibitive mixing time, thus producing biased posterior samples under a limited computational budget. In this case, errors could accumulate at each global step $k$ of \Cref{alg:1}. 

One key advantage of using MCMC methods for posterior sampling is the possibility of wisely warm-starting Markov chains from one global step of \Cref{alg:1} to another. Indeed, for a given $k\in\{0, K-1\}$, the posterior distributions $\pi_{1|t_k}(\cdot|X_{k})$ and $\pi_{1|t_{k+1}}(\cdot|X_{{k+1}})$ are likely to be close to each another if $K$ is set large enough. Hence, using the last state of the Markov chains built at time $t_{k}$ as initialization of the chains built at time $t_{k+1}$ (granted this last state lies in the support of $\pi_{1|t_{k+1}}(\cdot|X_{{k+1}})$, which may depends on $(t_{k+1}, X_{k+1})$, see \eqref{eq:p1t}) could drastically reduce the number of MCMC iterations needed for convergence. Simulating a long MCMC chain could then only be necessary for the first sampling sub-problem (when the posterior distribution is more complex) at the initialization of FRIPS, see \Cref{alg:ula,alg:pseudo_mcmc}. We discuss the initialization of the corresponding Markov chains in \Cref{app:expe}.

\paragraph*{(b) Rejection Sampling (RS)}\label{sec:rejection} Interestingly, we remark that, in all Riemannian settings considered in this paper, we are always able to sample without any bias from the distribution with un-normalized density $x_1\mapsto p_{t|1}(x_t|x_1)$, denoted by $\nu_{1|t}(\cdot|x_t)$; see \Cref{sec:expe} for applications. For instance, in the particular Euclidean case where $\pi_0$ is set to be Gaussian (with $\pi_{0,1}$ still being the independent coupling), $\nu_{1|t}(\cdot|x_t)$ is a tractable Gaussian distribution; see \Cref{app:euclidean}. In the case of the $d$-sphere or Grassmann manifold with $\pi_0$ the uniform distribution, we verify in \Cref{app:sphere,app:grassmann} that $\nu_{1|t}(\cdot|x_t)$ is a simple pushforward of $\pi_0$, enabling unbiased sampling as well. Note however that, for general Riemannian manifolds or base distribution $\pi_0$, sampling exactly from $\nu_{1|t}(\cdot|x_t)$ is not necessarily straightforward. In this case, only approximate samples from $x_1\mapsto p_{t|1}(x_t|x_1)$ may be obtained, for instance by applying MCMC routines. 
Assume now that the target un-normalized density $q_1$ is upper bounded by some $q_1^\star>0$; in this case,  $x_1\mapsto p_{t|1}(x_t|x_1)q_1^\star$ is still a valid un-normalized density for $\nu_{1|t}(\cdot|x_t)$. Then, recalling that $p_{1|t}(x_1|x_t) \propto p_{t|1}(x_t|x_1)q(x_1)$ holds for any $x_1\in \msm$, we may perform rejection sampling to target $\pi_{1|t}(\cdot|x_1)$ by (i) sampling a proposal state $\tilde{x}_1 \sim \nu_{1|t}(\cdot|x_t)$ and (ii) applying the acceptance-rejection step defined by the acceptance rate 
\begin{equation}\label{eq:acceptance_rs}
\alpha^{\text{RS}}(\tilde{x}_1)=\frac{p_{t|1}(x_t|\tilde{x}_1)q_1(\tilde{x}_1)}{p_{t|1}(x_t|\tilde{x}_1)q_1^\star} = \frac{q_1(\tilde{x}_1)}{q_1^\star}\eqsp.
\end{equation}
We summarize this procedure in \Cref{alg:rejection}. Note that this method has previously been used in Euclidean state spaces by \cite{he_zeroth-order_2024}. In particular, it requires knowledge of an upper bound $q_1^\star$ on $q_1$. If not provided, the sequential nature of \Cref{alg:1} allows for iterative updates of the bound $q_1^\star$ used in \Cref{alg:rejection} every time $q_1$ is evaluated, as suggested in \cite{he_zeroth-order_2024}. 

Although using the RS method for posterior sampling is appealing, as we are guaranteed to obtain unbiased samples from $\pi_{1|t}(\cdot|x_t)$, granted unbiased samples from the proposal $\nu_{1|t}(\cdot|x_t)$, we expect the acceptance rate $\alpha^{\text{RS}}$ to be low in average due to the potential mismatch between $\pi_{1|t}$ and $\nu_{1|t}$, especially in high dimension. In the case where this rate is prohibitively low during the execution of FRIPS, RS may be replaced by one of the other posterior sampling methods.

\paragraph*{(c) Importance Sampling (IS)}\label{sec:is}Taking inspiration from \cite{huang_reverse_2024, vacher_polynomial_2025}, we propose a third posterior sampling approach based on IS \cite{agapiou_importance_2017} that relies on the same reasoning as the RS method. To be more precise, we propose to set the distribution $\nu_{1|t}(\cdot|x_t)$, defined in the previous paragraph, as the IS proposal to target $\pi_{1|t}(\cdot|x_t)$; in this case, the corresponding un-normalized IS weights are single evaluations of the target density $q_1$, similarly to \eqref{eq:acceptance_rs}. To obtain normalized weights associated to the proposal states, we adopt the self-normalization strategy as summarized in \Cref{alg:importance}.
Although the overall approach incurs bias due to the use of self-normalizing weights, it has the advantage of requiring a fixed number of target density evaluations given by the number of samples $N$, making it a favorable candidate with limited computational budget. In contrast, the RS method is unbiased, but needs more than $N$ target evaluations to obtain $N$ samples. Finally, we emphasize that IS suffers from the same fundamental drawback as RS related to the mismatch between the proposal $\nu_{1|t}$ and $\pi_{1|t}$; we refer to \cite{chatterjee_sample_2017} for a more in-depth discussion.

\begin{algorithm}
\caption{ $\texttt{MALA-PosteriorSampling}(t,x_t)$ : MALA method with target $\pi_{1|t}(\cdot|x_t)$}
\label{alg:sample_ula}
\begin{algorithmic}[1]
\Input Time $t\in\coint{0,1}$, point $x_t\in\msm$, step size $\delta>0$, number of MCMC steps $M\geq 1$, number of MCMC samples $N\leq M$
\Output approximate samples $\{X_1^i\}_{i=1}^N$ from $\pi_{1|t}(\cdot|x_t)$ with normalized weights $\{w^i\}_{i=1}^N$
\State Set an initial point $y_1^0\in\msm$ such that $x_t\in\mso_t(y_1^0)$
\For{$m = 0$ to $M-1$}
    \State Compute $v = \grad \log p_{1|t}(y_1^{m}|x_t)$ based on \eqref{eq:p1t}
    \State Compute $\bar{y}_1^{m+1}=\Exp_{y_1^{m}}(\delta v)$
    \State Draw $Z\in T_{\bar{y}_1^{m}}\msm$ from $\mathcal{N}(0,1)$ (identifying $T_{\bar{y_1}^{m}}\msm \approx \R^d$ by picking an orthonormal basis)
    \State Compute $\tilde{y}_1^{m+1}=\Exp_{\bar{y}_1^{m}}(\sqrt{2\delta}Z)$ (Riemannian Langevin proposal)
    \State Compute $\alpha = \min\left(1,\frac{p_{1|t}(\tilde{y}_1^{m+1}|x_t)f(y_1^m|\tilde{y}_1^{m+1})}{ p_{1|t}(y_1^{m+1}|x_t)f(\tilde{y}_1^{m+1}|y_1^m)}\right)$ based on \eqref{eq:p1t} and \eqref{eq:trans_density} (see \Cref{app:details}) 
    \State Draw $U\sim\mathcal{U}(0,1)$
    \If{$U\leqslant \alpha$}  ${y}_1^{m+1} = \tilde{y}_1^{m+1}$ \Else $\,\,{y}_1^{m+1}={y}_1^{m}$
    \EndIf
\EndFor
\State Set $X_1^i=y_1^{M-N+i}$ and $w^i=1/N$, for any $i\in\{1, \ldots, N\}$
\State \Return $\{w^i, X_1^i\}_{i=1}^N$
\end{algorithmic}
\end{algorithm}

\begin{algorithm}
\caption{\texttt{RS-PosteriorSampling}$(t,x_t)$ : RS method with target $\pi_{1|t}(\cdot|x_t)$}
\label{alg:rejection}
\begin{algorithmic}[1]
\Input Time $t\in\coint{0,1}$, point $x_t\in\msm$, upper bound $q^* > \sup_{\msm} q$, number of RS samples $N\geq 1$
\Output approximate samples $\{X_1^i\}_{i=1}^N$ from $\pi_{1|t}(\cdot|x_t)$ with normalized weights $\{w^i\}_{i=1}^N$
\For{$i= 1$ to $N$}
\Repeat
    \State Sample $\tilde{X}_1\sim \nu_{1|t}(\cdot|x_t)$ associated to the un-normalized density $x_1\mapsto p_{t|1}(x_t|x_1)$, see \eqref{eq:pt1}
    \State Draw $U\sim\mathcal{U}(0,1)$
\Until{$\alpha^{\text{RS}}(\tilde{X}_1)>U$}, see \eqref{eq:acceptance_rs}
\State Set $X_1^i=\tilde{X}_1$
\EndFor
\State Set $w^i=1/N$, for any $i\in\{1, \ldots, N\}$
\State \Return $\{w^i, X_1^i\}_{i=1}^N$
\end{algorithmic}
\end{algorithm}

\begin{algorithm}
\caption{\texttt{IS-PosteriorSampling}$(t,x_t)$ : IS method with target $\pi_{1|t}(\cdot|x_t)$}
\label{alg:importance}
\begin{algorithmic}[1]
\Input Time $t\in\coint{0,1}$, point $x_t\in\msm$, number of IS samples $N\geq 1$
\Output approximate samples $\{X_1^i\}_{i=1}^N$ from $\pi_{1|t}(\cdot|x_t)$ with normalized weights $\{w^i\}_{i=1}^N$
\For{$i= 1$ to $N$}
    \State Sample $X_1^i\sim \nu_{1|t}(\cdot|x_t)$ associated to the un-normalized density $x_1\mapsto p_{t|1}(x_t|x_1)$, see \eqref{eq:pt1}
    \State Set $\tilde{w}^i= q_1(X_1^i)$
\EndFor
\State Set $w^i=\tilde{w}^i/\sum_{j=1}^N \tilde{w}^j$, for any $i\in\{1, \ldots, N\}$
\State \Return $\{w^i, X_1^i\}_{i=1}^N$
\end{algorithmic}
\end{algorithm}

\newpage

\section{FRIPS in practice} After describing the general methodology of FRIPS in \Cref{sec:riem-flow-iter}, we provide here several heuristics \emph{leveraging partial knowledge about $\pi_1$} to correctly set FRIPS hyper-parameters in practical settings.

\subsection{Choice of $t_0$}\label{sec:t0} As explained in the previous section, the most critical hyper-parameter of FRIPS is without a doubt the starting time $t_0\in [0,1)$, especially in the case where $\pi_1$ is multi-modal. In this section, we provide a heuristic to tune it based on prior knowledge about $\pi_1$ within a realistic setting, \ie, without access to ground truth samples. This strategy relies on the intrinsic dynamics of the Markovian projection described hereafter.

\paragraph*{Markovian dynamics in presence of multi-modality} In \cite{biroli_dynamical_2024}, the authors study a Euclidean stochastic interpolant derived from a diffusion model (\ie, almost-independent coupling $\pi_{0,1}$ with $\pi_0$ being Gaussian) in the case where $\pi_1$ is multi-modal. In particular, they highlight three regimes of the associated Markovian projection directly related to the multi-modality of $\pi_1$, that are identified by \emph{transition times} $t_S$ (speciation time) and $t_C$ (collapse time) such that $0<t_S<t_C<1$. They show that, when $t>t_C$, the particle $X_t^M$ has already committed to one of the modes of $\pi_1$, which strongly determines its evolution up to the final state $X_1^M$, while it is mostly mode-blind when $t<t_S$. In particular, this can be analyzed equivalently through the lens of the so-called \emph{duality of log-concavity} \cite{grenioux_stochastic_2024} for the posterior distributions $\{\pi_{1|t}\}_{t\in [0,1]}$: when $t>t_C$, $\pi_{1|t}$ resembles a uni-modal distribution (thus corroborating the ``collapsing'' behavior of the dynamics), while it is multi-modal for $t<t_S$. Both works come to the conclusion that the key to obtain accurate sampling in a diffusion-based formulation in multi-modal scenarios is to carefully handle the in-between time window $[t_S, t_C]$. Such regimes are by definition inherent to FRIPS (independently of the Riemannian nature of the state space). Following the arguments developed in \Cref{subsec:generic_method}, an appropriate starting time $t_0$ should hence belong to the critical intermediary phase $(t_S,t_C)$. We emphasize that the initialization of $t_0$ is deeply discussed in \cite{grenioux_stochastic_2024}, albeit in an Euclidean setting. The authors demonstrate theoretically and empirically the existence of a sweet spot for $t_0$ (similar to the interval $[t_S,t_C]$), at which both the initialization targeting $\pi_{t_0}$ and the subsequent iterative sampling sub-problems are relatively simple. While this analysis requires additional assumptions on the target and relies on a different interpolation process, it suggests the existence of an optimal starting time for other stochastic interpolation methods such as FRIPS.  

\paragraph*{Heuristic estimation of $t_0$} Based on this characterization of the Markovian dynamics, we propose a heuristic to estimate the transition times $t_S$ and $t_C$ (and therefore, set $t_0\in (t_S, t_C)$), granted prior knowledge about the target distribution. In the rest of this paragraph, we consider a multi-modal distribution $\pi_1$ defined as a mixture of $J$ uni-modal distributions $\{\pi_{1,j}\}_{j=1}^J$, each with their normalized density $q_{1,j}$ supported on $\msm$, such that
\begin{equation}
    \pi_1(\rmd x)=\sum_{j=1}^J w_j \pi_{1,j}(\rmd x)\eqsp,
\end{equation}
where $\{w_j\}_{j=1}^J$ are \emph{unknown} positive weights with $\sum_{j=1}^J w_j = 1$. We assume that we have access to ``mode-blind'' samples from $\pi_1$, \ie, a collection of (possibly approximate) $N_j$ samples from each $\pi_{1,j}$, with knowledge of their mode label, such that the whole dataset is not representative of the ground truth relative weights $\{w_j\}_{j=1}^J$. In practical applications, ``mode-blind'' samples may be available if we have access to the location of the modes, \eg{} in the context of drug design \cite{jorgensen_many_2004}. Still, estimating the true weights remains a challenging task \cite{grenioux2025improving}.

To empirically observe the speciation and collapse phenomenons evoked above, we propose to leverage those samples as follows. Given $j\in \{1,\ldots,J\}$, a time $t\in [0,1)$ and samples $\{x_1^i\}_{i=1}^{N_j}$ drawn from $\pi_{1,j}$, we run FRIPS initialized at $t_0=t$ with $X_0\sim \pi_{t_0|1}(\cdot|x_1^i)$ for each $i\in\{1, \ldots, N_j\}$, and observe if the resulting sample $\hat{x}_1$ is likely to have been drawn from $\pi_{1,j}$; to do so, we compute the posterior estimator associated to the empirical distribution of the available samples, defined by
\begin{equation} \label{eq:mode_weight_posterior}
    \text{Mode}(\hat{x}_1) = \argmax_{i=1, \ldots, J}\frac{N_i q_{1,i}(\hat{x}_1)}{\sum_{j=1}^J N_j q_{1,j}(\hat{x}_1)} \eqsp ,
\end{equation}
and verify if it is equal to $j$. The resulting ``Return Accuracy'', computed by averaging the outputs of this procedure for all samples $\{x_1^i\}_{i=1}^{N_j}$, is denoted by $\tau_j(t)=\texttt{ReturnAccuracy}(t,j)$, see \Cref{alg:mode_accuracy}. Regarding the value of $\tau_j(t)$, we expect the following behavior. 
\begin{itemize}[wide]
    \item If $t<t_S$, then $X_0$ is mode-blind and the performance of Monte Carlo posterior sampling is limited by the multi-modality of $\pi_{1|t}$. For these two reasons, we expect to have $\tau_j(t)\approx 1/J$, which corresponds to the regime where simulated particles return to the modes of $\pi_1$ with a random assignment.
    \item If $t>t_C$, we expect to have $\tau_j(t)\approx 1$, as the initialization $X_0$ is well informed about the origin mode, and Monte Carlo posterior sampling methods are likely to perform well.
\end{itemize}
Hence, by varying the time input $t$, we expect to have an increasing profile for $t\mapsto \tau_j(t)$, from $1/J$ to 1. As such, we may identify the transition times $t_S$ and $t_C$ by averaging the obtained information for all distributions $\{\pi_j\}_{j=1}^J$, and therefore, accordingly set $t_0$, see \Cref{sub:sphere_exp} for an illustration.

\begin{algorithm}
\caption{$\texttt{ReturnAccuracy}(t,j)$ : return accuracy of $\pi_{1,j}$, by starting FRIPS at $t_0=t$}
\label{alg:mode_accuracy}
\begin{algorithmic}[1]
\Input Time $t\in\coint{0,1}$, samples $\{x_1^i\}_{i=1}^{N_j}$ drawn from $\pi_{1,j}$
\State Initialize $\hat{N}_j=0$
\For{$i=1$ to $N_j$}
    \State Sample $x_{t}^i \sim \pi_{t_0|1}(\cdot|x_1^i)$ 
    \State Run \Cref{alg:1} initialized at $t_0=t$ and $X_0=x_{t}^i$, yielding a sample $\hat{x}_1^i$
    \State Compute $\text{Mode}(\hat{x}_1^i)$, see \eqref{eq:mode_weight_posterior}
    \If{$\text{Mode}(\hat{x}_1^i)=j$} Update $\hat{N}_j \leftarrow \hat{N}_j +1$
    \EndIf
\EndFor
\State Compute the return accuracy $\tau_j=\hat{N}_j/N_j$
\State \Return $\tau_j$
\end{algorithmic}
\end{algorithm}

\subsection{Very first initialization of FRIPS} The initialization of FRIPS with RLA strategy, see \Cref{alg:ula}, or Pseudo-IMH procedure, see \Cref{alg:pseudo_mcmc}, requires an initial sample ${X}_0^0\in \msm$, before running the MCMC refinement steps. While setting ${X}_0^0\sim\pi_0$ appears to be a natural choice, one should keep in mind that the next step of the algorithm involves Monte-Carlo estimation of an integral against $\pi_{1|t_0}(\cdot|{X}_0^0)$, which may be ill-conditioned if $p_{t_0}({X}_0^0)$ is too close to $0$. Unfortunately, without any knowledge about the target distribution, choosing a candidate state ${X}_0^0$ is not straightforward. Assuming that we have access to ``mode-blind'' samples from $\pi_1$, see \Cref{sec:t0} for more details, we propose to set ${X}_0^0 = \psi_{t_0}(X_0,\hat{X}_1)$ where $X_0\sim \pi_0$ and $\hat{X}_1$ is precisely a ``mode-blind'' sample, hence providing a ''plausible'' sample from $\pi_{t_0}$ to start FRIPS. Furthermore, by generating several such plausible samples at time $t_0$, we may compute a rough estimate of the dispersion of $\pi_{t_0}$, and tune the hyperparameters of \Cref{alg:ula,alg:pseudo_mcmc} accordingly.

\subsection{General noise schedule and time discretization} In this work, we define the interpolation process $(X_t)_{t\in [0,1]}$ as a minimizing geodesic, that is, a minimizing interpolation path with constant speed. While this choice appears the most natural and simple, we may consider, for any $(x_0,x_1)\in\msd$, the interpolation function
\begin{equation}
    t\in\ccint{0,1}\mapsto\psi^\alpha_t(x_0;x_1) = (\upgamma_{x_0}^{x_1}\circ \alpha)(t)=\Exp_{x_1}((1-\alpha(t))\Log_{x_1}x_0)\eqsp,
\end{equation}
where $\alpha:\ccint{0,1}\rightarrow\ccint{0,1}$ is an arbitrary increasing, differentiable function such that $\alpha(0)=0$ and $\alpha(1)=1$. As a consequence, by setting the stochastic interpolant as $X^\alpha_t = \psi^\alpha_t(X_0;X_1)$ almost surely for any $t\in\ccint{0,1}$, we would obtain the time derivative $\dot{X}^\alpha_t = \dot{\alpha}(t)\dot{\upgamma}_{X_0}^{X_1}(\alpha(t))$. Hence, the velocity field associated to the Markovian projection of $(X^\alpha_t)_{t\in [0,1]}$ would be given by 
\begin{equation}
    \bfu^\alpha_t(x_t) = \mathbb{E}[\dot{X}^\alpha_t|X^\alpha_t=x_t]= \int_{\msm}\frac{\dot{\alpha}(t) \Log_{x_t}x_1}{1-\alpha(t)} \pi^\alpha_{1|t}(\rmd x_1|x_t) \eqsp, 
\end{equation}
where $\pi^\alpha_{1|t}(\cdot|x_t)$ denotes the posterior distribution of $X_1^\alpha$ given $X^\alpha_t=x_t$, thus generalizing \eqref{eq:u_def}. Adopting the nomenclature of diffusion models, we refer to $\alpha$ as the noise schedule. In the Euclidean setting, various noise schedules are for instance studied by \cite{grenioux_stochastic_2024}, and are proved to have an influence on the sampling performance. Specifically, the authors design noise schedules that match pre-defined profiles of the Signal-to-Noise Ratio (SNR), which informs on the increasing recovery of information about $\pi_1$ with respect to $\pi_0$ as $t$ increases. In the case of the Euclidean linear interpolant $X_t = (1-t) X_0 + t X_1$, the SNR is for instance defined by 
\begin{equation}\label{eq:snr}
    \operatorname{SNR}(t)= \frac{\PE\left[\Vert t X_1\Vert^2\right]}{\PE\left[\Vert(1-t)X_0\Vert^2\right]} \propto \frac{t^2}{(1-t)^2}
\end{equation}
Additionally, they propose to set the time discretization $\{t_k\}_{k=0}^{K}$ so that the log-SNR gap between each step is kept constant, instead of using the canonical uniform discretization $t_k = t_0+k(t_K-t_0)/K$. Empirically, this choice allows for a drastic reduction in the number of global steps $K$ while maintaining sampling quality. 

For simplicity, we adopt the standard noise schedule $t\mapsto \alpha(t)=t$ together with a uniform time discretization in our numerical experiments (\Cref{sec:expe}), as the choice of $K$ has only a limited impact on performance in the settings considered. We defer the practical implementation of \frips{} for general noise schedules and non-uniform time discretization schemes to future work.

\section{Numerical experiments}\label{sec:expe}

In this section, we quantitatively assess the performance of \frips{} in comparison with classical Monte Carlo samplers on a collection of challenging multi-modal target distributions defined on Riemannian manifolds. Following the recommendations of \cite{grenioux2025improving}, we focus on controlled \emph{bi-modal} targets for which the exact mode weights are known, and evaluate sampling performance through the ability of the generated samples to recover these proportions. 

Concretely, we mainly report the relative statistical error of the Monte Carlo estimate of the weight of the dominant mode, systematically denoted by $w \in (0,1)$ and associated with mode~$j=1$. Given $N$ generated samples $\{\hat{x}_1^i\}_{i=1}^N$, this estimator is defined as
\begin{align}
    \hat{w}=(1/N)\sum_{i=1}^N \mathbf{1}_{\text{Mode}(\hat{x}_1^i)=1}
\end{align}
where the posterior mode assignment $\text{Mode}(\cdot)$ is specified in \eqref{eq:mode_weight_posterior}. In all experiments, we set $N=4096$ and repeat each sampling procedure four times to obtain confidence intervals. 

Overall, this mode weight metric provides a clear and interpretable measure of multi-modal sampling performance. For completeness, we may additionally report complementary sample-based metrics. To ensure a fair comparison across methods, we systematically match the total number of evaluations of $q_1$, which constitutes the main computational bottleneck in practice, within each experiment. Further details on the experimental setup and the choice of hyperparameters are deferred to \Cref{app:expe}.

\subsection{Mixture of Gaussian distributions on the $d$-sphere}
\label{sub:sphere_exp}

Inspired by \cite{grenioux_stochastic_2024}, we first consider an \emph{unbalanced} mixture of two isotropic \emph{Riemannian Gaussian} distributions \cite{pennec_intrinsic_2006} (Riemannian MoG) on the sphere $\sphere^d$, with strongest mode weight $w=2/3$. The target density $p_1$ is defined as
\begin{equation}\label{eq:mog_target}
    x\in\sphere^d \mapsto \frac{2/3}{Z_1}\exp\left(\frac{-\dist(\mu_1,x)^2}{2\sigma_d^2}\right) + \frac{1/3}{Z_2}\exp\left(\frac{-\dist(\mu_2,x)^2}{2\sigma_d^2}\right)\eqsp,
\end{equation}
where means $\mu_1\in\sphere^d$ and $\mu_2\in\sphere^d$ are chosen as antipodal points, \eg, $\mu_1 = (1,0,..,0)$ and $\mu_2 = (-1,0,...,,0)$, scalar standard deviation $\sigma_d$ decreases with $d$\footnote{This choice compensates for the shrinking volume of the sphere as $d$ increases, which would otherwise reduce the separation between modes and artificially simplify the sampling task. In particular, we set $\sigma_4 = \pi/10$, $\sigma_{16}=\pi/12$ and $\sigma_{32}=\sigma_{64}=\sigma_{128}=\pi/18$.}, and $Z_1, Z_2$ are two normalizing constants. 
By symmetry of $\sphere^d$, we have $Z_1=Z_2$, and these constants are therefore omitted.

We consider increasingly challenging regimes by taking $d \in \{4,16,32,64,92\}$. To highlight the benefit of the proposed approach, we compare each baseline sampler described in \Cref{subsec:posterior} (MALA, RS, IS), applied directly to $\pi_1$, with the corresponding instances of \frips{} that use these methods as posterior sampling backbones.

We take $\pi_0$ to be the uniform distribution. This comparison is designed to illustrate the advantage of leveraging denoising non-equilibrium dynamics for complex multi-modal targets.

\begin{table*}
\centering
\caption{Average relative error in estimating the dominant mode weight for the bi-modal Gaussian mixture target \eqref{eq:mog_target}. For \frips{} variants, we report the best performance over a wide range of $t_0$ values, as shown in \Cref{fig:weight}. Bold font indicates the best result within each setting, while N/A denotes methods that encountered numerical issues.}
\label{tab:weight}
\resizebox{\textwidth}{!}{%

\begin{tabular}{@{}l c c c c c@{}}
\hline
\textsc{Methods} 
& $d=4$ 
& $d=16$ 
& $d=32$ 
& $d=64$ 
& $d=96$\\
\hline
\frips{}-\textsc{MALA} 
& $2.87\% \pm \text{\scriptsize 1.15\%}$ 
& $\mathbf{1.39\% \pm \text{\scriptsize 0.45\%}}$ 
& $\mathbf{1.27\% \pm \text{\scriptsize 0.65\%}}$ 
& $\mathbf{1.82\% \pm \text{\scriptsize 0.94\%}}$
& $\mathbf{2.29\% \pm \text{\scriptsize 0.94\%}}$\\

\frips{}-\textsc{IS} 
& $\mathbf{0.76\% \pm \text{\scriptsize 0.46\%}}$ 
& $2.89\% \pm \text{\scriptsize 0.64\%}$ 
& $10.53\% \pm \text{\scriptsize 2.42\%}$ 
& $20.80\% \pm \text{\scriptsize 1.58\%}$
& $22.77\% \pm \text{\scriptsize 1.32\%}$\\

\frips{}-\textsc{RS} 
& $2.09\% \pm \text{\scriptsize 0.45\%}$ 
& $22.55\% \pm \text{\scriptsize 1.05\%}$ 
& $23.57\% \pm \text{\scriptsize 1.22\%}$ 
& $23.39\% \pm \text{\scriptsize 0.93\%}$
& $22.25\% \pm \text{\scriptsize 1.48\%}$\\
\hline
\textsc{MALA} 
& $14.54\% \pm \text{\scriptsize 1.82\%}$ 
& $15.80\% \pm \text{\scriptsize 0.57\%}$ 
& $22.88\% \pm \text{\scriptsize 1.30\%}$ 
& $23.52\% \pm \text{\scriptsize 1.03\%}$
& $20.42\% \pm \text{\scriptsize 0.80\%}$\\
\textsc{IS} 
& $\mathbf{0.59\%} \pm \text{\scriptsize 0.60\%}$ 
& $\mathbf{1.33\% \pm \text{\scriptsize 0.86\%}}$ 
& $\mathbf{1.29\% \pm \text{\scriptsize 0.77\%}}$ 
& $13.24\% \pm \text{\scriptsize 0.37\%}$ 
& $24.84\% \pm \text{\scriptsize 0.60\%}$\\
\textsc{RS} 
& $\mathbf{0.76\% \pm \text{\scriptsize 0.48\%}}$ 
& $\mathbf{1.08\% \pm \text{\scriptsize 0.42\%}}$ 
& N/A 
& N/A 
& N/A\\
\hline
\end{tabular}
}
\end{table*}

\paragraph*{Main results}

The main results are summarized in \Cref{tab:weight}; additional results based on the Wasserstein distance are reported as well: see \Cref{app:sphere},  \Cref{tab:wass}. 

Among the baseline methods, MALA consistently fails across all dimensions due to severe mode trapping induced by the strong multi-modality of the target. RS and IS perform adequately in low-dimensional settings ($d \in \{4,16\}$), although RS already exhibits numerical instabilities at $d=32$, while IS remains effective up to this dimension. Beyond $d=32$, none of the classical samplers achieves satisfactory performance. 
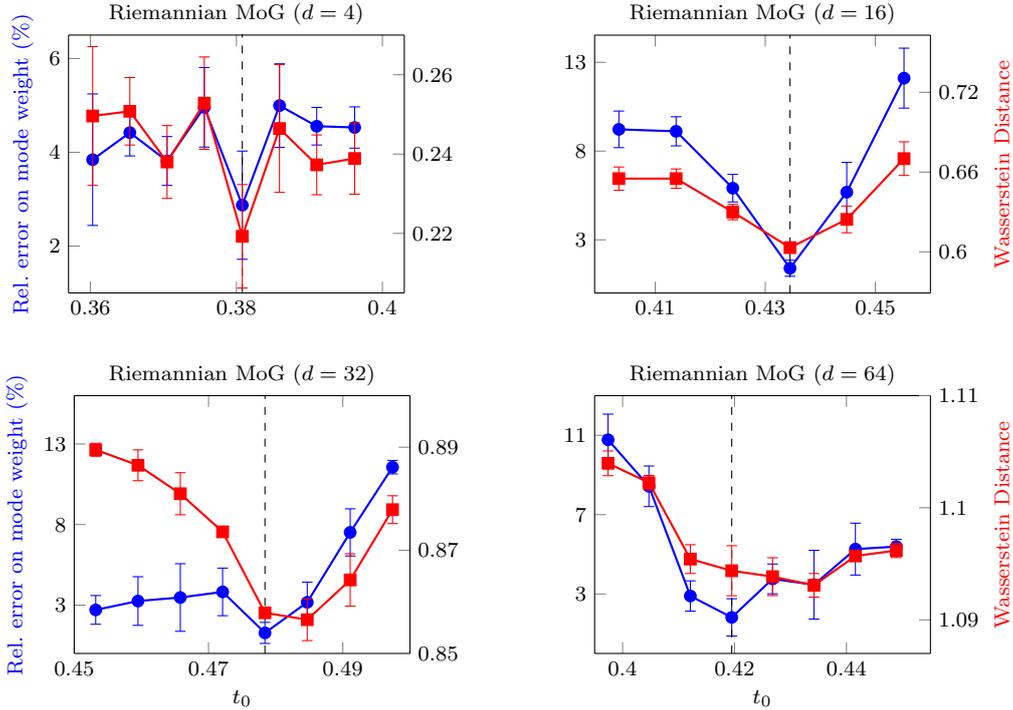
\begin{figure}[h!]
    \centering
    \begin{subfigure}[t]{0.45\textwidth}
        \centering
            \begin{tikzpicture}

\begin{axis}[
title style={yshift=-1.5ex},
    title = {Riemannian MoG ($d=4$)},
    name=leftaxis,
    width=6cm,
    height=5cm,
    ylabel={\textcolor{blue}{Rel. error on mode weight (\%)}},
    ymin=1., ymax=6.5,
    xmin=0.357, xmax=0.403,
    ytick={2,4,6},
    axis y line*=left,
    grid=none,
    legend style={at={(0.5,-0.2)}, anchor=north, legend columns=-1},
]

\addplot+[color=blue, thick, mark=*, every mark/.append style={color=blue}, error bars/.cd, y dir=both, y explicit] coordinates {
(0.3603, 3.842163) +- (0,1.402730)
(0.3654, 4.418945) +- (0,0.496753)
(0.3705, 3.814697) +- (0,0.520806)
(0.3756, 4.959106) +- (0,0.848778)
(0.3808, 2.871704) +- (0,1.154400)
(0.3859, 4.995728) +- (0,0.891923)
(0.3910, 4.556274) +- (0,0.401477)
(0.3962, 4.528809) +- (0,0.440215)
};
\addplot[dashed, black] coordinates {(0.3808, 0) (0.3808, 7)};
\end{axis}

\begin{axis}[
    name=rightaxis,
    width=6cm,
    height=5cm,
    at={(leftaxis.north west)},
    anchor=north west,
    ymin=0.205, ymax=0.27,
    xmin=0.357, xmax=0.403,
    ytick={0.22,0.24,0.26},
    axis y line*=right,
    axis x line=none,
    grid=none
]

\addplot+[color=red, thick, mark=square*, every mark/.append style={color=red}, error bars/.cd, y dir=both, y explicit] coordinates {
(0.3603, 0.249600) +- (0,0.017509)
(0.3654, 0.250789) +- (0,0.008504)
(0.3705, 0.238004) +- (0,0.009181)
(0.3756, 0.252845) +- (0,0.011642)
(0.3808, 0.219271) +- (0,0.013045)
(0.3859, 0.246444) +- (0,0.016096)
(0.3910, 0.237281) +- (0,0.007544)
(0.3962, 0.238892) +- (0,0.009003)
};

\end{axis}
\end{tikzpicture}
    \end{subfigure}
    \begin{subfigure}[t]{0.45\textwidth}
        \centering
            \begin{tikzpicture}

\begin{axis}[title style={yshift=-1.5ex},
title = {Riemannian MoG ($d=16$)},
    name=leftaxis,
    width=6cm,
    height=5cm,
    ymin=0., ymax=14.531,
    xmin=0.399, xmax=0.460,
    axis y line*=left,
    grid=none,
    xtick={0.41,0.43,0.45},
    ytick={3,8,13},
    legend style={at={(0.5,-0.2)}, anchor=north, legend columns=-1},
]

\addplot+[color=blue, thick, mark=*, every mark/.append style={color=blue}, error bars/.cd, y dir=both, y explicit] coordinates {
(0.4034, 9.222412) +- (0,1.030772)
(0.4138, 9.112549) +- (0,0.820713)
(0.4241, 5.908203) +- (0,0.777282)
(0.4345, 1.394653) +- (0,0.451774)
(0.4448, 5.682373) +- (0,1.679088)
(0.4552, 12.109375) +- (0,1.689936)
};
\addplot[dashed, black] coordinates {(0.4345, 0) (0.4345, 15)};
\end{axis}

\begin{axis}[
    name=rightaxis,
    width=6cm,
    height=5cm,
    at={(leftaxis.north west)},
    anchor=north west,
    ylabel={\textcolor{red}{Wasserstein Distance}},
    ymin=0.569, ymax=0.763,
    xmin=0.399, xmax=0.460,
    ytick={0.6,0.66,0.72},
    axis y line*=right,
    axis x line=none,
    grid=none
]

\addplot+[color=red, thick, mark=square*, every mark/.append style={color=red}, error bars/.cd, y dir=both, y explicit] coordinates {
(0.4034, 0.655064) +- (0,0.008750)
(0.4138, 0.655072) +- (0,0.007273)
(0.4241, 0.629917) +- (0,0.005804)
(0.4345, 0.603143) +- (0,0.001195)
(0.4448, 0.624392) +- (0,0.010071)
(0.4552, 0.670117) +- (0,0.012656)
};
\end{axis}
\end{tikzpicture}
    \end{subfigure}\par\medskip
    \begin{subfigure}[t]{0.45\textwidth}
        \centering
            \begin{tikzpicture}

\begin{axis}[title style={yshift=-1.5ex},
    title={Riemannian MoG ($d=32$)},
    name=leftaxis,
    width=6cm,
    height=5cm,
    xlabel={$t_0$},
    ylabel={\textcolor{blue}{Rel. error on mode weight (\%)}},
    ymin=0., ymax=16.,
    xmin=0.45, xmax=0.5,
    xtick={0.45,0.47,.49},
    ytick={3,8,13},
    axis y line*=left,
    grid=none,
    legend style={at={(0.5,-0.2)}, anchor=north, legend columns=-1},
]

\addplot+[color=blue, thick, mark=*, every mark/.append style={color=blue}, error bars/.cd, y dir=both, y explicit] coordinates {
(0.4532, 2.703857) +- (0,0.891594)
(0.4595, 3.253174) +- (0,1.509371)
(0.4658, 3.472900) +- (0,2.097577)
(0.4721, 3.820801) +- (0,1.476243)
(0.4784, 1.275635) +- (0,0.653817)
(0.4847, 3.173828) +- (0,1.254774)
(0.4911, 7.513428) +- (0,1.464271)
(0.4974, 11.560059) +- (0,0.419148)
};
\addplot[dashed, black] coordinates {(0.4784, 0) (0.4784, 17)};
\end{axis}

\begin{axis}[
    name=rightaxis,
    width=6cm,
    height=5cm,
    at={(leftaxis.north west)},
    anchor=north west,
    xlabel = $t_0$,
    ymin=0.85, ymax=0.9,
    xmin=0.45, xmax=0.5,
    ytick={0.85,0.87,0.89},
    axis y line*=right,
    axis x line=none,
    grid=none
]

\addplot+[color=red, thick, mark=square*, every mark/.append style={color=red}, error bars/.cd, y dir=both, y explicit] coordinates {
(0.4532, 0.889500) +- (0,0.001257)
(0.4595, 0.886496) +- (0,0.002990)
(0.4658, 0.880982) +- (0,0.004079)
(0.4721, 0.873585) +- (0,0.001051)
(0.4784, 0.857886) +- (0,0.000660)
(0.4847, 0.856552) +- (0,0.004060)
(0.4911, 0.864258) +- (0,0.005081)
(0.4974, 0.877893) +- (0,0.002695)
};
\end{axis}
\end{tikzpicture}
    \end{subfigure}
    \begin{subfigure}[t]{0.45\textwidth}
        \centering
            \begin{tikzpicture}

\begin{axis}[title style={yshift=-1.5ex},
title = {Riemannian MoG ($d=64$)},
    name=leftaxis,
    width=6cm,
    height=5cm,
    xlabel={$t_0$},
    ymin=0., ymax=13.,
    xmin=0.395, xmax=0.455,
    ytick={3,7,11},
    axis y line*=left,
    grid=none,
    legend style={at={(0.5,-0.2)}, anchor=north, legend columns=-1},
]

\addplot+[color=blue, thick, mark=*, every mark/.append style={color=blue}, error bars/.cd, y dir=both, y explicit] coordinates {
(0.3974, 10.772705) +- (0,1.295140)
(0.4047, 8.428955) +- (0,1.023263)
(0.4121, 2.899170) +- (0,0.763574)
(0.4195, 1.818848) +- (0,0.938672)
(0.4268, 3.759766) +- (0,0.750509)
(0.4342, 3.466797) +- (0,1.735546)
(0.4416, 5.261230) +- (0,1.307635)
(0.4489, 5.389404) +- (0,0.363915)
};
\addplot[dashed, black] coordinates {(0.4195, 0) (0.4195, 15)};
\end{axis}

\begin{axis}[
    name=rightaxis,
    width=6cm,
    height=5cm,
    at={(leftaxis.north west)},
    anchor=north west,
    ylabel={\textcolor{red}{Wasserstein Distance}},
    ymin=1.087, ymax=1.11,
    xmin=0.395, xmax=0.455,
    ytick={1.09,1.1,1.11},
    axis y line*=right,
    axis x line=none,
    grid=none
]

\addplot+[color=red, thick, mark=square*, every mark/.append style={color=red}, error bars/.cd, y dir=both, y explicit] coordinates {
(0.3974, 1.103963) +- (0,0.001104)
(0.4047, 1.102235) +- (0,0.000648)
(0.4121, 1.095419) +- (0,0.001285)
(0.4195, 1.094379) +- (0,0.002235)
(0.4268, 1.093852) +- (0,0.001690)
(0.4342, 1.093080) +- (0,0.001052)
(0.4416, 1.095696) +- (0,0.000465)
(0.4489, 1.096184) +- (0,0.000626)
};
\end{axis}
\end{tikzpicture}
    \end{subfigure}
    \caption{Performances of \frips{} for the bi-modal targets defined in \Cref{tab:weight}, as a function of $t_0$. \textcolor{blue}{Blue}: Average relative error on strongest mode weight estimation. \textcolor{red}{Red}: Wasserstein distance with respect to ground truth samples. For each setting, we observe a ''sweet spot'', \ie{} a range of $t_0$ where both metrics reach a minimum. The vertical dashed line highlights the value $t_0$ leading to the best mode weight estimation, corresponding to the values reported in \Cref{tab:weight}.}
    \label{fig:weight}
\end{figure}

In contrast, \frips{} delivers robust and accurate results across all dimensions, with an average relative error below $3\%$. More specifically, \frips{}-IS performs best in the lowest-dimensional setting ($d=4$), while \frips{}-MALA already significantly improves upon standard MALA, demonstrating the benefit of the proposed framework without additional algorithmic complexity. For $d \geq 16$, \frips{}-MALA consistently achieves the best performance, as IS and RS become ineffective posterior samplers, in line with expectations. Notably, these observations are consistent with the findings of \cite{grenioux_stochastic_2024} in the Euclidean setting.

\paragraph*{Study of $t_0$ hyperparameter} 

For each target configuration in \Cref{tab:weight}, we first conducted a grid search over $t_0$ to identify values yielding optimal performance at the end of the \frips{} procedure; see \Cref{fig:weight} for an illustration. To further elucidate the role of $t_0$, we apply \Cref{alg:mode_accuracy} in this setting by exploiting ``mode-blind'' samples, obtained by considering a balanced version of the target distribution. The resulting return accuracy curves, shown in \Cref{fig:return}, indicate a collapse time $t_C$ between $0.5$ and $0.6$, beyond which the accuracy rapidly approaches~$1$. According to the heuristic discussed in \Cref{sec:t0}, this suggests that the posterior distribution $\pi_{1|t_0}$ becomes effectively unimodal for $t_0 \geq t_C$. Selecting $t_0$ beyond this threshold is therefore unlikely to further improve posterior sampling quality, while potentially complicating the initialization step. Empirically, we observe that the optimal values of $t_0$ (see \Cref{fig:weight}) lie in the range $\tau_1 \in \ccint{0.9,0.99}$ for these targets. A theoretical analysis of this phenomenon is left for future work.

\input{simus/return_4}

\subsection{Mixture of Student's distributions via stereographic projection}
\label{sub:stereo}
To further assess the performance of \frips{} in challenging multi-modal sampling scenarios, we consider a bi-modal target distribution $\pi_1$ with heavy tails, defined on Euclidean space, \ie, $\R^d$ with $d \geqslant 1$. Such targets are notoriously difficult for standard Euclidean sampling methods, including diffusion-based approaches that rely on Gaussian noising schemes, which are poorly adapted to heavy-tailed distributions; see, for instance, \cite{shariatianheavy} in the context of generative modeling. In contrast with the previous experiment in \Cref{sub:sphere_exp}, $\pi_1$ is not originally defined on a Riemannian manifold, a choice made deliberately. To demonstrate the effectiveness of \frips{} in this setting, we adopt the strategy proposed in \cite{yang_stereographic_2024} and rather aim to sample from the push-forward of $\pi_1$ under the inverse stereographic projection (SP). This yields a bi-modal distribution, denoted $\pi_1^S$, supported on the $d$-sphere $\sphere^d$, where \frips{} can be naturally applied. The underlying rationale is that sampling from the resulting Riemannian distribution using geometry-aware tools is expected to be more easier in practice than sampling directly from its Euclidean counterpart. Samples generated on the sphere are subsequently mapped back to $\R^d$ to recover samples from $\pi_1$.

Concretely, we consider an \emph{unbalanced} mixture of two multivariate Student's $t$-distributions on $\R^d$, with dominant mode weight $w=2/3$. The un-normalized target density $q_1$ is given by
\begin{equation}\label{eq:target_student}
    x\in\R^d \mapsto (2/3)\left(1+\frac{1}{\nu}\left\Vert\frac{x-\mu_1}{\tau}\right\Vert^2\right)^{-\frac{\nu+d}{2}} +(1/3)\left(1+\frac{1}{\nu}\left\Vert\frac{x-\mu_2}{\tau}\right\Vert^2\right)^{-\frac{\nu+d}{2}}\eqsp,
\end{equation}
where $\mu_1,\mu_2 \in \R^d$ are the mode locations, $\tau=0.05$ is the scale parameter, and $\nu=1$ denotes the degree of freedom, taken identical for both modes for simplicity. In \Cref{app:stereo}, we derive the corresponding un-normalized density $q_1^S$ associated with $\pi_1^S$, along with details on the stereographic projection operator. Importantly, the projected distribution remains bi-modal with dominant mode weight unchanged at $w=2/3$.

To isolate the combined effect of heavy tails and multi-modality, we fix the dimension to $d=4$ and consider two regimes of increasing difficulty, taking inspiration from \cite{grenioux2025improving}: a \emph{low mode separation} regime with $\mu_1=(3/2,\ldots,3/2)$ and $\mu_2=(-3/4,\ldots,-3/4)$, and a \emph{high mode separation} regime with $\mu_1=(2,\ldots,2)$ and $\mu_2=(-1,\ldots,-1)$. To highlight the benefit of combining \frips{} with SP, we compare its performance to a naive Euclidean counterpart that directly targets $\pi_1$ using \frips{} in $\rset^d$ with a Gaussian initial distribution $\pi_0$, which corresponds to the deterministic variant of the algorithm proposed in \cite{grenioux_stochastic_2024}. Both approaches use MALA for posterior sampling (\Cref{alg:sample_ula}).
We also include standard Monte Carlo baselines, namely MALA on $\R^d$, as well as MALA combined with SP, which can be viewed as a first-order instance of the approach in \cite{yang_stereographic_2024}, who originally consider a Random-Walk Metropolis-Hastings scheme.

Beyond assessing multi-modality through mode weight estimation, we additionally evaluate the ability of the samplers to capture heavy tails using a metric specifically designed for this purpose: the mean squared logarithmic error (MSLE), computed with respect to ground-truth samples; see, \eg, \cite[Section~3.2]{allouche_ev-gan_2022}. This metric quantifies the accuracy of tail sampling through a hyperparameter $\xi \in (0,1)$, corresponding to the recovery of the $\xi$th percentile of the distribution. In practice, we set $\xi=0.99$. All metrics are systematically evaluated in $\R^d$.

\paragraph*{Main results} 

The results are reported in \Cref{fig:student}, which displays the performance of \frips{} with stereographic projection (orange) and its Euclidean counterpart (blue) as a function of $t_0 \in [0,0.9]$. The MALA baselines, shown as dashed lines, are independent of $t_0$ and use the same color coding. Two main observations emerge. 

First, within each setting—Euclidean or Riemannian—there exists a value of $t_0$ for which \frips{} outperforms MALA simultaneously in terms of multi-modality and heavy-tail recovery. These optimal values differ markedly across settings: they correspond to relatively small $t_0$ values in the Euclidean case (around $t_0 \approx 0.2$) and much larger values in the Riemannian case (around $t_0 \approx 0.8$), highlighted by vertical dashed lines. 

This discrepancy reflects fundamentally different behaviors of the two algorithms. As expected, increasing the distance between modes exacerbates the difficulty of multi-modal sampling, as evidenced by the degradation of MALA’s mode weight estimates. More importantly, while the Euclidean and Riemannian versions of \frips{} exhibit comparable performance in recovering mode proportions, the Euclidean variant performs significantly worse in capturing heavy tails, as measured by MSLE. In contrast, \frips{} combined with SP achieves consistently strong tail performance, largely invariant to the degree of mode separation. These results underscore the advantage of the Riemannian formulation of \frips{} for sampling from distributions that are both multi-modal and heavy-tailed.

\begin{figure}
    \centering
\input{simus/stereo2}
\label{fig:student}
\end{figure}

\subsection{Mixture of Gaussian distributions on the Grassmann manifold}
\label{sub:grass_exp}
To further illustrate the generality of our approach, we apply \frips{} to multi-modal distributions defined on the Grassmann manifold. These matrix manifolds arise in a wide range of applications, including dimensionality reduction \cite{holbrook_bayesian_2016} and shape analysis \cite{hong_regression_2017,tong_diverse_2026}. Given $n>p\geqslant1$, we denote by $\grass$ the Grassmann compact manifold with dimension $d=p(n-p)$, whose elements represent $p$-dimensional linear subspaces of $\R^n$.
We refer the reader to \cite{bendokat_grassmann_2020} for an overview of the geometry of $\grass$, and summarize the relevant properties and formulas in \Cref{app:grassmann}. 

As a representative target example, we choose the same family of \emph{unbalanced} Riemannian MoG as in \eqref{eq:mog_target}, with $\sigma = 0.25$, $\mu_1$ and $\mu_2$ two points on $\grass$ separated by a prescribed distance $\dist(\mu_1,\mu_2)\in\{2.18,2.45\}$ respectively corresponding to a \emph{low mode separation} and \emph{high mode separation} regime. Experiments are carried out on the $15$-dimensional manifold $\text{Gr}(8,3)$, still assessing the performance of samplers via the statistical estimation of the dominant mode weight $w=2/3$ in both regimes. Similar to the sphere case, $\pi_0$ is the uniform distribution. 

\paragraph*{Computation of the initial state}
In contrast with the sphere experiments of \Cref{sub:sphere_exp}, we rely on the pseudo-marginal MCMC scheme described in \Cref{alg:pseudo_mcmc} to compute the initial state at time $t_0$. This is necessitated by the lack of regularity of the conditional density $p_{t_0|1}$ in \eqref{eq:pt1} near the boundary of its support, a phenomenon induced by the geometry of $\grass$; see \Cref{app:grassmann} for further discussion. Empirically, this initialization procedure performs well in considered moderate dimension $d=15$. However, we expect its scalability to be deteriorated in higher dimensions compared to \Cref{alg:ula}. Extending \Cref{alg:ula} to Grassmann manifolds could allow larger values of $d$ to be addressed while maintaining accurate sampling, similar to the sphere case of \Cref{sub:sphere_exp}. We leave this extension for future work.

\begin{figure}
\centering
\begin{tikzpicture}
\begin{groupplot}[group style={group size=2 by 1},height=5cm,width=7cm]

\nextgroupplot[title style={yshift=-1.5ex},
title = {Riemannian MoG (low mode separation)},
    ylabel={Rel. error on mode weight (\%)},
    xlabel ={$t_0$},
    legend style={font=\small, at={(0.5,1.05)}, anchor=south, legend columns=2},
    ymin=0,
    ymax = 20,
    grid style=dashed,
]
\addplot+[color=blue, thick, mark=*, every mark/.append style={color=blue}, error bars/.cd, y dir=both, y explicit] coordinates {
(0.100000, 16.200380) +- (0,0.823373)
(0.150000, 15.118976) +- (0,0.751931)
(0.200000, 14.514123) +- (0,1.305603)
(0.250000, 12.278001) +- (0,1.103697)
(0.300000, 10.243495) +- (0,1.101411)
(0.350000, 8.538910) +- (0,0.743619)
(0.400000, 5.991196) +- (0,0.954686)
(0.450000, 3.635936) +- (0,1.020961)
(0.500000, 2.242941) +- (0,0.572246)
(0.550000, 3.837553) +- (0,0.893568)
(0.600000, 4.359926) +- (0,0.245908)
(0.650000, 4.387420) +- (0,0.570482)
(0.700000, 5.752921) +- (0,0.813934)
(0.750000, 6.990120) +- (0,0.637242)
(0.800000, 7.485000) +- (0,1.323080)
(0.850000, 8.548074) +- (0,0.771284)
(0.900000, 8.474759) +- (0,0.545188)
};
\addplot[dashed, black] coordinates {(0.5, 0) (0.5, 20)};
\addplot[thick,blue, dashed] coordinates {(0.1,11.29) (0.9,11.29)};

\nextgroupplot[title style={yshift=-1.5ex},
title = {Riemannian MoG (high mode separation)},
    xlabel ={$t_0$},
    legend style={font=\small, at={(0.5,1.05)}, anchor=south, legend columns=2},
    ymin=0,
    ymax = 20,
    grid style=dashed,
]
\addplot+[color=blue, thick, mark=*, every mark/.append style={color=blue}, error bars/.cd, y dir=both, y explicit]
coordinates {
(0.100000, 18.573969) +- (0,0.800355)
(0.150000, 15.247279) +- (0,0.738178)
(0.200000, 14.477466) +- (0,0.729019)
(0.250000, 12.882854) +- (0,0.279778)
(0.300000, 10.032713) +- (0,1.296404)
(0.350000, 7.640795) +- (0,0.676184)
(0.400000, 5.972867) +- (0,0.894554)
(0.450000, 4.589037) +- (0,1.079810)
(0.500000, 4.231624) +- (0,0.526457)
(0.550000, 3.177714) +- (0,1.013033)
(0.600000, 3.333509) +- (0,1.364240)
(0.650000, 4.900628) +- (0,0.656458)
(0.700000, 6.192814) +- (0,0.982606)
(0.750000, 5.853730) +- (0,1.414385)
(0.800000, 7.448342) +- (0,0.599203)
(0.850000, 8.465595) +- (0,0.347766)
(0.900000, 9.711958) +- (0,1.785914)
};

\addplot[dashed, black] coordinates {(0.55, 0) (0.55, 20)};
\addplot[thick,blue, dashed] coordinates {(0.1,17.79)(0.9,17.79)};

\end{groupplot}
\end{tikzpicture}
    \caption{Performances of \frips{} (solide \textcolor{blue}{blue} line) and MALA (dashed \textcolor{blue}{blue} line) instances for the bi-modal Gaussian targets \eqref{eq:mog_target} on the Grassmann manifold $\text{Gr}(8,3)$, as a function of $t_0$, in \emph{low mode separation} regime (left) and \emph{high mode separation} regime (right). The vertical dashed line indicates the value of $t_0$ that yields the lowest relative error on the dominant mode weight estimation.}
\label{fig:grassmann}
\end{figure}
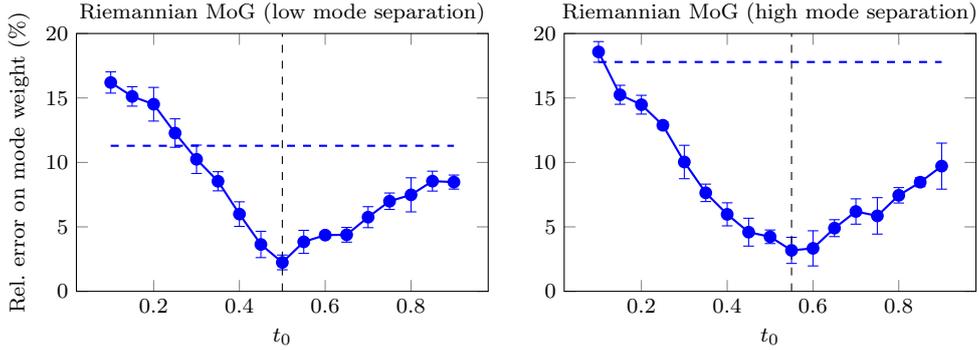

\paragraph*{Main results}
Following the same protocol as in \Cref{sub:sphere_exp}, we run \frips{} for a large range of $t_0$ values, relying on MALA as posterior sampling backbone (\Cref{alg:sample_ula}), and compare to the naive MALA baseline. As shown in \Cref{fig:grassmann}, FRIPS recovers well the true mixture weight $w$, achieving an average relative error below $3\%$ when $t_0$ is within a suitable range, independently of the multi-modality regime. In line with the results obtained on $\sphere^d$, \frips{} consistently outperforms naive MALA sampling for comparable computational budgets, highlighting its effectiveness when targeting multi-modal distributions across different manifolds.

\section{Conclusion}

In this paper, we studied the problem of sampling from multi-modal distributions defined on Riemannian manifolds when only an un-normalized density is available. Our approach draws inspiration from generative modeling, where stochastic interpolants \cite{albergo_stochastic_2023} have emerged as an effective framework for handling multi-modality. Leveraging recent extensions of these ideas to Riemannian manifolds \cite{chen_flow_2024}, we introduced a novel sampling methodology, termed Flow-based Riemannian Iterative Posterior Sampling (FRIPS). The proposed method is entirely training-free, relying solely on Monte Carlo techniques, and can be viewed as a Riemannian analogue of the Euclidean approach developed in \cite{grenioux_stochastic_2024}. To support our methodology, we developed a rigorous theoretical framework for Riemannian stochastic interpolants tailored to the sampling setting and provided practical guidelines for instantiating FRIPS. We demonstrated the effectiveness of the approach on a range of challenging multi-modal targets defined on the $d$-sphere and Grassmann manifolds, where traditional Monte Carlo samplers fail, highlighting FRIPS robustness to pronounced multi-modality and increasing dimensionality. One remaining limitation concerns the choice of the initial time of the denoising dynamics that underlies FRIPS, which currently requires a grid search to achieve optimal performance. Nevertheless, our empirical findings suggest a structured dependence of this hyperparameter with respect to the target distribution, opening the door to principled tuning strategies. Developing such automated and practical selection procedures based on theoretical investigation is an important direction for future work.

\begin{acks}[Acknowledgments]
The authors would like to thank Louis Grenioux for insightful discussion.
Part of this work has been supported by the PRIMaL project (Carnot Institute).
\end{acks}

\newpage
\bibliographystyle{imsart-number.bst}
\bibliography{biblio_clean.bib}

\newpage
\begin{appendix}

\section{Background on Riemannian geometry}
\label{app:geometry}

To complete the brief introduction to Riemannian geometry of \Cref{sub:geometry},
 we properly introduce some notions used in the paper. We refer to \cite{lee_introduction_2012, sakai_riemannian_1996} for basics on this topic.

\subsection{Riemannian structure and geodesics}
Assume $\msm$ is a $d$-dimensional smooth, connected manifold. 
\paragraph*{Smooth manifold structure} By definition, there exist a (maximal) \emph{atlas} on $\msm$, \ie{} a collection of compatible \emph{coordinate charts} whose domains cover the whole manifold. A coordinate chart at $x\in\msm$ is a homeomorphism $\varphi : \msu \rightarrow \msu'$, where $\msu$ is an open subset of $\msm$ that contains $x$ and $\msu'$ is an open subset of $\R^d$. By compatible, we mean that the compositions $\varphi\circ\psi^{-1}, \psi\circ\varphi^{-1}$ are diffeomorphisms (in the usual $\R^d$ sense), where $\varphi,\psi$ are two charts with overlapping domains. A function $f$ between two smooth manifolds is called smooth if its composition with the appropriate coordinate charts is, in the usual $\R^d$ sense.

We call \emph{local coordinates} the components of $\varphi = (x^1,...,x^d)$. Given a point $z\in\msm$, the \emph{tangent space at $z$} is a $d$-dimensional vector space spanned by the \emph{coordinate vectors} in any local coordinate system. For every $1\leqslant i\leqslant d$ and every smooth function $f : \msm \rightarrow \R$, coordinate vectors are defined by
\begin{equation}
    \left.\frac{\partial}{\partial x^i}\right\vert_z f = \partial_i\left[f\circ\varphi^{-1}\right]_{|\varphi(z)}\eqsp,
\end{equation}
where $\partial_i$ denotes the usual partial derivative in $\R^d$. The \emph{differential of $f$ at  $z$}, denoted $\rmd f_z: T_z\msm \rightarrow \R$, is the linear form defined by $\rmd f_z\left(\frac{\partial}{\partial x^i}\right) = \left.\frac{\partial}{\partial x^i}\right\vert_{z}f, 1\leq i \leq d$. The \emph{Riemannian gradient} $\grad f(z) \in T_z\msm$ is the only tangent vector at $z$ satisfying $\rmd f_z = \langle \grad f(z),\,\cdot\,\rangle_z $.

\paragraph*{Metric and vector fields} A \emph{Riemannian metric} $z\in\msm\mapsto\metric_z$ defines an inner product in $T_z\msm$ for every $z\in\msm$, varying smoothly in $z$, and denoted $\ps{v}{w}_z$ for every $v,w\in T_z\msm$.

A (smooth) \emph{vector field} $X$ is a (smooth) map from $\msm$ to the tangent bundle $T\msm$, \ie{} the disjoint union of all tangent spaces. In local coordinates, a vector field can be written as $X_z = X^i_z \left.\frac{\partial}{\partial x^i}\right\vert_{z}$ for every $z\in\msm$. The \emph{Riemannian divergence} $\text{div}X$ is defined in local coordinates by
\begin{equation}
\text{div}\left(X^i\frac{\partial}{\partial x^i}\right) = \frac{1}{\sqrt{\det g}}\frac{\partial}{\partial x^i}\left(X^i\sqrt{\det g}\right)\eqsp,
\end{equation}
where $\det g$ denotes the determinant of the Gram matrix $g$ of $\metric$ for these coordinate vectors, whose elements are denoted $g_{i,j} = \ps{\frac{\partial}{\partial x^i}}{\frac{\partial}{\partial x^j}}$.

The \emph{Levi-Civita connection} $\nabla$ \cite[Chapter 5]{lee_introduction_2012} associated to $\metric$ defines the differentiation of a vector field with respect to another: given two vector fields $X,Y$, $\nabla_X Y$ outputs a vector field. Given a smooth curve $t\in I\subset \R \mapsto \upgamma(t)\in \msm$, a vector field $X$ along $\upgamma$ is \emph{parallel} if $\nabla_{\dot{\upgamma}}X=0$, where for every $t\in I, \dot{\upgamma}(t)\in T_{\upgamma(t)}\msm$ is the velocity field of $\upgamma$. Given $a\in I$ and a tangent vector $v\in T_{\upgamma(a)}$, a vector field $X$ along $\upgamma$ satisfying $X_{\upgamma(a)} = v$ and $\nabla_{\dot{\upgamma}}X = 0$ is called the \emph{parallel transport of $v$ along $\upgamma$}.

\paragraph*{Geodesics and completeness}
\emph{Geodesics} are smooth functions $t\in I\subset \R \mapsto \upgamma(t)\in \msm$ satisfying $\nabla_{\dot{\upgamma}}\dot{\upgamma} = 0$, which results in the geodesic equation in local coordinates \cite[Chapter 6]{lee_introduction_2012}. Given $a\in I, v\in T_{\upgamma(a)}\msm$, $t\in I\mapsto \tau_{\upgamma(a),v}(t)\in T_{\upgamma(t)}\msm$ denotes the parallel transport of $v$ along $\upgamma$ at $\upgamma(t)$. 

In \Cref{ass:1}, we suppose that $\msm$ is complete, that is, metric complete for the Riemannian distance associated with $\metric$. According to the Hopf-Rinow theorem \cite[Theorem 6.19]{lee_introduction_2012}, metric completeness is actually equivalent to geodesic completeness for a connected manifold: every geodesic can be extended to $\R$, and every pair of points in $\msm$ can be connected by at least one minimizing geodesic.

\subsection{Jacobi fields and spherical coordinates}
\paragraph*{Jacobi fields}
Fix $z\in\msm$ and $v\in T_z\msm$, and let $\upgamma_{z,v}$ denote the geodesic starting at $z$ with initial velocity $v$, \ie{}, $\upgamma_{z,v}(0)=z$ and $\dot{\upgamma}_{z,v}(0)=v$, so that $\upgamma_{z,v}(t) = \Exp_z(tv)$ (the exponential map has been defined in \Cref{sub:geometry}). A \emph{Jacobi field} along $\upgamma_{z,v}$ (see \cite[Section II.5]{chavel_riemannian_2006}, \cite[Chapter 10]{lee_algorithmic_2018}) is a smooth vector field $t\in\coint{0,+\infty} \mapsto J(t)\in T_{\upgamma_{z,v}(t)}\msm$ satisfying the Jacobi equation
\begin{equation}
    \nabla_t^2 J + \curv(\dot{\upgamma}_{z,v},J)\dot{\upgamma}_{z,v}=0\eqsp,
\end{equation}
with $\curv$ the curvature tensor \cite[Section I.4]{chavel_riemannian_2006} and $\nabla_t$ the covariant derivative. Since $\nabla_t\dot{\upgamma}_{z,v} = 0 = \curv(\dot{\upgamma}_{z,v}, \dot{\upgamma}_{z,v})$, $\dot{\upgamma}_{z,v}$ is a Jacobi field, as well as $t\mapsto t\dot{\upgamma}_{z,v}(t)$. The corresponding trivial solutions, which lie in the subspace $\R\dot{\upgamma}_{z,v}(t)\subset T_{\upgamma_{z,v}(t)}\msm$ at each $t$, are called \emph{tangential} Jacobi fields. The other solutions, orthogonal to $\dot{\upgamma}_{z,v}(t)$ for any $t\geq 0$, are called \emph{normal} Jacobi fields.

Jacobi fields relate to the differential of the exponential map: from \cite[Proposition 10.10]{lee_introduction_2018}, the Jacobi field along $\upgamma$ with initial condition $J(0) = 0, \nabla_t J(0) = w$ reads
\begin{equation}
\label{eq:jacobi_exp}
    J(t) = \rmd {\Exp_z}_{\vert t v} t w\eqsp.
\end{equation}

If $\msm$ has constant sectional curvature $\kappa$  \cite[Proposition 8.29]{lee_introduction_2018}, Jacobi fields admit a closed-form expression \cite[Proposition 10.12]{lee_introduction_2018}. This condition holds \eg{} for the Euclidean space ($\kappa=0$), and the unit sphere $\sphere^d$ ($\kappa=1$). 
Normal Jacobi fields are of the form
\begin{equation}\label{eq:jacobi_cst}
    J(t) \propto s_{\kappa}(t)\bfe(t)\eqsp,
\end{equation}
with $\bfe$ a parallel unit normal vector field along $\upgamma_{z,v}$, \ie{} obtained as the parallel transport along $\upgamma$ of a unit tangent vector of $T_{z}\msm$, orthogonal to the initial velocity vector $\dot{\upgamma}_{z,v}(0)=v$. The function $s_{\kappa}$ is defined as $t\mapsto \sin(\sqrt{\kappa}t)/\sqrt{\kappa}$ if $\kappa>0$, $t\mapsto \sinh(\sqrt{-\kappa}t)/\sqrt{-\kappa}$ if $\kappa<0$, and $t\mapsto t$ if $\kappa=0$. 

\paragraph*{Spherical coordinates and Riemannian  measure}
Following \cite[Chapter 3]{chavel_riemannian_2006}, we provide a complete construction of the spherical coordinates. These coordinates allow us to write a practical expression of the \emph{Riemannian measure} $\Rvol$, proving \eqref{eq:change-variable}.
We repeat the corresponding paragraph of \Cref{subsec:posterior}, with additional details.

By definition, the Riemannian measure $\Rvol$ is constructed as follows; see also \cite[Section III.3]{chavel_riemannian_2006}. Given a chart $\varphi = (x^1,...,x^d)\,:\,\msu \rightarrow \R^d$, where $\msu$ is an open subset of $\msm$, and a function $f\,:\, \msm \rightarrow \R$ such that $f\circ x^{-1}$ is measurable, the expression

\begin{equation}
    \int_{\varphi(\msu)} \left(f\sqrt{\det g}\right)\circ \varphi^{-1}\rmd \lambda\eqsp,
\end{equation}
with $\lambda$ the Lebesgue measure on $\R^d$, does not depend on the specific chart but only on $\msu$ and $f$. Consequently, it defines $\int_{\msu}f\rmd \Rvol$.
Then, by picking any atlas on $\msm$, along with a partition of unity, the measure $\Rvol$ is extended to the whole manifold.

Although this construction of $\Rvol$ is valid for any chart, spherical coordinates yield a particularly convenient form for functions expressed along geodesics.
Fix $z\in\msm$. Recall that, under \Cref{ass:1}, the map $\Exp_z$ is a diffeomorphism from $\ID(z)\subset T_z\msm$, defined in \eqref{eq:def_id_x} onto $\msm\setminus\cut(z)$. By restriction to $\ID^*(z) = \ID(z)\setminus\{0\}$, we obtain a diffeomorphism onto $\msm^z = \msm\setminus\{\cut(z) \cup \{z\}\}$. Denote $\Smsm{z}\subset T_z\msm$ the unit sphere of $T_z\msm$.
Then, it holds that $  \ID^*(z) = \{  t\xi\,:\,\xi\in \Smsm{z}, 0< t<\cuttime_z(\xi)\}$.

Let $\Dmsm{z}\subset \Smsm{z}  \times \ooint{0,+\infty} $ denotes the set of pairs $(\xi,t)$ such that $t\xi\in\ID^*(z)$.
The map $(\xi,t) \mapsto \varphi_z(\xi,t) = \Exp_z t\xi$ is a diffeomorphism from $\Dmsm{z}$ onto $\msm^z$,   
whose inverse writes
\begin{equation}
    x\in\msm^z\mapsto \varphi^{-1}_z(x) = \left(\frac{\Log_z x}{\dist(z,x)},\dist(z,x)\right)\eqsp.
\end{equation}
Since $\Rvol(\cut(z)\cup\{z\})=0$, almost every point in $\msm$ can be parameterized in this way.

To construct an atlas on 
$\msm^z$, we pick a chart $u=(u^1,...,u^{d-1})$ mapping an open subset $\mso\subset \Smsm{z}$ to $\msu \subset \R^{d-1}$, and denote $\xi=u^{-1} : \msu \rightarrow \mso$.
Then, $(v,t) \mapsto \Exp_z t\xi(v)$ is a diffeomorphism from $\msu' \coloneq(u\times\Id)(\Dmsm{z}\cap \mso \times \ooint{0,+\infty})$ onto $\mso' \coloneq \varphi_z(\Dmsm{z}\cap \mso \times \ooint{0,+\infty})$.
The inverse map provides a chart $x = (x^1,...,x^d) : \mso' \rightarrow \msu'$ on $\msm^z$, namely
\begin{equation}
    x^{i} = u^{i}\circ\left(\frac{\Log_z}{\dist(z,\cdot)}\right), 1\leq i\leq d-1; \qquad x^d = \dist(z,\cdot)\eqsp.
\end{equation}
Then, by definition of $\Rvol$, an integral on $\mso'\subset\msm^z$ against $\Rvol$ can be computed as an integral on $\msu'\subset\R^d$, against the Lebesgue measure.
For every $\Rvol$-integrable function $f$,
\begin{align}
    \int_{\mso'} f \rmd \Rvol &= \int_{\msu'} \left(f\sqrt{\det g}\right)\circ x^{-1}(v,t)\rmd t \rmd v \\
\label{eq:change_of_measure}&= \int_{\msu}\int_0^{\cuttime_z(\xi(v))}f(\Exp_z t\xi(v))\sqrt{\det\left\langle \frac{\partial}{\partial x^i},\frac{\partial}{\partial x^j} \right\rangle_{\Exp_z t\xi(v)}}\rmd t \rmd v \eqsp.
\end{align}
To integrate over the entire manifold $\msm$ (equivalently over $\msm^z$), one selects an atlas on $\Smsm{z}$ and constructs a corresponding atlas on $\msm^z$ using the procedure described for a single chart $u$. Using a partition of unity, we can write the change of measure formula on $\msm^z$ using the expression \eqref{eq:change_of_measure} for a specific chart: see \cite[Chapter 16]{lee_introduction_2012} for details regarding integration on manifolds.

Now, still following \cite[Chapter III]{chavel_riemannian_2006}, we make explicit the determinant factor of \eqref{eq:change_of_measure} using Jacobi fields. 
Consider the coordinate vector fields $\frac{\partial}{\partial u^i}, \frac{\partial}{\partial t}$ on $\Dmsm{z}$, and write $\partial_i \xi \coloneqq \frac{\partial}{\partial u^i}$ for every $1\leqslant i\leqslant d-1$.

It follows from \eqref{eq:jac_exp} that for every $1\leq i\leq d-1$,
\begin{equation}\left.\frac{\partial}{\partial x^i}\right\vert_{\Exp_z t\xi} = J_{i}(t,\xi)\eqsp,
\end{equation}
with $J_{i}$ the Jacobi field along $ \upgamma_{z,\xi}$ satisfying the initial condition $J_{i}(0,\xi) = 0$, $\nabla_tJ_{i} (0,\xi) = \partial_{i}\xi$. Furthermore, $\left.\frac{\partial}{\partial x^d}\right\vert_{\Exp_z t\xi} = \dot{\upgamma}_{z,\xi}$. Hence, in these coordinates, the metric coefficient $g_{i,j}$ along $\upgamma_{z,\xi}$ is computed as
\begin{equation}\label{eq:metric_jac}
    \left\langle \frac{\partial}{\partial x^i},\frac{\partial}{\partial x^j} \right\rangle_{\Exp_z t\xi} = \langle J_{i}(t,\xi), J_{j}(t,\xi)\rangle_{\Exp_z t\xi} \eqsp,
\end{equation}
for $1\leq i,j\leq d-1$. As for $\frac{\partial}{\partial x^d}$, since it corresponds to the tangential Jacobi field $\dot{\upgamma}_{z,\xi}$, we simply have $\Vert\frac{\partial}{\partial x^d}\Vert =\Vert\xi\Vert= 1$ by definition, and $\langle\frac{\partial}{\partial x^d},\frac{\partial}{\partial x^i}\rangle = 0$ by \cite[Equation II.5.7]{chavel_riemannian_2006}. In contrast, the Jacobi fields $J_{i}$, $1\leq i\leq d-1$, are normal Jacobi fields.

The Jacobi equation can be restated as a differential equation in the tangent space $T_z\msm$, by introducing the linear map $\mcr_{z,\xi}(t)\colon T_z\msm \to T_z\msm$ defined as
\begin{equation}
    \mcr_{z,\xi}(t)v = \tau_{z,\xi}(t)^{-1} \parentheseDeux{ \curv\left(\dot{\upgamma}_{z,\xi}(t),\tau_{z,\xi}(t) v\right)\, \dot{\upgamma}_{z,\xi}(t)}\eqsp,
  \end{equation}
for any $v\in T_z\msm$.
We consider the associated linear second-order ODE
\begin{equation}\label{eq:ode_A_supp}
    \ddot{\scrA}_{z,\xi} + \mcr_{z,\xi}(t)\scrA_{z,\xi} = 0\eqsp,
\end{equation}
with initial conditions $\scrA_{z,\xi}(0)=0$ and $\dot{\scrA}_{z,\xi}(0)=\Id$.

Using the symmetries of the curvature tensor (see \cite[Equation I.4.4]{chavel_riemannian_2006}), the operator $\mcr_{z,\xi}(t)$ vanishes on $\R\xi$ for all $t\geq 0$. Consequently, any solution $\scrA_{z,\xi}$ of \eqref{eq:ode_A_supp} can be viewed as an endomorphism of $\xi^{\perp}\subset T_z\msm$, the orthogonal complement of $\R\xi$. Denote $(t,\xi)\in (0,+\infty)\times \Smsm{z} \mapsto \scrA_{z,\xi}(t)$ the solution of \eqref{eq:ode_A} associated with $\xi$.

We then have, for every $1\leq i\leq d-1$,
\begin{equation}
    \left.\frac{\partial}{\partial x^i}\right\vert_{\Exp_z t\xi} = J_{i}(t,\xi) = \tau_{t,\xi}\scrA_{z,\xi}(t)\partial_{i}\xi\eqsp.
\end{equation}
It follows that the square root of the determinant of \eqref{eq:metric_jac}, that appears in the change of measure formula \eqref{eq:change_of_measure}, is equal to  $\det\scrA(t,\xi(v))$ \cite[Theorem III.3.1]{chavel_riemannian_2006}. Rather than integrating using charts on $\Smsm{z}$, we can integrate directly over $\Smsm{z}$ using the measure $\mu_z$ induced from the Euclidean Lebesgue measure. We have
\begin{align}
    \int_{\msu'}f\rmd \Rvol &= \int_{\msu}\int_0^{\cuttime_z(\xi(v))}f(\Exp_z t\xi(v))\det\scrA_{z,\xi(v)}(t)\rmd t \rmd v \\
    &= \int_{\mso}\int_0^{\cuttime_{x_1}(\xi)} f(\Exp_z t\xi)\det\scrA_{z,\xi}(t)\rmd t \rmd \mu_z(\xi)\eqsp,
\end{align}
which is then extended to the whole manifold as
\begin{equation}
    \int_{\msm}f\rmd \Rvol
    = \int_{\Smsm{z}}\int_0^{\cuttime_{x_1}(\xi)} f(\Exp_z t\xi)\det\scrA_{z,\xi}(t)\rmd t \rmd \mu_z(\xi)\eqsp.
\end{equation}

\paragraph*{Computing $\det\scrA$}
If $\msm$ has constant sectional curvature $\kappa$, normal Jacobi fields are given by \eqref{eq:jacobi_cst}: it follows that $\scrA_{z,\xi}(t) = s_{\kappa}(t)I_{d-1}$, with $I_{d-1}$ the identity matrix in dimension $d-1$, hence $\det\scrA_{z,\xi}(t) = s_{\kappa}(t)^{d-1}$ for any $(t,\xi)\in \Smsm{z}$.

If $\msm$ is a locally symmetric space (see \cite[Chapter 10]{lee_introduction_2018}), $\det \scrA$ is available in closed form as well, through the Jacobian of the exponential map \cite{chevallier_exponential-wrapped_2022}. Given $v\in T_z\msm$ a non-zero vector, it holds
\begin{equation}
\label{eq:jac_exp_detA}
    \Jac_{z,v} \coloneqq \det \left[\tau_{z,v}(1)^{-1}\circ \rmd {\Exp_z}_{|v}\right] = \frac{1}{\Vert v\Vert^{d-1}}\det \scrA_{z,v/\Vert v\Vert}\left(\Vert v\Vert\right)\eqsp,
\end{equation}
where the composition with the isometry $\tau_{z,v}(1)^{-1}$ allows for a non-ambiguous definition of the determinant. On locally symmetric spaces, the left-hand term is explicit (\cite[Theorem 4.1]{chevallier_exponential-wrapped_2022}):
\begin{equation}\label{eq:jac_exp}
    \Jac_{z,v} = \prod_{i = 1}^d \left(\frac{\sinh(\sqrt{-\lambda_i(R_u)})}{\sqrt{-\lambda_i(R_u)}}\right)^{n_i}\eqsp,
\end{equation}
with $\lambda_i(R_u)$ the $i$-th complex eigenvalue (and $n_i$ its algebraic multiplicity) of the linear map $v\in T_z\msm \mapsto R_u(v) = \curv(v,u)u \in T_z\msm$.

\section{Experiments details}
\label{app:expe}
In this section, we provide additional information regarding the practical implementation of \frips{} for the three examples we consider in \Cref{sec:expe}.

\subsection{Algorithmic details}
\label{app:details}
\paragraph*{Initializing the MCMC chains} When using an MCMC method for posterior sampling, \eg{} MALA, the iterative nature of \frips{} allows for some simplifications. As already mentioned in \Cref{sec:mcmc_method}, once the first MCMC chain is run, we can reuse it to initialize the next one. However, we must keep in mind that the posterior density may have a constrained support, due to the expression of $p_{t|1}$ \eqref{eq:pt1}.
If, given a state $(t_k,X_{k})$ of \Cref{alg:1}, the final MCMC sample $x_1$ obtained from the previous chain at time $t_{k-1}$ does not satisfy the support constraint $X_k\in\mso_{t_k}(x_1)$, due to the update of $t_k, X_k$, we may project it within the appropriate domain. From \Cref{prop:diffeo}, the condition $X_k\in\mso_{t_k}(x_1)$ rewrites as
\begin{equation}
    \dist(X_k,x_1) < (1-t_k)\cuttime_{x_1}\left(\frac{\Log_{x_1}X_k}{\dist(x_1,X_k)}\right)\eqsp.
\end{equation}
Thus, replacing $x_1$ by
\begin{equation}
\upgamma_{X_k}^{x_1}(s) = \Exp_{x_1}\left((1-s)\Log_{x_1}X_k\right)\eqsp,
\end{equation}
with $s$ slightly smaller than $(1-t_k)\cuttime_{x_1}\left(\frac{\Log_{x_1}X_k}{\dist(x_1,X_k)}\right)/\dist(x_1,X_k)$ provides a plausible point to initialize the next chain. Similarly, we use this projection procedure to reuse the MCMC samples when running the RLA initialization (\Cref{alg:ula}).

The question of initializing the very first MCMC chain remains. Given a starting point $X_0$, if we initialize the MCMC chain targeting $\pi_{1|t_0}(\cdot|X_0)$ arbitrarily (for instance, distributed as $\pi_0$), it might lie outside of the support. One solution is to resort to the projection procedure we just described. Alternatively, we can sample a random vector $v$  within an appropriate disk in $T_{X_0}\msm$, and set $x_1 = \Exp_{X_0}v$. If $\msm=\sphere^d$, we can pick $\Exp_{X_0}v$, with $\Vert v\Vert < (1-t_0)\pi$ chosen uniformly, as an initial state, which is what we do in practice.

\paragraph*{MALA on manifold}
In our numerical experiments, we include a Metropolis-Hasting step after the usual RLA proposal (\Cref{alg:sample_ula}). While this step usually induces a small bias on non-trivial manifold, we find it empirically improves the sampling performance. Furthermore, it allows for geometrically increase/decrease the step-size $\delta$ to maintain a given acceptance rate. Since the posterior densities \eqref{eq:p1t} support gradually shrink as $t$ approaches $1$, the automatic tuning of the step-size is an important feature.

Given a target density $p$ on a $d$-dimensional Riemannian manifold $\msm$ and a current point $x\in\msm$, the RLA proposal is \cite{gatmiry_convergence_2022}
\begin{equation}
x^* = \Exp_y(\sqrt{2\delta} Z),\quad y = \Exp_x\bigl(\delta \grad \log p(x)\bigr)\eqsp,
\end{equation}
where $Z \in T_y \msm \simeq \R^d$ is drawn from $\mathcal{N}(0,1)$.

The corresponding proposal distribution given $x$ is therefore $({\Exp_y})_{\sharp}\mathcal{N}(0,2\delta)$. On a general Riemannian manifold, $\Exp_y$ usually fails to be a global diffeomorphism. As a result, the exact proposal density should account for all points $z \in \msm$ satisfying $\Exp_y(\sqrt{2\delta}\Log_y z) = x^*$. Assuming a small step size $\delta$, we neglect these contributions and approximate the proposal density as
\begin{equation}
\label{eq:trans_density}
f(x^*|x) = \frac{\mathcal{N}(\Log_y x^*;0,2\delta)}{\Jac_{y,\Log_y x^*}}\eqsp,
\end{equation}
where the denominator denotes the Jacobian determinant of $\rmd\Exp_y$ at $\Log_y x^*$, see \eqref{eq:jac_exp_detA}. Empirically, we find that omitting this determinant has little impact on the algorithm.
The Metropolis–Hastings acceptance probability is then computed as
\begin{equation}
\alpha = \min\Bigl(1, \frac{p(x^*)}{p(x)} \frac{f(x|x^*)}{f(x^*|x)}\Bigr)\eqsp.
\end{equation}
Note that this procedure is biased due to the approximation on $f$.

\subsection{The $d$-sphere}
\label{app:sphere}
In this section, we summarize useful formulas and algorithmic technicalities in the case $\msm = \sphere^d$ with $\pi_0$ the uniform distribution. We also provide all the hyperparameters used in the experiments of \Cref{sub:sphere_exp}.

\paragraph*{Validity of \Cref{ass:finite}}
\Cref{ass:finite}-\ref{ass:finite_iii} follows from the compactness of $\sphere^d$. We verify that \Cref{ass:finite}-\ref{ass:finite_ii} holds under mild assumptions on the density $p_1$ of the target distribution $\pi$, namely that $p_1>0$ and has regularity $C^1$ on $\sphere^d$. 

First, following \Cref{prop:density}, for every $x_t\in\sphere^d$, 
\begin{equation}
\label{eq:pt1_sphere}
    x_t \in \sphere^d\setminus\{-x_1\} \mapsto p_{t|1}(x_t|x_1) = \Rvol(\sphere^d)^{-1} \frac{\1(\widehat{x_tx_1}< (1-t)\pi)}{1-t}\left(\frac{\sin\left(\frac{\widehat{x_tx_1}}{1-t}\right)}{\sin(\widehat{x_tx_1})}\right)^{d-1} \eqsp,
\end{equation}
as $\det\scrA(s) = \sin(s)^{d-1}$ depends solely on $s$ \cite[Section III.1]{chavel_riemannian_2006}. The factor $\Rvol(\sphere^d)^{-1}$ corresponds to the uniform density on $\sphere^d$.

Let $t\in\coint{0,1}$. Leveraging the submanifold structure of $\sphere^d\subset\R^{d+1}$, to show a function defined on $\sphere^d$ is $C^1$, it suffices to prove it is the restriction of a $C^1$ function defined on an open subset of $\R^{d+1}$.
Let $x_t\in\sphere^d$. Using spherical coordinates \eqref{eq:change-variable} centred at $x_t$ and the expression of $p_{t|1}$, we get
\begin{align}
    p_t(x_t) &= \int_{\sphere^d}p_{t|1}(x_t|x_1)p_1(x_1)\Rvol(\rmd x_1) \\
    &= \Rvol(\sphere^d)^{-1}\int_{S_{x_t}\sphere^d}\int_0^{\pi} \1\left(\frac{s}{1-t}<\pi\right) \frac{\sin^{d-1}(s/(1-t))}{1-t}p_1(\Exp_{x_t}s\xi) \rmd s\rmd \mu_{x_t}(\xi) \\
    &= \Rvol(\sphere^d)^{-1}\int_{S_{x_t}\sphere^d}\int_0^{\pi}  \sin^{d-1}(s)p_1(\Exp_{x_t}s(1-t)\xi) \rmd s\rmd \mu_{x_t}(\xi) \\
    &= \Rvol(\sphere^d)^{-1}\int_{\sphere^d} p_1(\psi_t(x_1;x_t)) \Rvol(\rmd x_1)\eqsp.
\end{align}
We check that $p_t$ is the restriction to $\sphere^d$ of a function $\msu\subset\R^{d+1}\rightarrow \R$, with $\sphere^d\subset \msu$, that is $C^1$ in the usual Euclidean sense. 
For almost every $x_1\in\msm$, $x_t  \mapsto \psi_t(x_1;x_t)$ is readily extended to a  $C^1$ function on $\R^{d+1}\setminus\{0\}$.

Assuming $p_1$ is $C^1$, the usual Leibniz rule applies, as we integrate on the compact set $\sphere^d$ with respect to the finite measure $\Rvol$. $p_t$ is thus $C^1$ as well. With similar reasoning, we see that $t\mapsto p_t(x)$, $x\in\msm$ being fixed, is continuous.

We prove $\bfu_t$ is $C^1$, hence locally Lipschitz continuous, using the same arguments. We assumed that $p_1$ is non-vanishing and as a consequence, $p_t$ is non-vanishing as well. Then, for every $x_t\in\sphere^d$, using spherical coordinates again,
\begin{align}
    &\bfu_t(x_t) = \int_{\sphere^d} \frac{\Log_{x_t}x_1}{1-t} \frac{p_{t|1}(x_t|x_1)}{p_t(x_t)} p_1(x_1)\Rvol(\rmd x_1) \\
    &= \frac{\Rvol(\sphere^d)^{-1}}{p_t(x_t)}\int_{S_{x_t}\sphere^d}\int_0^{\pi} \frac{s\xi}{1-t}\1\left(\frac{s}{1-t}<\pi\right) \frac{\sin^{d-1}(s/(1-t))}{1-t}p_1(\Exp_{x_t}s\xi)\rmd s\rmd\mu_{x_t} \\
    &= \frac{\Rvol(\sphere^d)^{-1}}{p_t(x_t)}\int_{S_{x_t}\sphere^d}\int_0^{\pi} s\xi \sin^{d-1}(s) p_1(\Exp_{x_t}(1-t)s\xi)\rmd s\rmd\mu_{x_t} \\
    &= \frac{\Rvol(\sphere^d)^{-1}}{p_t(x_t)}\int_{\sphere^d} \Log_{x_t}x_1  \times p_1(\psi_t(x_1;x_t)) \Rvol(\rmd x_1)\eqsp.
\end{align}
For almost every $x_1\in\msm$, $x_t\mapsto\Log_{x_t}x_1$ can be extended to a $C^1$ function defined almost everywhere on $\R^{d+1}$, similar to $x_t\mapsto \psi_t(x_1;x_t)$. The usual Leibniz rule applies again, proving that $x_t\in\sphere^d \mapsto \bfu_t(x_t)$ is the restriction of a $C^1$ function.
For every fixed $x\in\sphere^d$, continuity of $t\in\coint{0,1}\mapsto \bfu_t(x)$ holds as well.

\paragraph*{Posterior density and score} Fix $t\in\coint{0,1}$ and $x_t\in\sphere^d$. From \eqref{eq:pt1_sphere},
\begin{equation}
    x_1 \in \sphere^d\setminus\{-x_t\} \mapsto p_{1|t}(x_1|x_t) \propto \1(\widehat{x_tx_1}< (1-t)\pi)\left(\frac{\sin\left(\frac{\widehat{x_tx_1}}{1-t}\right)}{\sin(\widehat{x_tx_1})}\right)^{d-1} \frac{q_1(x_1)}{1-t}\eqsp.
\end{equation}
 The score $\grad \log p_{t|1}(\cdot|x_1)$ appearing in Tweedie's formula \eqref{eq:tweedie}, essential to run \Cref{alg:ula}, is computed by differentiating the logarithm of the Jacobian term $\Jac_t(\cdot;x_1)$ in the ambient space $\R^{d+1}$ and projecting on the tangent space. This yields, for every $x_1\in\sphere^d$, $x_t\in \mso_t(x_1)$,
\begin{equation}\label{eq:score_ula_sphere}
    \grad_{x_t}\log p_{t|1}(x_t|x_1) = (d-1)\left(\frac{\cos(\widehat{x_tx_1})}{\sin(\widehat{x_tx_1})}-\frac{1}{1-t}\frac{\cos\left(\frac{\widehat{x_tx_1}}{1-t}\right)}{\sin\left(\frac{\widehat{x_tx_1}}{1-t}\right)}\right)\frac{x_1-\cos(\widehat{x_tx_1})x_t}{\sin(\widehat{x_t x_1})}\eqsp.
\end{equation}
Similarly, given $x_t$, the score with respect to the $x_1$ variable of $p_{1|t}(\cdot|x_t)$ (used \eg{} in \Cref{alg:sample_ula}) writes
\begin{align}
    \grad_{x_1}\log p_{1|t}(x_1|x_t) = (d-1)\left(\frac{\cos(\widehat{x_tx_1})}{\sin(\widehat{x_tx_1})}-\frac{1}{1-t}\frac{\cos\left(\frac{\widehat{x_tx_1}}{1-t}\right)}{\sin\left(\frac{\widehat{x_tx_1}}{1-t}\right)}\right)\frac{x_t-\cos(\widehat{x_tx_1})x_1}{\sin(\widehat{x_t x_1})} \\
    + \grad \log q_1(x_1)\eqsp.
\end{align}

Recall that in both \Cref{sub:sphere_exp,sub:grass_exp}, $q_1$ is proportional to a mixture of Riemannian normal distributions. The score of a Riemannian normal distribution with density proportional to $x\mapsto \exp\left(-\frac{\dist(\mu,x)^2}{2\sigma^2}\right)$ can be computed as 
\begin{equation}
x\mapsto\frac{\Log_{x}\mu}{\sigma^2}\eqsp,
\end{equation}
for almost all $x\in\sphere^d$. This formula actually holds for any manifold satisfying \Cref{ass:1}. 

\paragraph*{Sampling from $\nu_{1|t}$} Assuming $\pi_0$ is the uniform distribution, samples from $\nu_{1|t}(\cdot|x_t)$, the distribution with un-normalized density $x_1 \mapsto p_{t|1}(x_t|x_1)$ are easily obtained.
Indeed, for any measurable function $h$, using spherical coordinates \eqref{eq:change-variable} and the expression of $p_{t|1}$,
\begin{align}
    \int_{\sphere^d} h(x_1) p_{t|1}(x_t|x_1) \Rvol(\rmd x_1) &\propto \int_{S_{x_t}\sphere^d}\int_0^{\pi} h(\Exp_{x_t}s\xi)\frac{\sin^{d-1}\left(\frac{s}{1-t}\right)}{1-t} \1(s<(1-t)\pi) \rmd s \mu_{x_t}(\rmd \xi) \\
    &= \int_{S_{x_t}\sphere^d}\int_0^{\pi} h(\Exp_{x_t}s(1-t)\xi)\sin^{d-1}(s)  \rmd s \mu_{x_t}(\rmd \xi) \\
    &= \int_{\sphere^d} h\circ\psi_t(x_1;x_t) \Rvol(\rmd x_1)\eqsp.
\end{align}
 This trick uses the fact that $\sphere^d$ has constant sectional curvature, which implies that $\det\scrA_{x_1,\xi_t} = \det\scrA_{x_t,\tilde{\xi}_t}$, where $\xi_t=\Log_{x_1}x_t/\dist(x_t,x_1)$ and $\tilde{\xi}_t = \Log_{x_t}x_1/\dist(x_t,x_1)$. As a result, it holds $\nu_{1|t}(\cdot|x_t) = {\psi_t(\cdot;x_t)}_{\sharp}\pi_0$: an exact sample from $\nu_{1|t}(\cdot|x_t)$ is obtained by applying the map $\psi_t(\cdot;x_t)$ to a uniform sample, which is the procedure we use within \Cref{alg:rejection,alg:importance}.

\paragraph*{Mixture of Gaussian distributions on $\sphere^d$: hyperparameters}
The bi-modal targets are such that $\sigma_4 = \pi/10, \sigma_{16}=\pi/12$ and $\sigma_{32}=\sigma_{64} = \sigma_{92} = \pi/18$, with a mode at each pole.

\begin{itemize}
\item First, initial samples from $\pi_0$ are refined using \Cref{alg:ula}, with $128$ RLA steps for the initialization and step-size $\delta = 0.05$ (except in dimension $d=96$ where $\delta=0.01$).
\item Posterior sampling being more challenging for smaller $t$, we suggest using a larger part of the computational budget at the beginning of the algorithm: during RLA initialization at $t_0$, we set $M=320$.
\item We use $K=128$ uniform timesteps ranging from $t_0$ to $0.99$, to avoid numerical instabilities when $t\approx 1$.

\item Each MALA posterior sampling step (\Cref{alg:sample_ula}) uses $n_{\text{chains}}=8$ chains and $M=32$ steps, keeping the last $8$ samples of each chain. We geometrically decrease or increase the step-size ( initially set to $0.01$) of each chain to maintain an acceptance rate of $\approx 0.57$. The step-size is propagated from the chain at time $t_k$ to the one at time $t_{k+1}$, along with the final state of the chain. 
\item Each IS/RS posterior sampling step (\Cref{alg:importance,alg:rejection}) uses $M\times n_{\text{chains}}$ proposals, to ensure fairness. If every proposal is rejected during RS, random samples from the support of the posterior density are used instead. For simplicity, RS is informed with $q^\star$. 
\item We run MALA chains targeting the distribution $\pi_1$ directly, with the same overall number of MCMC steps, and similarly for direct IS/RS sampling of $\pi_1$.
\end{itemize}

\paragraph*{Mixture of Gaussian distributions on $\sphere^d$: local metric}
\begin{table}
\centering
\caption{Wasserstein distance between approximate samples and exact ones for the bi-modal Gaussian mixture target \eqref{eq:mog_target}. The setting is identical to \Cref{tab:weight}.}

\begin{tabular}{@{}l c c c c c@{}}
\hline
\textsc{Methods} 
& $d=4$ 
& $d=16$ 
& $d=32$ 
& $d=64$ 
& $d=96$\\
\hline
\frips{} \textsc{MALA} 
& $0.22 \pm \text{\scriptsize 0.01}$ 
& $\mathbf{ 0.61\pm \text{\scriptsize 0.00}}$ 
& $0.86\pm \text{\scriptsize 0.00}$ 
& $\mathbf{1.09\pm \text{\scriptsize 0.00}}$
& $\mathbf{1.16\pm \text{\scriptsize 0.00}}$\\

\frips{} \textsc{IS} 
& $0.17\pm \text{\scriptsize 0.01}$ 
& $ 0.63\pm \text{\scriptsize 0.00}$ 
& $ 0.86\pm \text{\scriptsize 0.01}$ 
& $ 1.16\pm \text{\scriptsize 0.00}$
& $ 1.24\pm \text{\scriptsize 0.00}$\\

\frips{} \textsc{RS} 
& $ 0.23\pm \text{\scriptsize 0.01}$ 
& $ 0.85\pm \text{\scriptsize 0.00}$ 
& $ 1.08\pm \text{\scriptsize 0.00}$ 
& $ 1.18\pm \text{\scriptsize 0.00}$
& $ 1.24\pm \text{\scriptsize 0.00}$\\
\hline
\textsc{MALA} 
& $0.46 \pm \text{\scriptsize 0.02}$ 
& $ 0.69\pm \text{\scriptsize 0.00}$ 
& $ 0.92\pm \text{\scriptsize 0.01}$ 
& $\mathbf{1.09\pm \text{\scriptsize 0.00}}$ 
& $ 1.18\pm \text{\scriptsize 0.00}$\\
\textsc{IS} 
& $\mathbf{0.16} \pm \text{\scriptsize 0.01}$ 
& $\mathbf{0.61\pm \text{\scriptsize 0.00}}$ 
& $\mathbf{0.78\pm \text{\scriptsize 0.01}}$ 
& $ 1.10\pm \text{\scriptsize 0.00}$
& $ 1.24\pm \text{\scriptsize 0.00}$\\
\textsc{RS} 
& $ \mathbf{0.16}\pm \text{\scriptsize 0.01}$ 
& $\mathbf{0.61\pm \text{\scriptsize 0.00}}$ 
& N/A 
& N/A 
& N/A\\
\hline
\textsc{Base} &$0.15\pm \text{\scriptsize 0.00}$&$0.60\pm \text{\scriptsize 0.00}$&$0.76\pm \text{\scriptsize 0.00}$& $1.02\pm \text{\scriptsize 0.00}$ 
& $1.15\pm \text{\scriptsize 0.00}$\\
\hline
\end{tabular}

\label{tab:wass}
\end{table}
Even though we focus on the estimation of the dominant mode weight, we also compute the Wasserstein distance between each batch of samples and exact samples from $\pi_1$ (numerically solving the optimal transport problem with a cost matrix matching the Riemannian distance on $\sphere^d$). We report the results in \Cref{tab:wass}, including the distance between two batches of exact samples. While this metric corroborates the relative weight estimation error in low-dimensional settings (see also \Cref{fig:weight}), it becomes less informative in higher dimensions, where it fails to adequately discriminate between the relative quality of the samples. 

\subsection{Stereographic projection} We detail below how to compute the target density and its score once projected to $\sphere^d$, and provide the hyperparameters of the experiments of \Cref{sub:stereo}.
\label{app:stereo}
\paragraph*{Density and score}
Given a scale parameter $R>0$, the stereographic projection operator $\mathrm{SP} : \sphere^d\setminus{N} \subset \mathbb{R}^{d+1} \rightarrow \mathbb{R}^d$, where $N$ denotes the north pole $(0,...,0,1)\in\R^{d+1}$, is defined by
\begin{equation}
    \mathrm{SP}(z) = R\frac{z-z_{d+1}e_{d+1}}{1-z_{d+1}}\eqsp,
\end{equation}
where $z_{d+1}$ is the $d$-th component of $z$ in the canonical basis $(e_i)_{i=1}^{d+1}$ of $\mathbb{R}^{d+1}$.
The projected target distribution is defined as $\pi_1^S = {\mathrm{SP}^{-1}}_{\sharp}\pi_1$, and the associated unnormalized density is given by
\begin{eqnarray}
    z \in \sphere^d\setminus{N}&\mapsto q_1^S(z) = q_1(\mathrm{SP}(z)) \left(\frac{\Vert \mathrm{SP}(z)\Vert^2 + R^2}{2R}\right)^d\eqsp,
\end{eqnarray}
where the multiplicative factor is the Jacobian determinant of the transformation. 

To apply first-order sampling schemes such as MALA (\Cref{alg:sample_ula}) within \frips{}, we compute the gradient log of this density. With $\cdot^T$ denoting the transpose, it is given by
\begin{equation}
    \grad\log q_1^S(z) = \left[\grad \log q_1(\text{SP}(z))\right]^T\rmd \text{SP}_z + d\times\frac{e_{d+1}-z_{d+1}z}{1-z_{d+1}}\eqsp.
\end{equation}
The differential $\rmd \text{SP}_z$ admits the explicit expression
\begin{align}
\begin{split}
    &\rmd\text{SP}_z = \frac{R}{1-z_{d+1}}\left[(z-z_{d+1}e_{d+1})\left(\frac{e-z_{d+1}z}{1-z_{d+1}^2}\right)^T\right. \\
    &\hspace{4cm}\left.+I_{d+1}-zz^T-(e-z_{d+1}z)\left(\frac{e-z_{d+1}z}{1-z_{d+1}^2}\right)^T\right]\eqsp.
\end{split}
\end{align}
This expression is obtained by computing the differential in the ambient space $\R^{d+1}$ and subsequently projecting onto the tangent space of $\sphere^d$.
The posterior density $p_{1|t}$ and its score are then derived similarly to \Cref{app:sphere}.

\paragraph*{Mixture of Student's distributions: hyperparameters} 
While the Euclidean \frips{} runs conducted in \Cref{sub:stereo} are essentially a reparametrization of \cite{grenioux_stochastic_2024}, additional details are provided in \Cref{app:euclidean}. 

For both the $\R^d$ and $\sphere^d$ runs of \frips{}, we use $K=128$ uniform timestep ranging from $t_0$ to $0.99$, and 128 RLA steps during initialization with step-size $\delta=0.01$. MALA posterior sampling steps use $n_{\text{chains}}=8$, $M=32$ steps, and we keep the last 8 samples of each chain for Monte-Carlo estimation. During RLA initialization, we set $M=160$ instead. For fairness, naive sampling methods with MALA, in both $\R^d$ and $\sphere^d$, use the same total number of MALA steps.

Following the recommendation of \cite{yang_stereographic_2024}, we use a radius $R = \sqrt{d}$ for the stereographic projection.

While the uniform distribution is the natural choice for $\pi_0$ on $\sphere^d$, we use a standard Gaussian in the $\R^d$ case. In \cite{grenioux_stochastic_2024}, it is suggested to adjust the Gaussian base distribution using estimators of the mean and variance of the target distribution. The heavy-tailed target distributions of our experiments have undefined mean and variance, hence this default choice.

\paragraph*{Heavy-tail metric} The heavy-tail metric reported on \Cref{fig:student} is defined in \eg{} \cite{allouche_ev-gan_2022}. We include it here for completeness.
Given $n$ exact samples from $\pi_1$, $(x_i)_{1\leq i\leq n}$, for every $1\leq j\leq d$, order the $j$-th components of the samples, so we have $x_1^j\leq...\leq x_n^j$. Then, if $(\tilde{x}_i)_{1\leq i\leq n}$ are approximate samples, set for any $\xi\in\ccint{0,1}$:
\begin{equation}\label{eq:msle}
    \text{MSLE}(\xi) = \frac{1}{d \lceil (1-\xi)n\rceil}\sum_{j=1}^d\sum_{i=1}^{\lceil (1-\xi)n \rceil} \left(\log x_{n-i+1}^j - \log \tilde{x}_{n-i+1}^j\right)^2\eqsp,
\end{equation}
granted every $x_{n-i+1}^j,\tilde{x}_{n-i+1}^j, 1\leq j\leq d, 1\leq i\leq \lceil (1-\xi)n \rceil$ is non-negative. By setting $\xi$ close to $1$, this metric measures the discrepancy between tail samples of the target distribution and approximate ones. In our numerical experiments, we set $\xi=0.99$.

\subsection{Grassmann manifold}
\label{app:grassmann}
We introduce the properties of the Grassmann manifold required to run \frips{}, and specify the hyperparameters used in the corresponding experiment (\Cref{sub:grass_exp}).

\paragraph*{Definition and basic properties}
We refer to \cite{bendokat_grassmann_2020} for a comprehensive introduction.
The Grassmann manifold $\grass$ is the set of all $p$-dimensional linear subspaces of $\R^n$.
It can be identified with the set of rank-$p$ orthogonal projectors $P\in\R^{n\times n}$, satisfying $P^T=P$ and $P^2=P$.
An alternative and useful viewpoint is to see $\grass$ as a quotient of the orthogonal group $\mso(n)$, the group of $n\times n$ matrices $Q$ such that $Q^T Q = I_n$ (the identity matrix of $\R^n)$.
Since $\mso(n)$ is a compact Lie group, its quotients naturally inherit a smooth manifold structure.

To make this construction explicit, consider the \emph{Stiefel manifold} $\stiefel$, defined as the set of $n\times p$ matrices $U$ satisfying $U^T U = I_p$.
Such matrices are obtained by retaining the first $p$ columns of an element of $\mso(n)$.
Each point $P\in\grass$ then corresponds to an equivalence class of matrices $U\in\stiefel$ through the projection map $\pi(U)\coloneqq UU^T=P$.

The Grassmann manifold $\grass$ is a compact, embedded submanifold of $\mathrm{Sym}(n)$, the space of $n\times n$ symmetric matrices, with dimension $d=p(n-p)$.
Equipped with its canonical Riemannian structure, $\grass$ is complete and connected \cite[Section 4.1]{bendokat_grassmann_2020}, and therefore satisfies \Cref{ass:1}.

In practice, it is computationally advantageous to work with $n\times p$ Stiefel representatives rather than $n\times n$ projectors.
Importantly, both the exponential and logarithm maps on $\grass$ can be computed entirely using Stiefel representatives and their tangent spaces; see \cite[Proposition 3.3 and Algorithm 5.3]{bendokat_grassmann_2020}.

Let $U\in\stiefel$.
The tangent space $T_U\stiefel$ admits an orthogonal decomposition into a vertical space $\mathsf{Ver}_U\stiefel$, defined as the kernel of $\rmd \pi_U$, and a horizontal space $\mathsf{Hor}_U\stiefel$, given by its orthogonal complement with respect to the Riemannian metric on $\stiefel$.
The tangent space at $P=UU^T\in\grass$ can then be identified with $\mathsf{Hor}_U\stiefel$.
As a consequence, any tangent vector $\Delta\in T_P\grass$ admits a \emph{horizontal lift} $\Delta_U^{\mathsf{hor}}\in\mathsf{Hor}_U\stiefel$, which can be written as
\begin{equation}
\Delta_U^{\mathsf{hor}} = U_\perp \Delta\eqsp,
\end{equation}
where $U_\perp$ denotes any orthogonal completion of $U$.

\paragraph*{Posterior density}
Following \cite[Section 5.1]{bendokat_grassmann_2020}, the cut time at $P$ in direction $\Delta\in T_P\grass$ is
\begin{equation}
\label{eq:cut_time_grass}
    \cuttime_P(\Delta) = \frac{\pi}{2\sigma_1}\eqsp,
\end{equation}
with $\sigma_1$ the largest the singular value of $\Delta_U^{h}$, where $U\in\stiefel$ is any Stiefel representative of $P$. From \Cref{prop:density}, the cut time defines the support $\mso_t(x_1)$ of the density $p_{t|1}(\cdot|x_1)$ for every $t\coint{0,1},\,x_1\in\grass$.

Then, relying on the symmetric space structure of $\grass$, the Jacobian determinant of $\Exp_P$ at $\Delta$ can be computed from the singular values of $\Delta_U^{h}$. From \eqref{eq:jac_exp}, we can derive \cite[Section 4.2.1]{chevallier_exponential-wrapped_2022}:
\begin{equation}
\label{eq:jac_exp_grassmann}
\Jac_{P,\Delta} = \prod_{1\leq i < j \leq q} \sinc(\sigma_i + \sigma_j)\sinc(\sigma_i -\sigma_j)\prod_{i=1}^q\sinc(\sigma_i)^{\vert n-2p\vert}\eqsp,
\end{equation}
where $x\mapsto \sinc(x) = \sin(x)/x$ and $\sigma_i,\, 1\leq i\leq q\coloneqq \min(p,n-p)$, denote the singular values of $\Delta_U^{\mathsf{hor}}$ counted with multiplicity one. A closed-form expression of $\det\scrA$ then follows from \eqref{eq:jac_exp_detA}, and as a consequence, we obtain a closed-form expression of the density $p_{t|1}$ of \Cref{prop:density} as well. 

Notice that the Jacobian determinant \eqref{eq:jac_exp_grassmann} \emph{does not} vanish on the boundary of the injectivity domain, which corresponds to tangent vectors $\Delta$ such that $\sigma_1 = \pi/2$. This differs from the $\sphere^d$ case, where we easily check that the Jacobian term \emph{does vanish}: as a result, the density $p_{t|1}$ from \Cref{prop:density} is continuous on $\sphere^d$, but not on $\grass$, where it only continuous on its support.

\paragraph*{Score of the posterior density} To use a first-order posterior sampling scheme like MALA within \frips{}, we need to compute the score $x_1\mapsto \grad \log p_{t|1}(x_t|x_1)$, $x_t\in\grass, t\in\coint{0,1}$ being fixed. We suggest using automatic differentiation, using the closed-form expression of $p_{t|1}$ derived above. Since we are working with Stiefel representatives, we actually compute the usual Euclidean gradient of $\log p_{t|1}(x_t|\cdot)\circ \pi$, with $\pi$ the projector $\stiefel \subset \R^{n\times p}\rightarrow \grass$, using automatic differentiation. Then, we project the result to $\mathsf{Hor}_{\tilde{x}_1}\stiefel$, where $\tilde{x}_1$ is a Stiefel representative of $x_1$. We refer to \cite[Section 3.3]{bendokat_grassmann_2020} for a precise justification.

\paragraph*{Sampling from $\nu_{1|t}$} Similar to the sphere case, the distribution $\nu_{1|t}(\cdot|x_t)$, with density $x_1\mapsto p_{t|1}(x_t|x_1)$ for a fixed $x_t\in\grass$, is actually the pushforward measure $\psi_t(\cdot;x_t)_{\sharp}\pi_0$, when $\pi_0$ is the uniform distribution on $\grass$. The justification is similar (see \Cref{app:sphere}), noticing that the horizontal lifts of $\Log_{x_1} x_t$ and $\Log_{x_t}x_1$ share the same singular values: this follows directly from the computation of the logarithm \cite[Algorithm 5.3]{bendokat_grassmann_2020}. As a consequence, it holds $\det\scrA_{x_1,\xi_t} = \det\scrA_{x_t,\tilde{\xi}_t}$, with $\xi_t = \Log_{x_1} x_t$ and $\tilde{\xi}_t = \Log_{x_t}x_1$.

\paragraph*{Mixture of Gaussian distributions on $\grass$: hyperparameters} We set $\sigma = 0.25$, and a mode distance $\dist(\mu_1,\mu_2) = 2.18$ and $\dist(\mu_1,\mu_2) =2.45$ for the low and high mode-sepration regimes, respectively. 
We use $K=128$ uniform timesteps ranging from $t_0$ to $0.99$, with one MALA chain of length $M=32$ at each posterior sampling step. The IMH  initialization (\Cref{alg:pseudo_mcmc}) uses $M=256$ steps, where each evaluation of the posterior $p_{t_0}$ is estimated with a batch of $512$ samples from $\nu_{1|t_0}$.

As explained in \Cref{sub:grass_exp}, Tweedie's formula \eqref{eq:tweedie} cannot be applied here due to the discontinuity exhibited above, which is why we settle for a zero-th order scheme for the initialization at $t_0$.

\subsection{Euclidean space} 
\label{app:euclidean}
In \Cref{sub:stereo}, we run \frips{} in the Euclidean space $\R^d$, which obviously satisfies \Cref{ass:1}. This particular setting is closely related to \eg{} \cite{grenioux_stochastic_2024}, and not the main object of this paper. Still, for completeness, we provide some details below.

\paragraph*{Posterior density}
The interpolation process \eqref{eq:interpolation_def} is simply $X_t = tX_1 + (1-t)X_0$, with $X_0\sim\pi_0 = \mathcal{N}(0,\sigma^2I_d),\, \sigma>0$. A simple change of variable shows that for every $t,x_t$,
\begin{equation}
    p_{t|1}(x_t|x_1) =  \mathcal{N}(x_t;tx_1,(1-t)^2\sigma^2I_d)\eqsp,
\end{equation}
which is consistent with \Cref{prop:density}.

\paragraph*{Validity of \Cref{ass:finite}} We can check the technical assumption \Cref{ass:finite}, assuming the target density $p_1$ has a finite first moment, which ensures \Cref{ass:finite}-\ref{ass:finite_ii}, and that it has regularity $C^1$. 

Let $t\in\coint{0,1}$. For every $x_t\in\R^d$, notice that $p_t(x_t) = \int p_{t|1}(x_t|x_1)p_1(x_1)\rmd x_1$ is non-negative and $C^1$ using the Leibniz rule and the fact that for every $x_1\in\R^d, x_t\mapsto p_{t|1}(x_t|x_1)$ is $C^1$ and bounded. Then,
\begin{equation}
    x_t\mapsto\bfu_t(x_t) = \frac{1}{p_t(x_t)}\int \frac{x_1-x_t}{1-t}p_{t|1}(x_t|x_1)p_1(x_1)\rmd x_1
\end{equation}
is $C^1$ using the Leibniz rule and the first moment assumption on $p_1$. Similarly, it is continuous in the $t$ variable, $x_t$ being fixed. Thus, \Cref{ass:finite}-\ref{ass:finite_iii} holds.

\section{Proofs}\label{app:proof}
\subsection{Proof of \Cref{prop:diffeo}}\label{proof:diffeo}
\begin{proof}

\ref{prop_smooth_psi_t_i}
Under \Cref{ass:1}, \cite[Theorem 16.34]{lee_introduction_2018} ensures that $\Rvol(\mathrm{Cut}(x_1)) = 0$ which implies by the Fubini theorem that $\Rvol \otimes \Rvol (\msd^{\complement}) = \int \rmd \Rvol(x_1) \Rvol(\cut(x_1)) = 0$. Then for $\Rvol\otimes\Rvol$-almost all $(x_0,x_1)$, $t\in\ccint{0,1}\mapsto\psi_t(x_0;x_1)$ is the unique minimizing geodesic between $x_0$ and $x_1$, hence smooth by definition of geodesics \cite[Chapter 4]{lee_introduction_2018}.

  \ref{prop_smooth_psi_t_ii}
  Let $t\in\coint{0,1}$ and $x_1\in\msm$. Under \Cref{ass:1}, the map $v\in\ID(x_1) \mapsto \Exp_{x_1}v$ is a smooth diffeomorphism onto its image $\mcut$ \cite[Proposition 5.19, Theorem 10.34]{lee_introduction_2018}, and its inverse $x\mapsto \Log_{x_1}x$ defines a diffeomorphism from $\mcut$ to $\ID(x_1)$. Furthermore $v\in \ID(x_1) \mapsto (1-t)v $ is a diffeomorphism valued in $\ID(x_1)$ by definition of $\ID(x_1)$ and since $t\neq 1$. We conclude that  $\psi_t(\cdot;x_1)$ is a diffeomorphism from $\mcut$ into its image. The map defined in \eqref{eq:inv_psi} is actually well defined for every $x_t\in \mcut$, and we easily verify using the properties of the exponential and logarithm maps that it is the desired inverse map when restricted to $\mso_t(x_1)$.

\ref{prop_smooth_psi_t_iii}
The fact that $\mso_t(x_1)$ is open is an immediate consequence of the fact that $\psi_t(\cdot;x_1)$ is a diffeomorphism and $\msm\setminus \cut(x_1)$ is open \cite[Theorem 16.34]{lee_introduction_2018}. Regarding the last point, any point  $x_t\in \mcut$ satisfies $\dist(x_1,x_t)<(1-t)\cuttime_{x_1}\left(\frac{\Log_{x_1}x_t}{\dist(x_1,x_t)}\right)$ if and only if $\Log_{x_1}x_t/(1-t)\in\ID(x_1)$ by definition of the cut time, and this statement is equivalent to
\begin{equation}
        \Log_{x_1}\Exp_{x_1}(\Log_{x_1}x_t/(1-t)) = \Log_{x_1}x_t/(1-t)\eqsp, \end{equation}
        which rewrites as $x_t = \psi_t(\psi_t^{-1}(x_t;x_1);x_1)\in\mso_t(x_1)$.
  
\end{proof}

\subsection{Proof of \Cref{lem:condi_vector_field}}
\label{proof:log}
\begin{proof}
By \Cref{prop:diffeo}, $t\in\ccint{0,1}\mapsto \psi_t(x_0;x_1)$ is smooth for almost every $(x_0,x_1)\in\msm\times\msm$; hence the conditional vector field $t\in\ccint{0,1}\mapsto \dot{X}_t = \dot{\psi_t}(X_0;X_1) \in T_{X_t}\msm$ is $\pi_{0,1}$-almost surely well-defined.

Recall that, by \eqref{eq:psi_def}, for every $(x_0,x_1)\in\msd,\, t\in\coint{0,1}$, $\dot{\psi}_t(x_0;x_1)$ is the velocity vector at time $t$ of the minimizing geodesic $\dot{\upgamma}_{x_0}^{x_1}$. Denote $x_t = \upgamma_{x_0}^{x_1}(t)$ and notice that the time-shift $s\in\ccint{0,1}\mapsto \upgamma_{x_0}^{x_1}(s(1-t)+t)$ coincides with the minimizing geodesic $s\in\ccint{0,1}\mapsto\upgamma_{x_t}^{x_1}(s)$, as $0\leq s(1-t)+t\leq 1$ for every $s\in\ccint{0,1}$, by uniqueness of $ \upgamma_{x_0}^{x_1}$. Therefore, since $\dot{\upgamma}_{x_t}^{x_1} (0) = \Log_{x_t}{x_1}$ by definition of minimizing geodesics and the logarithm map, we have $(1-t)\dot{\upgamma}_{x_0}^{x_1}(t) = \Log_{x_t}x_1$. Thus, almost surely,~\eqref{eq:lem:condi_vector_field} holds.
\end{proof}

\subsection{Proof of \Cref{prop:verify_prop_cont}}
\label{proof:mass}
\begin{proof}
By \Cref{lem:condi_vector_field}, for any smooth function with compact support $f:\ooint{0,1} \times \msm\to \rset$, almost surely, we have for every $t\in\ooint{0,1}$,
  \begin{equation}
    \frac{\rmd}{\rmd t}f(t,X_t) = \partial_tf(t,X_t) + \ps{\grad_x f(t,X_t)}{\dot{X}_t}_{X_t}\eqsp,
  \end{equation}
  hence
  \begin{equation}
    \PE\left[\frac{\rmd}{\rmd t}f(t,X_t)\right] = \PE[\partial_tf(t,X_t)] + \PE\left[\ps{\grad_xf(t,X_t)}{\PE[\dot{X}_t|X_t]}_{X_t}\right]\eqsp.
  \end{equation}
  Integrating over $\ooint{0,1}$, the left-hand term vanishes, yielding
  \begin{equation}
    \int_0^1 \int_{\msm} \partial_t f(t,x)\rmd \pi_t(x) \rmd t + \int_0^1\int_{\msm} \ps{\grad_xf(t,x)}{\bfu_t(x)}_{x} \rmd\pi_t(x) \rmd t = 0\eqsp,
  \end{equation}
  completing the proof.

\end{proof}

\subsection{Proof of \Cref{prop:density}}\label{proof:spherical}
\begin{proof}

 Let $t\in\coint{0,1}$ and $x_1\in\msm$.
  Let $f:\msm \rightarrow \R$ be a bounded continuous function. Since $\pi_{t|1}(\cdot|x_1) = \psi_t(\cdot;x_1)_{\sharp}\pi_{0|1}(\cdot|x_1)$, $\psi_t(\cdot;x_1)$ being well-defined almost everywhere, we have
\begin{equation}
    \int_{\msm} f(x_t)\pi_{t|1}(x_t|x_1) = \int_{\msm} f\circ\psi_t(x_0;x_1)p_{0|1}(x_0|x_1)\Rvol(\rmd x_0)\eqsp.
\end{equation}
 We write this integral using spherical coordinates centred at $x_1$: see \Cref{eq:change-variable}. 

\begin{equation}
\int_{\Smsm{x_1}}
\int_0^{\cuttime_{x_1}(\xi)}f\left(\psi_t(\Exp_{x_1}s\xi;x_1)\right) p_{0|1}\left(\Exp_{x_1}s\xi|x_1\right) \det \detA(s) \rmd s \rmd \mu_{x_1}(\xi) \eqsp.
\end{equation}
    Now, we can use the expression of $\psi_t$ to simplify the Riemannian exponential and logarithm maps, yielding
    \begin{align}&\int_{\Smsm{x_1}}\int_0^{\cuttime_{x_1}(\xi)}f\left(\Exp_{x_1}s(1-t)\xi\right) p_{0|1}\left(\Exp_{x_1}s\xi|x_1\right) \det \detA(s) \rmd s \rmd \mu_{x_1}(\xi) \\
    &=\int_{\Smsm{x_1}}\int_0^{\cuttime_{x_1}(\xi)(1-t)}f\left(\Exp_{x_1}s\xi\right) p_{0|1}\left(\Exp_{x_1}\frac{s}{1-t}\xi|x_1\right) \det \detA\left(\frac{s}{1-t}\right)\frac{\rmd s}{1-t} \rmd \mu_{x_1}(\xi) \\
    \begin{split}
&=\int_{\Smsm{x_1}}\int_0^{\cuttime_{x_1}(\xi)}\1\left(\frac{s}{1-t}\xi \in \ID(x_1)\right)f\left(\Exp_{x_1}s\xi\right) p_{0|1}\left(\Exp_{x_1}\frac{s}{1-t}\xi|x_1\right) \\ &\hspace{8cm}\det \detA\left(\frac{s}{1-t}\right) \frac{\rmd s}{1-t} \rmd \mu_{x_1}(\xi)\eqsp,
    \end{split}
    \end{align}
    where we first applied a linear change of variable $s\leftarrow s(1-t)$, and then used that for a given $\xi\in\Smsm{x_1}$, $s\xi\in\ID(x_1)$ if and only if $s<\cuttime_{x_1}(\xi)$. Finally, we recognize the expression of $\psi_t^{-1}(\cdot;x_1)$, and note that $\det \detA \neq 0$ when $s\xi\in\ID(x_1)$ \cite[Section III.2]{chavel_riemannian_2006}, to go back to an integral on $\msm$. Using the variable $x_t = \Exp_{x_1}s\xi$ and denoting $\Log_{x_1}x_t/\dist(x_1,x_t) = \xi_t$, we get
\begin{align}
\begin{split}
&=\int_{\Smsm{x_1}}\int_0^{\cuttime_{x_1}(\xi)}\1\left(\frac{s}{1-t}\xi \in \ID(x_1)\right) f\left(\Exp_{x_1}s\xi\right) p_{0|1}\left(\Exp_{x_1}\frac{1}{1-t}\Log_{x_1}\Exp_{x_1}s\xi|x_1\right) \\
&\hspace{8cm}\frac{\detA\left(\frac{s}{1-t}\right)}{\detA(s)} \frac{\rmd s}{1-t} \rmd \mu_{x_1}(\xi)
\end{split}\\
\begin{split}
&=\int_{\msm}\1\left(\frac{\Log_{x_1}x_t}{1-t} \in \ID(x_1)\right) f(x_t)  p_{0|1}\left(\Exp_{x_1}\frac{1}{1-t}\Log_{x_1}x_t|x_1\right) \\
    &\hspace{8cm}\frac{\det \scrA_{x_1,\xi_t}\left(\frac{\dist(x_t,x_1)}{1-t}\right)}{(1-t)\det \scrA_{x_1,\xi_t}(\dist(x_t,x_1))} \Rvol(\rmd x_t)\end{split}\\
    &= \int_{\msm} f(x_t) \1_{\mso_t(x_1)}(x_t) p_{0|1}\left(\psi_t^{-1}(x_t;x_1)|x_1\right) \Jac_t(x_t;x_1) \Rvol(\rmd x_t)\eqsp,
\end{align}
where $\Jac_t(x_t;x_1)$ is defined for every $x_t\in\mso_t(x_1)$ as
\begin{equation}
    \Jac_t(x_t;x_1) = \frac{\det \scrA_{x_1,\xi_t}\left(\frac{\dist(x_t,x_1)}{1-t}\right)}{(1-t)\det \scrA_{x_1,\xi_t}\left(\dist(x_t,x_1)\right)}\eqsp.
\end{equation}
 This proves that $\pi_{t|1}(\cdot|x_1)$ has a density with respect to $\Rvol$, given by
\begin{equation}
    p_{t|1}(x_t|x_1) = \1(x_t\in\mso_t(x_1)) p_{0|1}(\psi_t^{-1}(x_t;x_1)|x_1)\Jac_t(x_t;x_1)\eqsp.
\end{equation}

Since $x_t\in\mso_t(x_1)\mapsto \psi_t^{-1}(x_t;x_1)$ is smooth, as well as $(s,\xi)\mapsto \scrA_{x_1,\xi}(s)$ (from \cite[Equation III.1.4]{chavel_riemannian_2006}, it is defined as the solution of a second-order ODE with smooth coefficients), this density is smooth on $\mso_t(x_1)$ granted smoothness of $p_{0|1}(\cdot|x_1).$

\end{proof}

\end{appendix}

\end{document}